\documentclass[paperletter]{article}
\pdfoutput=1 
\usepackage{uai2019}
\usepackage[margin=1in]{geometry}

\usepackage{times}

\title{The Role of Memory in Stochastic Optimization}

\author{} 

\author{ {\bf Antonio Orvieto$^*$, \ Jonas Kohler, \ Aurelien Lucchi} \\
Department of Computer Science \\
ETH Z\"urich, Switzerland \\
}

\newcommand\blfootnote[1]{%
  \begingroup
  \renewcommand\thefootnote{}\footnote{#1}%
  \addtocounter{footnote}{-1}%
  \endgroup
}

\usepackage{microtype}
\usepackage{amsmath}
\usepackage{amssymb}
\usepackage{amsthm}
\usepackage{booktabs}  
\usepackage{color}
\usepackage{graphicx}
\usepackage{float}
\usepackage{mathtools}
\usepackage{mathrsfs}  
\usepackage{tikz}
\usepackage{scalerel}
\usepackage{subcaption}
\usepackage{enumerate}
\usepackage{enumitem}

\usepackage{tcolorbox}
\usepackage{algorithm}
\usepackage[noend]{algpseudocode}
\usepackage[utf8]{inputenc}
\usepackage[export]{adjustbox}
\usepackage[T1]{fontenc}

\usepackage[export]{adjustbox}

\makeatletter
\def\BState{\State\hskip-\ALG@thistlm}
\makeatother

\setlist[2]{noitemsep} 
\setenumerate{noitemsep} 
\usepackage{cancel}

\newcommand{\tr}{\text{Tr}}

\DeclareMathOperator{\var}{Cov}
\DeclareMathOperator*{\argmin}{arg\,min}

\newcommand{\EE}{\mathcal{E}}
\newcommand{\G}{\mathcal{G}}
\newcommand{\OO}{\mathcal{O}}

\newcommand{\m}{{\mathfrak{m}}}

\newcommand{\C}{{\mathcal C}}

\newcommand{\E}{{\mathbb E}}

\newcommand{\N}{{\mathbb N}}
\newcommand{\R}{{\mathbb{R}}}

\usepackage{mdframed}
\newtheorem*{theorem*}{Theorem}
\newmdtheoremenv{theorem}{Theorem}[section]
\newmdtheoremenv{lemma}[theorem]{Lemma}
\newmdtheoremenv{proposition}[theorem]{Proposition}
\newmdtheoremenv{corollary}[theorem]{Corollary}

\newtheorem{lemmama}{Lemma}[section]
\newtheorem{remark}{Remark}[section]

\newcommand*\circled[1]{\tikz[baseline=(char.base)]{
            \node[shape=circle,draw,inner sep=2pt] (char) {#1};}}
\def\msquare{\mathord{\scalerel*{\Box}{gX}}}            

\definecolor{colorantonio}{RGB}{0,128,255}

\tcbset{colframe=white}

\begin{document}
\maketitle
\blfootnote{
$^*$ Correspondence to [orvietoa@ethz.ch].\\
Accepted paper at the 35th Conference on Uncertainty in Artificial Intelligence (UAI), Tel Aviv, 2019.
}

\begin{abstract}
\vspace{-2mm}
The choice of how to retain information about past gradients dramatically affects the convergence properties of state-of-the-art stochastic optimization methods, such as Heavy-ball, Nesterov's momentum, RMSprop and Adam. Building on this observation, we use stochastic differential equations (SDEs) to explicitly study the role of memory in gradient-based algorithms. We first derive a general continuous-time model that can incorporate arbitrary types of memory, for both deterministic and stochastic settings. We provide convergence guarantees for this SDE for weakly-quasi-convex and quadratically growing functions.
We then demonstrate how to discretize this SDE to get a flexible discrete-time algorithm that can implement a board spectrum of memories ranging from short- to long-term. Not only does this algorithm increase the degrees of freedom in algorithmic choice for practitioners but it also comes with better stability properties than classical momentum in the convex stochastic setting. In particular, no iterate averaging is needed for convergence. Interestingly, our analysis also provides a novel interpretation of Nesterov's momentum as stable gradient amplification and highlights a possible reason for its unstable behavior in the (convex) stochastic setting. Furthermore, we discuss the use of long term memory for second-moment estimation in adaptive methods, such as Adam and RMSprop. Finally, we provide an extensive experimental study of the effect of different types of memory in both convex and nonconvex settings.
\end{abstract}

\section{INTRODUCTION}
Our object of study is the classical problem of minimizing finite-sum objective functions:
\begin{equation}
    \tag{P}
    x^* = \argmin_{x\in\R^d} f(x):=\frac{1}{n}\sum_{i=1}^nf_i(x).
    \label{P}
\end{equation}

Accelerated gradient methods play a fundamental role in optimizing such losses, providing optimal rates of convergence for certain types of function classes such as the ones being convex~\cite{nesterov2018lectures}.
The two most popular momentum methods are Heavy-ball (\ref{HB})~\cite{polyak1964some} and Nesterov's accelerated gradient (NAG)~\cite{nesterov1983method}.
They are based on the fundamental idea of augmenting gradient-based algorithms with a momentum term that uses \textit{previous gradient directions} in order to accelerate convergence, which yields the following type of iterative updates:
\begin{equation}
\tag{HB}
x_{k+1} = x_{k}+\beta_k (x_k-x_{k-1}) - \eta \nabla f(x_k),
\label{HB}
\end{equation}
with $\beta_k$ an iteration dependent \textit{momentum parameter}\footnote{Gradient Descent~\cite{cauchy1847methode} can be seen as a special case of \ref{HB} for $\beta_k = 0$.} and $\eta$ a positive number called \textit{learning rate} (a.k.a. \textit{stepsize}).

Although both \ref{HB} and NAG have received a lot of attention in the literature, the idea of acceleration is still not entirely well understood.
For instance, a series of recent works~\cite{su2014differential,wibisono2016variational,yang2018physical} has studied these methods from a physical perspective, which yields a connection to damped linear oscillators. Arguably, the insights provided by these works are mostly descriptive and have so far not been able to help with the design of conceptually new algorithms. Furthermore, the resulting analysis often cannot be easily translated to stochastic optimization settings, where stability of momentum methods may actually be reduced due to inexact gradient information~\cite{jain2018accelerating, kidambi2018insufficiency}.



This lack of theoretical understanding is rather unsatisfying. Why is it that acceleration able to provide faster rates of convergence for convex functions but fails when used on non-convex functions or in a stochastic setting? This question is especially relevant given that momentum techniques (such as Adam~\cite{kingma2014adam}) are commonly used in machine learning in order to optimize non-convex objective functions that arise when training deep neural networks.

In order to address this issue, we here exploit an alternative view on the inner workings of momentum methods which is not physically-inspired but instead builds upon the theoretical work on memory gradient diffusions developed by~\cite{cabot2009long} and~\cite{ gadat2014long}. In order to leverage this analogy, we first rewrite~\ref{HB} as follows:
\begin{tcolorbox}
\vspace{-4mm}
\begin{equation}
\tag{HB-SUM}
\hspace{-1mm}
x_{k+1} = x_k - \eta \sum_{j=0}^{k-1} \left(\prod_{h=j+1}^k\beta_h\right) \nabla f(x_j)- \eta\nabla f(x_k)
\label{HB-SUM}
\end{equation}
\end{tcolorbox}
where $x_0=x_{-1}$ is assumed. That is, at each iteration $k$ the next step is computed using a \textit{weighted average of past gradients} : $x_{k+1} = x_k + \eta\sum_{j=0}^{k} w(j,k) \nabla f(x_k))$. In particular, if $\beta_h$ is constant across all iterations, the memory --- which is controlled by the weights --- vanishes exponentially fast (short-term memory). Such averaging provides a cheap way to 1) adapt to the geometry of ill-conditioned problems (also noted in~\cite{sutskever13}) and 2) denoise stochastic gradients if the underlying true gradients are changing slowly. Similarly, adaptive methods~\cite{duchi2011adaptive,kingma2014adam} use memory of past square gradients to automatically adjust the learning rate during training. For the latter task, it has been shown~\cite{reddi2018} that some form of long-term memory is convenient both for in theory (to ensure convergence) and in practice, since it has been observed that, in a mini-batch setting, \textit{large gradients are not common and might be quite informative}. 

In summary --- most modern stochastic optimization methods can be seen as composition of memory systems. Inspired by this observation and by the undeniable importance of shining some light on the acceleration phenomenon, we make the following contributions.
\vspace{-5mm}
\begin{enumerate}
    \item  Following previous work from \cite{cabot2009long}, we generalize the continuous-time limit of \ref{HB} to an interpretable ODE that can implement various types of gradient forgetting (Sec.~\ref{sec:mem_fun_grad_forg}). Next, we extend this ODE to the stochastic setting (Sec.~\ref{sec:stoc_gradients}). 
    \item By comparing the resulting SDE to the model for Nesterov momentum developed in~\cite{su2014differential}, we provide a novel interpretation of acceleration as gradient amplification and give some potential answers regarding the source of instability of stochastic momentum methods (Sec.~\ref{nesterov_comparison}).
    \item We study the convergence guarantees of our continuous-time memory system and show that, in the convex setting, long-term (polynomial) memory is more stable than classical momentum (Sec.~\ref{sec:analysis_discretization}).
    \item We discretize this memory system and derive an algorithmic framework that can incorporate various types of gradient forgetting efficiently. Crucially, we show the discretization process preserves the convergence guarantees.
    \item We run several experiments to support our theory with empirical evidence in both deterministic and stochastic settings (Sec.~\ref{sec:experiments}).
    \item We propose a modification of Adam which uses long-term memory of gradient second moments to adaptively choose the learning rates (Sec.~\ref{padam}).
\end{enumerate}
\vspace{-3mm}
We provide an overview of our notation in App. \ref{notation}.

\section{RELATED WORK}
\label{sec:related_work}
\vspace{-1mm}

\begin{scriptsize}
\begin{table*}
\centering
\begin{tabular}{| l | l | c | c |}
  \hline
   \textbf{Function}& \textbf{Gradient} &  \textbf{Rate} &  \textbf{Reference} \\
 \hline
$\mu$-strongly-convex, $L$-smooth & Deterministic & $f(x_k) -f(x^*) \le \mathcal{O}(q^k)$ & \cite{polyak1964some} \\
Convex, $L$-smooth & Deterministic & $f(\bar{x}_k) -f(x^*) \le \mathcal{O}(1/k)$ & \cite{ghadimi2015global} \\
Convex & Stochastic & $\E\left(f(\bar{x}_k) -f(x^*)\right)\le \mathcal{O}(1/\sqrt{k})$ & \cite{yang2016unified} (*) \\
Non-convex, $L$-smooth & Stochastic & $\min_{i\le k}\E \left[||\nabla f(x_i)||^2\right]\le \mathcal{O} (1/\sqrt{k})$ & \cite{yang2016unified} (*) \\
\hline
\end{tabular}
\caption{\footnotesize{Existing convergence rate for Heavy-ball for general functions (special cases for quadratic functions are mentioned in the main text). The term $\bar{x}_k$ denotes the Cesaro average of the iterates. The constant $q$ is defined as $q = \frac{\sqrt{L}-\sqrt{\mu}}{\sqrt{L}+\sqrt{\mu}}$. (*) The results of~\cite{yang2016unified} also require bounded noise and bounded gradients as well as a step size decreasing as $1/\sqrt{k}$.}}
\label{tb:summary_existing_rates}
\end{table*}
\end{scriptsize}
\paragraph{Momentum in deterministic settings.}
The first accelerated proof of convergence for the deterministic setting dates back to~\cite{polyak1964some} who proved a local linear rate of convergence for Heavy-ball (with constant momentum) for twice continuously differentiable, $\mu$-strongly convex and $L$-smooth functions (with a constant which is faster than gradient descent). \cite{ghadimi2015global} derived a proof of convergence of the same method for convex functions with Lipschitz-continuous gradients, for which the Cesàro average of the iterates converges in function value like $\mathcal{O}(1/k)$ (for small enough $\eta$ and $\beta$).

A similar method, Nesterov's Accelerated Gradient (NAG), was introduced by~\cite{nesterov1983method}. It achieves the optimal $\mathcal{O}(1/k^2)$ rate of convergence for convex functions and, with small modifications, an accelerated (with respect to gradient desacent) linear convergence rate for smooth and strongly-convex functions.

\vspace{-2mm}
\paragraph{Momentum in stochastic settings.}
Prior work has shown that the simple momentum methods discussed above lack stability in stochastic settings, where the evaluation of the gradients is affected by noise~(see motivation in \cite{allen2017katyusha} for the Katyusha method). In particular, for quadratic costs, ~\cite{polyak1987introduction} showed that stochastic Heavy-ball does not achieve any accelerated rate but instead matches the
rate of SGD. More general results are proved in \cite{yang2016unified} for these methods, both for convex and for smooth functions, requiring a decreasing learning rate, bounded noise and bounded subgradients (see Tb.\ref{tb:summary_existing_rates}). For strongly-convex functions,~\cite{yuan2016influence} also studied the mean-square error stability and showed that convergence requires small (constant) learning rates. Furthermore, the rate is shown to be equivalent to SGD and therefore the theoretical  benefits of acceleration in the deterministic setting do not seem to carry over to the stochastic setting.

\vspace{-2mm}
\paragraph{Continuous-time perspective.} The continuous time ODE model of NAG for convex functions presented in~\cite{su2014differential} led to the developments of several variants of Nesterov-inspired accelerated methods in the deterministic setting (e.g.~\cite{krichene2015accelerated} and~\cite{wilson2016lyapunov}). In this line of research, interesting insights often come from a numerical analysis and discretization viewpoint~\cite{zhang2018direct,betancourt2018symplectic}. Similarly, in stochastic settings, guided by SDE models derived from Nesterov's ODE in \cite{su2014differential} and by the variational perspective in~\cite{wibisono2016variational}, ~\cite{xu2018accelerated} and~\cite{xu2018continuous} proposed an interpretable alternative to AC-SA (an accelerated stochastic approximation algorithm introduced in ~\cite{lan2012optimal} and \cite{ghadimi2012optimal}). This is a sophisticated momentum method that in expectation achieves a $\mathcal{O}(L/k^2 + \varsigma_*^2 d/(\mu k))$  rate\footnote{$\varsigma_*^2$ bounds the stochastic gradient variance in each direction.} for $\mu$-strongly convex and $L$-smooth functions and $\mathcal{O}(L/k^2 + \varsigma_*^2 d/\sqrt{k})$ for convex $L$-smooth functions. These rates are nearly optimal, since in the deterministic limit $\varsigma_*\to 0$ they still capture acceleration.

Unlike~\cite{xu2018accelerated,xu2018continuous}, we focus on how the memory of past gradients relates to the \textit{classical} and most widely used momentum methods (\ref{HB}, NAG) and, with the help of the SDE models, show that the resulting insights can be used to design building blocks for new optimization methods.

\vspace{-1mm}
\section{MEMORY GRADIENT SDE}
\label{sec:MGSDE}
\vspace{-2mm}
In his 1964 paper, Polyak motivated \ref{HB} as the discrete time analogue of a second order ODE:
\begin{tcolorbox}
\vspace{-2mm}
\begin{equation}
    \tag{HB-ODE}
    \ddot X(t) + a(t) \dot X(t) + \nabla f(X(t)) = 0, \ \ 
    \label{HB-ODE}
\end{equation}
\end{tcolorbox}
which can be written in phase-space as
\begin{tcolorbox}
\vspace{-5mm}
\begin{equation}
    \tag{HB-ODE-PS}
    \begin{cases}
    \dot V(t) = -a(t) V(t) -\nabla f(X(t))\\
    \dot X(t) = V(t)
    \end{cases}.
    \label{HB-ODE_PS}
\end{equation}
\vspace{-5mm}
\end{tcolorbox}

This connection can be made precise: in App.~\ref{discretization_HB_proof} we show that \ref{HB} is indeed the result of semi-implicit Euler integration\footnote{~\cite{hairer2006geometric} for an introduction.} on \ref{HB-ODE_PS}.

\subsection{MEMORY AND GRADIENT FORGETTING}
\label{sec:mem_fun_grad_forg}
If the viscosity parameter $\alpha = a(t)$ is time-independent, \ref{HB-ODE}, with initial condition $\dot X(0) = 0$ and $X(0) = x_0$, can be  cast into an integro-differential equation\footnote{By computing $\ddot X$ from \ref{HB-ODE-INT-C} using the fundamental theorem of calculus and plugging in $\dot X(0) = 0$.}:
\begin{tcolorbox}
\vspace{-2mm}
\begin{equation}
    \tag{HB-ODE-INT-C}
    \dot X(t) = -\int_0^t e^{-\alpha \cdot (t-s)} \nabla f(X(s)) ds.
    \label{HB-ODE-INT-C}
\end{equation}
\end{tcolorbox}
\vspace{-1mm}
\paragraph{Bias correction.}
Notice that the instantaneous update direction of \ref{HB-ODE-INT-C} is a weighted average of the past gradients, namely $\int_0^t w(s,t) \nabla f(X(s)) ds$ with $w(s,t) := e^{\alpha(t-s)}$. However, the weights do not integrate to one. Indeed, for all $t$, we have $\int_0^t w(s,t) ds = (1-e^{-\alpha t})/\alpha$, which goes to $1/\alpha$ as $t\to\infty$. As a result, in the constant gradient setting, the previous sum is a \textit{biased estimator} of the actual gradient. This fact suggests a simple modification of \ref{HB-ODE-INT-C}, for $t>0$:
\begin{equation}
   \dot X(t) = -\frac{\alpha}{1-e^{-\alpha t}}\int_0^t e^{-\alpha\cdot(t-s)} \nabla f(X(s)) ds.\label{normalization_step} 
\end{equation}
which we write as $\dot X = -\frac{\alpha}{e^{\alpha t}-1}\int_0^t e^{\alpha s} \nabla f(X(s)) ds$. We note that this normalization step follows exactly the same motivation as bias correction in Adam; we provide an overview of this method in App.~\ref{sec:unbiasing}. If we define $\mathfrak{m}(t) := e^{\alpha t}-1$, the previous formula takes the form:
\begin{tcolorbox}
\vspace{-2mm}
\begin{equation}
\tag{MG-ODE-INT}
    \dot X(t) = -\int_0^t \frac{\dot\m(s)}{\mathfrak{m}(t)} \nabla f(X(s)) ds.
    \label{MG-ODE-INT}
\end{equation}
\end{tcolorbox}
This memory-gradient integro-differential equation (\ref{MG-ODE-INT}) provides a generalization of \ref{HB-ODE-INT-C}, with bias correction. The following crucial lemma is consequence of the fundamental theorem of calculus.
\begin{lemmama}
    For any $\mathfrak{m}\in\C^1(\R,\R)$ s.t. $\mathfrak{m}(0) = 0$, \ref{MG-ODE-INT} is normalized : $\int_0^t \frac{\dot\m(s)}{\m(t)}ds =1$, for all $t > 0$.
    \label{lemma:unbiased}
\end{lemmama}
\vspace{-2em}
\begin{proof}
Since $\mathfrak{m}(0)=0$, $\int_0^t\dot \m(s)ds = \mathfrak{m}(t)$.
\end{proof}
\vspace{-1mm}
Based on Lemma~\ref{lemma:unbiased}, we will always set $\m(0) = 0$. What other properties shall a general $\m(\cdot)$ have? Requiring $\dot \m(s) \neq 0$ for all $s\ge0$ ensures that there does not exist a time instant where the gradient is systematically discarded. Hence, since $\m(0)=0$, $\m(\cdot)$ is either monotonically decreasing and negative or monotonically increasing and positive. In the latter case, without loss of generality, we can flip its sign. This motivates the following definition.
\vspace{-1 mm}
\paragraph{Definition.}
$\m\in\mathcal{C}^1(\R_+,\R)$ is a \textbf{memory function} if it is non-negative, strictly increasing and s.t. $\m(0)=0$.

For example, $e^{\alpha t}-1$, from which we started our discussion, is a valid memory function. Crucially, we note that $\dot \m(\cdot)$ plays the important role of controlling the speed at which we \textit{forget} previously observed gradients. For instance, let $\m(t) = t^3$; since $\dot\m(s) = 3s^2$, the system forgets past gradients \textit{quadratically fast}. In contrast, $\m(t) = e^{\alpha t}-1$ leads to \textit{exponential forgetting}. Some important memory functions are listed in Tb.~\ref{tb:memories}, and their respective influence on past gradients is depicted in Fig.~\ref{fig:memory_types}. We point out that, in the limit $\alpha\to\infty$, the weights $w(s,t) = \frac{\dot\m(s)}{\m(t)}$ associated with exponential forgetting converge to a Dirac distribution $\delta(t-s)$. Hence, we recover the Gradient Descent ODE~\cite{mertikopoulos2018convergence}: $\dot X(t) = -\nabla f(X(t))$. For the sake of comparability, we will refer to this as \textit{instantaneous forgetting}.

Finally, notice that \ref{MG-ODE-INT} can be written as a second order ODE. Too see this, we just need to compute the second derivative. For $t>0$ we have that
$$\ddot X(t) = \frac{\dot\m(t)}{\m(t)^2} \int_0^t\dot \m(s) \nabla f(X(s)) ds - \frac{\dot\m(t)}{\m(t)}\nabla f(X(t)).$$
Plugging in the definition of $\dot X$ from the integro-differential equation, we get the memory-gradient ODE:
\begin{tcolorbox}
\vspace{-5mm}
\begin{equation}
    \tag{MG-ODE}
    \ddot X(t) + \frac{\dot\m(t)}{\m(t)}\dot X(t) + \frac{\dot\m(t)}{\m(t)}\nabla f(X(t)) = 0.
    \label{MG-ODE}
\end{equation}
\vspace{-5mm}
\end{tcolorbox}

\begin{figure}
    \centering
    \includegraphics[width=0.70\linewidth]{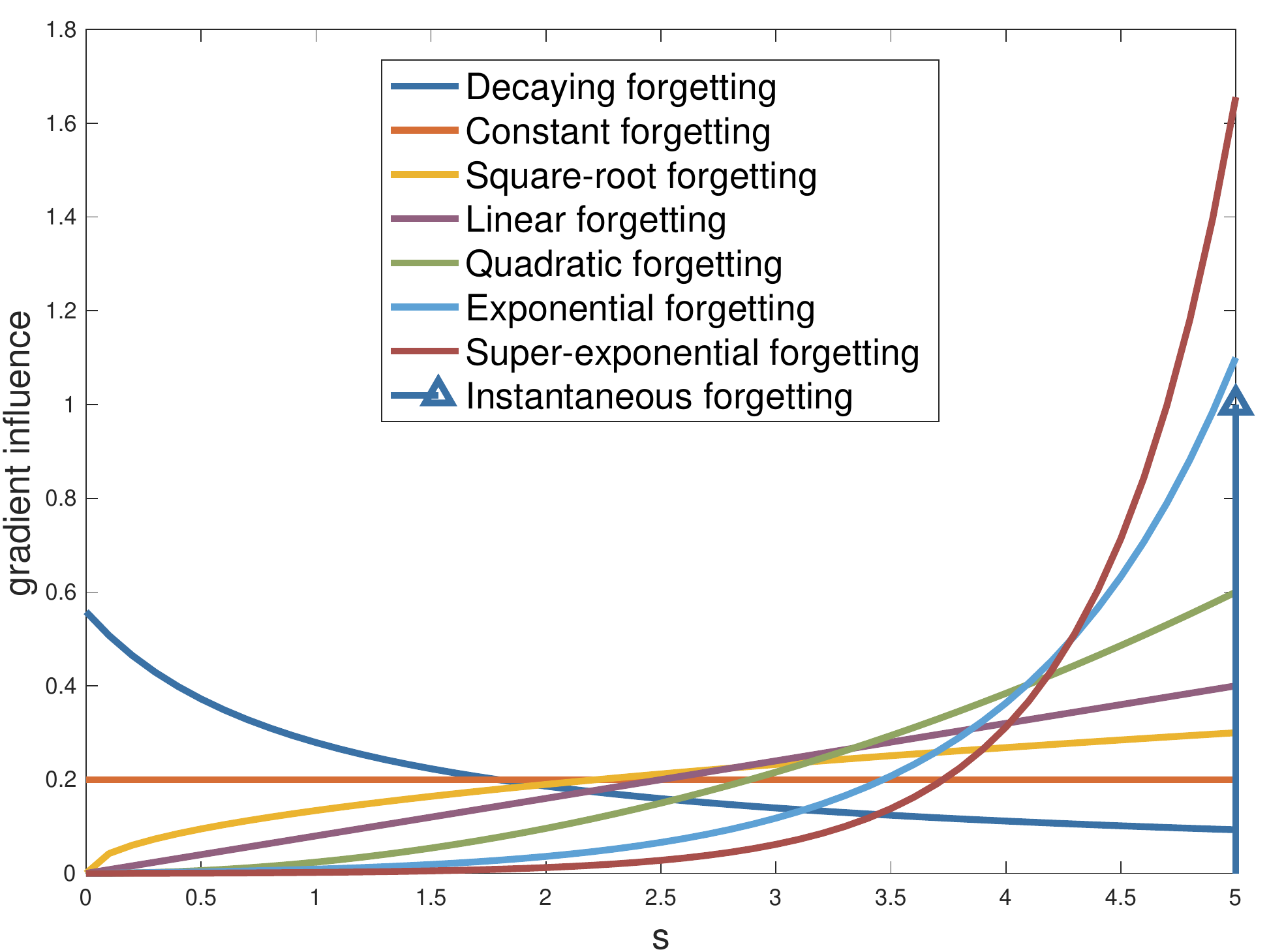}
    \caption{\footnotesize{Illustration of the influence of past gradients on $\dot X(6)$ (i.e. the right hand side of equation \ref{MG-ODE-INT} with $t=5$). The corresponding memory function can be found in Tb. \ref{tb:memories}. The influence is computed as $\dot\m(s)/\m(6)$. By Lemma \ref{lemma:unbiased}, the area under all curves is 1.}}
    \label{fig:memory_types}
\end{figure}
\begin{table}[ht]
\begin{center}
\begin{tabular}{l|l|l}
\small{\textbf{Forgetting}}  &\small{\textbf{Memory}   $\boldsymbol{\m}$} & \small{\textbf{ODE Coeff. } $\boldsymbol{\dot\m/\m}$} \\
\hline
Decaying    &$\log(1+t)$ & $1/(t\log(t+1))$ \\
Constant    &$t$ & $1/t$ \\
Square-root    &$t^{1.5}$ & $1.5/t$ \\
Linear      &$t^2$ & $2/t$ \\
Quadratic   &$t^3$ & $3/t$ \\
Exponential   &$e^{\alpha t}-1$ & $\alpha e^{\alpha t}/\left(e^{\alpha t}-1\right)$ \\
Super-exp   &$e^{t^\alpha}-1$ & $\alpha t^{\alpha-1} e^{t^\alpha}/\left(e^{t^\alpha}-1\right)$ \\
Instantaneous  &$-$ & $-$ \\
\hline
\end{tabular}
\end{center}
\caption{\footnotesize{Some important examples of memory functions.}}
\label{tb:memories}
\end{table}

Equivalently, we can transform this second order ODE into a system of two first order ODEs by introducing the variable $V(t):=\dot X(t)$ and noting that $\dot V(t) = -\frac{\dot\m(t)}{\m(t)} V(t) -\frac{\dot\m(t)}{\m(t)}\nabla f(X(t))$. This is called the \textit{phase-space representation} of~\ref{MG-ODE}, which we use in Sec.~\ref{sec:stoc_gradients} to provide the extension to the stochastic setting. Also, for the sake of comparison with recent literature~(e.g.~\cite{wibisono2016variational}), we provide a variational interpretation of~\ref{MG-ODE} in App.~\ref{lagrangian}.
\vspace{-2mm}
\paragraph{Existence and uniqueness.}
Readers familiar with ODE theory probably realized that, since by definition $\m(0) = 0$, the question of existence and uniqueness of the solution to~\ref{MG-ODE} is not trivial. This is why we stressed its validity for $t>0$ multiple times during the derivation. Indeed, it turns out that such a solution may not exist globally on $[0,\infty)$ (see App.~\ref{app:counterexample}).  Nevertheless, if we allow to start integration from \textit{any} $\epsilon>0$ and assume $f(\cdot)$ to be $L$-smooth, standard ODE theory~\cite{khalil2002nonlinear} ensures that the sought solution exists and is unique on $[\epsilon,\infty)$. Since $\epsilon$ can be made as small as we like (in our simulations in App.~\ref{app:ODEs_quadratic} we use $\epsilon = 10^{-16}$) this apparent issue can be regarded an artifact of the model. Also, we point out that the integral formulation~\ref{MG-ODE-INT} is well defined for every $t>0$. Therefore, in the theoretical part of this work, we act as if integration starts at $0$ but we highlight in the appendix that choosing the initial condition $\epsilon>0$ induces only a negligible difference (see Remarks~\ref{rmk1_exi} and~\ref{rmk2_exi}).

\subsection{INTRODUCING STOCHASTIC GRADIENTS}
\label{sec:stoc_gradients}
In this section we introduce stochasticity in the~\ref{MG-ODE} model. As already mentioned in the introduction, at each step $k$, iterative stochastic optimization methods have access to an estimate $\G(x_k)$ of $\nabla f(x_k)$: the so called \textit{stochastic gradient}. This information is used and possibly combined with previous gradient estimates $\G(x_0),\dots, \G(x_{k-1})$, to compute a new approximation $x_{k+1}$ to the solution $x^*$. There are many ways to design $\G(k)$: the simplest~\cite{robbins1951stochastic} is to take $\mathcal{G}_{\text{MB}}(x_k) := \nabla f_{i_k}(x_k)$,
where $i_k \in \{1,\dots,n\}$ is a uniformly sampled datapoint. This gradient estimator is trivially unbiased (conditioned on past iterates) and we denote its covariance matrix at point $x$ by $\Sigma(x) = \frac{1}{n}\sum_{i=1}^n (\nabla f_i(x)-\nabla f(x))(\nabla f_i(x)-\nabla f(x))^T$.

Following~\cite{krichene2017acceleration} we model such stochasticity adding a volatility term in \ref{MG-ODE}.

\begin{tcolorbox}
\vspace{-5mm}
\small{
\begin{equation}
    \tag{MG-SDE}
    \begin{cases}
    dX(t) = V(t) dt\\ 
    dV(t) = -\frac{\dot \m(t)}{\m (t)}V(t)dt\\
    \ \ \ \ \ \ \ \ \ \ \ \ \ \ \  -\frac{\dot \m(t)}{\m (t)}\left[\nabla f(X(t))dt +\sigma(X(t)) dB(t)\right]
    \end{cases}
    \label{MG-SDE}
\end{equation}
}
\vspace{-5mm}
\end{tcolorbox}
where $\sigma(X(t))\in\R^{d\times d}$ and $\{B(t)\}_{t\ge0}$ is a standard Brownian Motion. Notice that this system of equations reduces to the phase-space representation of \ref{MG-ODE} if $\sigma(X(t))$ is the null matrix. The connection from $\sigma(x)$ to the gradient estimator covariance matrix $\Sigma(x)$ can be made precise: \cite{li2017stochastic} motivate the choice  $\sigma(x) = \sqrt{h\Sigma(x)}$, where $\sqrt{\cdot}$ denotes the principal square root and $h$ is the discretization stepsize. 

The proof of existence and uniqueness to the solution of this SDE\footnote{See e.g. Thm. 5.2.1 in \cite{oksendal2003stochastic}, which gives sufficient conditions for (strong) existence and uniqueness.} relies on the same arguments made for \ref{MG-ODE} in Sec. \ref{sec:mem_fun_grad_forg}, with one additional crucial difference: \cite{orvieto2018continuous} showed that $f(\cdot)$ needs to additionally be three times continuously differentiable with bounded third derivative (i.e. $f\in\C^3_b(\R^d,\R)$), in order for $\sigma(\cdot)$ to be Lipschitz continuous. Hence, we will assume this regularity and refer the reader to~\cite{orvieto2018continuous} for further details.


\subsection{THE CONNECTION TO NESTEROV'S SDE}
\label{nesterov_comparison}
\cite{su2014differential} showed that the continuous-time limit of NAG for convex functions is \ref{HB-ODE} with time-dependent viscosity $3/t$: $\ddot X(t) +\frac{3}{t}\dot X(t) + \nabla f(X(t))=0$, which we refer to as \textit{Nesterov's ODE}. Using Bessel functions, the authors were able to provide a new insightful description and analysis of this mysterious algorithm. In particular, they motivated how the vanishing viscosity is essential for acceleration\footnote{Acceleration is not achieved for a viscosity of e.g. $2/t$.}. Indeed, the solution to the equation above is s.t. $f(X(t))-f(x^*) \le \OO(1/t^2)$; in contrast to the solution to the GD-ODE $\dot X(t) = -\nabla f(X(t))$, which only achieves a rate $\OO(1/t)$.

A closer look at Tb.~\ref{tb:memories} reveals that the choice $3/t$ is related to \ref{MG-ODE} with quadratic forgetting, that is $\ddot X(t) +\frac{3}{t}\dot X(t) +\frac{3}{t} \nabla f(X(t))=0$. However, it is necessary to note that in \ref{MG-SDE} also the gradient term is premultiplied by $3/t$. Here we analyse the effects of this intriguing difference and its connection to acceleration.


\vspace{-2mm}
\paragraph{Gradient amplification.}
A na\"ive way to speed up the convergence of the GD-ODE $\dot X(t) = -\nabla f(X(t))$ is to consider $\dot X(t) = -t\nabla f(X(t))$. This can be seen by means of the Lyapunov function $\EE(x,t) = t^2(f(x)-f(x^*)) +\|x-x^*\|^2$. Using convexity of $f(\cdot)$, we have $\dot \EE(X(t),t) = -t^2\|\nabla f(X(t))\|^2\le 0 $ and therefore, the solution is s.t. $f(X(t))-f(x^*)\le \mathcal{O}(1/t^2)$. However, the Euler discretization of this ODE is the gradient-descent-like recursion $x_{k+1} = x_{k} -\eta k \nabla f(x_k)$ --- which is \textit{not} accelerated. Indeed, this \textit{gradient amplification} by a factor of $t$ is effectively changing the Lipschitz constant of the gradient field from $L$ to $kL$. Therefore, each step is going to yield a descent only if\footnote{See e.g. \cite{bottou2018optimization}.} $\eta \le \frac{1}{kL}$. Yet, this iteration dependent learning rate effectively cancels out the gradient amplification, which brings us back to the standard convergence rate $\mathcal{O}(1/k)$. It is thus natural to ask: "\textit{Is the mechanism of acceleration behind Nesterov's ODE related to a similar gradient amplification?}"

In App.~\ref{app:nesterov_vs_quadratic_noise} we show that $\{X_N(t), V_N(t)\}_{t\ge0}$, the solution to Nesterov's SDE\footnote{Nesterov's SDE is defined, as for \ref{MG-SDE} by augmenting the phase space representation with a volatility term. The resulting system is then : $dX(t) = V(t) dt; \  dV(t) = -3/t V(t) dt -\sigma(X(t))dB(t)$.}, is s.t. the infinitesimal update direction $V_N(t)$ of the position $X_N(t)$ can be written as
\begin{equation}
    V_{N}(t) =  -\int_0^t \frac{s^{3}}{t^3}\nabla f(X(s))ds  +\zeta_{N}(t),
    \label{eq:vn}
\end{equation}
where $\zeta_{N}(t)$ is a random vector with $\E[\zeta_{N}(t)]=0$ and $\var[\zeta_{N}(t)]=\frac{1}{7} t\sigma\sigma^T$. In contrast, the solution $\{X_{\m2}(t), V_{\m2}(t)\})_{t\ge 0}$ of \ref{MG-SDE} with quadratic forgetting satisfies
\begin{equation}
    V_{\m2}(t) =  -\int_0^t \frac{3s^{2}}{t^3}\nabla f(X(s))ds  +\zeta_{\m2}(t),
    \label{eq:vm}
\end{equation}

\begin{figure}
  \centering
    \includegraphics[width=0.97\linewidth]{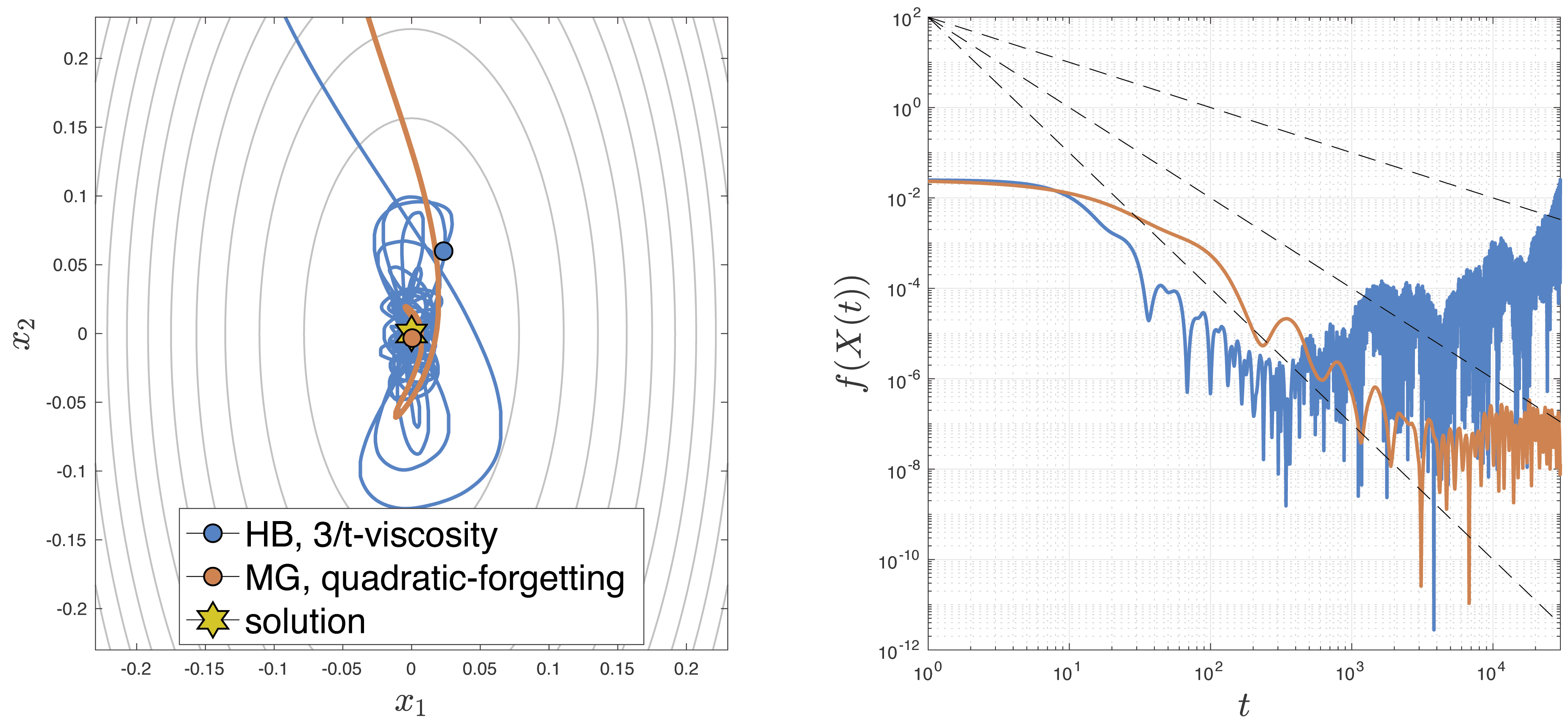} 
    \caption{\footnotesize{HB-SDE with $\alpha(t)=3/t$ (i.e. Nesterov's SDE) compared to \ref{MG-SDE} with quadratic forgetting. Setting as in \cite{su2014differential}: $f(x) = 2\times10^{-2} x_1^2 + 5\times10^{-3} x_2^2$ starting from $X_0 = (1,1)$ and $\dot X(0) = (0,0)$. Both systems are exposed to the same noise volatility. Simulation using the Milstein scheme \cite{mil1975approximate} with stepsize $10^{-3}$.}
    }
    \label{fig:HB_amplification}
\end{figure}

where $\zeta_{\m2}(t)$ is a random vector with $\E[\zeta_{\m2}(t)]=0$ but $\var[\zeta_{\m2}(t)]=\frac{9}{5t}\sigma\sigma^T$.
Even though the reader might already have spotted an important difference in the noise covariances, to make our connection to gradient amplification even clearer, we consider the simpler setting of constant gradients: in this case, we have $V_{N}(t) =  -\frac{1}{4} t \nabla f(X(t))  +\zeta_{N}(t)$, $V_{\m2}(t) =  -\nabla f(X(t))  +\zeta_{\m2}(t)$. That is, stochastic algorithms with increasing momentum (i.e. decreasing\footnote{See the connection between $\alpha$ and $\beta$ in Thm.~\ref{thm:HB_discretization}.} viscosity, like the Nesterov's SDE) \textit{are systematically amplifying the gradients over time}. Yet, at the same time they also linearly amplify the noise variance (see Fig.~\ref{fig:verif_noise} in the appendix). This argument can easily be extended to the non-constant gradient case by noticing that $\E[V_{\m2}(t)]$ is a weighted average of gradients where the weights integrate to 1 for all $t\ge 0$ (Lemma~\ref{lemma:unbiased}) . In contrast, in $\E[V_{N}(t)]$ these weights integrate to $t/4$. This behaviour is illustrated in Fig.~\ref{fig:HB_amplification}: While the Nesterov's SDE is faster compared to~\ref{MG-SDE} with $\m(t) = t^3$ at the beginning, it quickly becomes unstable because of the increasing noise in velocity and hence position. 

This gives multiple insights on the behavior of Nesterov's accelerated method for convex functions, both in for deterministic and the stochastic gradients:
\vspace{-3mm}
\begin{enumerate}
    \item Deterministic gradients get linearly amplified overtime, which counteracts the slow-down induced by the vanishing gradient problem around the solution. Interestingly Eq.~\eqref{eq:vn} reveals that  this amplification is \textit{not} performed directly on the local gradient but on past history, with \textit{cubic} forgetting. It is this feature that makes the discretization stable compared to the na\"ive approach $\dot X = -t\nabla f(X(t))$.
    \item Stochasticity corrupts the gradient amplification by an increasing noise variance (see Eq.~\eqref{eq:vn}), which makes Nesterov's SDE unstable and hence not converging. This finding is in line with~\cite{allen2017katyusha}.
\end{enumerate}
\vspace{-2mm}
Furthermore, our analysis also gives an intuition as to why a constant momentum cannot yield acceleration. Indeed, we saw already that \ref{HB-ODE-INT-C} does not allow such persistent amplification, but at most a constant amplification inversely proportional to the (constant) viscosity. Yet, as we are going to see in Sec.~\ref{sec:analysis_discretization}, this feature makes the algorithm more stable under stochastic gradients.

To conclude, we point the reader to App.~\ref{sec:quadratic_variance}, where we extend the last discussion from the constant gradient case to the quadratic cost case and get a close form for the (exploding) covariance of Nesterov's SDE (which backs up theoretically the unstable behavior shown in Fig.~\ref{fig:HB_amplification}). Nonetheless, we remind that this analysis still relies on continuous-time models; hence, the results above can only be considered as insights and further investigation is needed to translate them to the actual NAG algorithm.

\vspace{-2mm}
\paragraph{Time warping of linear memory.}
Next, we now turn our attention to the following question: "\textit{How is the gradient amplification mechanism of NAG related to its --- notoriously wiggling\footnote{Detailed simulations in App.~\ref{app:ODEs_quadratic}.}--- path?}". Even though Nesterov's ODE and \ref{MG-ODE} with quadratic forgetting are described by similar formulas, we see in Fig. \ref{fig:HB_amplification} that the trajectories are very different, even when the gradients are large. The object of this paragraph is to show that Nesterov's path has a strong link to --- surprisingly --- linear forgetting. Consider speeding-up the linear forgetting ODE $\ddot X(t) + \frac{2}{t} \dot X(t) + \frac{2}{t} \nabla f(X(t))$ by introducing the time change $\tau(t) = t^2/8$ and let $Y(t) = X(\tau(t))$ be the \textit{accelerated} solution to linear forgetting. By the chain rule, we have $\dot Y(t) = \dot\tau(t) \dot X(\tau(t))$ and $\ddot Y(t) = \ddot\tau(t) \dot X(\tau(t)) + \dot\tau(t)^2\ddot X(\tau(t))$.  It can easily be verified that we recover $\ddot Y(t) + \frac{3}{t} \dot Y(t) + \nabla f(Y(t))$. However, in the stochastic setting, the behaviour is still quite different: as predicted by the theory, in Fig. \ref{fig:timewarp} we see that --- when gradients are large --- the trajectory of the two sample paths are almost identical$^{11}$ (yet, notice that Nesterov moves faster); however, as we approach the solution, Nesterov diverges while linear forgetting stably proceeds towards the minimizer along the Nesterov's ODE path, but at a different speed, until convergence to a neighborhood of the solution, as proved in Sec.~\ref{sec:analysis_discretization}.
Furthermore, in App. \ref{time} we prove that there are no other time changes which can cast \ref{MG-ODE} into \ref{HB-ODE}, which yields the following interesting conclusion: the \textit{only} way to translate a memory system into a momentum method is by using a time change $\tau(t) = \mathcal{O}(t^2)$.
 
 \begin{figure}
  \centering
     \includegraphics[width=0.97\linewidth]{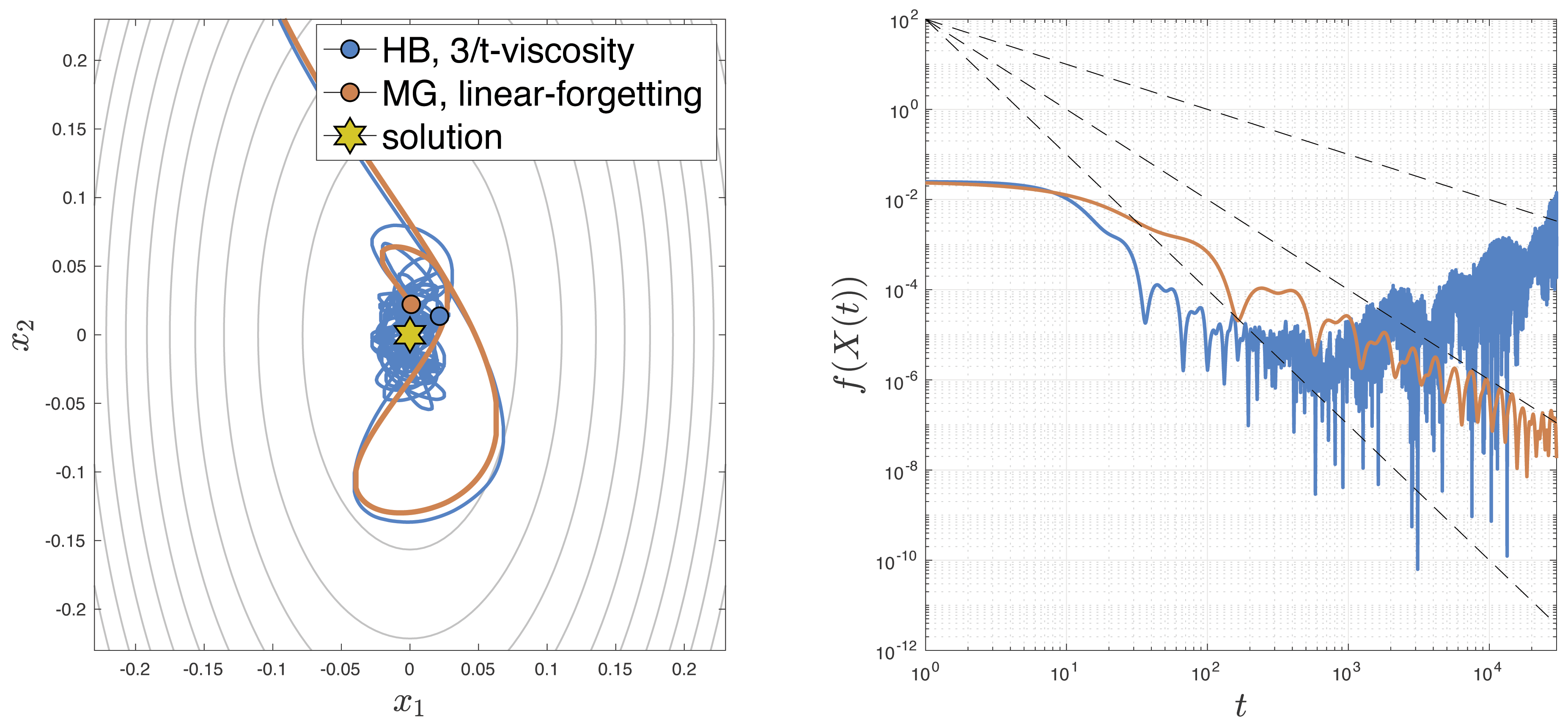}\caption{\footnotesize{Nesterov's ODE compared to \ref{MG-SDE} with linear forgetting (i.e. $\dot\m(t) / \m(t) = 2/t$). Same settings as Fig. \ref{fig:HB_amplification}.}}
    \label{fig:timewarp}
\end{figure}

\begin{table*}
\centering
\small{
\bgroup
\def\arraystretch{1.5}
\begin{tabular}{| l | l | l | l |}
  \hline
  \textbf{Forgetting}& \textbf{Assumption} &  \textbf{Rate} &  \textbf{Reference} \\
  \hline
  Instantaneous & \textbf{(H0c)}, \textbf{(H1)}& $\E[f(\bar X(t))-f(x^*)] \le C_i / t  + \  d\ \sigma^2_*/2$& \cite{mertikopoulos2018convergence}\\   
  \hline
  Exponential & \textbf{(H0c)}, \textbf{(H1)}& $\E[f(\bar X(t))-f(x^*)] \le C_e / t  + \  d \ \sigma^2_*/2$& App. \ref{proofs}, Thm. \ref{prop:CONV_HBF_WQC_app}\\
  \hline
    Polynomial  & \textbf{(H0c)}, \textbf{(H1)}, $p\ge 2$& $\E[f(X(t))-f(x^*)] \le C_p/ t + p \  d \ \sigma^2_*/2$& App. \ref{proofs}, Thm. \ref{prop:memory_continuous_WQC}\\    
\hline
\end{tabular}
\egroup
}
\caption{\footnotesize{Rates of \ref{MG-SDE} on convex smooth functions . $\bar X(t)=\int_0^t X(s)ds$} and $C_i, C_e, C_p$ can be found in the references.}
\label{tb:summary_rates_cont}
\end{table*}

\section{ANALYSIS AND DISCRETIZATION}
\label{sec:analysis_discretization}
\vspace{-1mm}
In this section we first analyze the convergence properties of \ref{MG-SDE} under different memory functions. Next, we use the Lyapunov analysis carried out in continuous-time to derive an iterative discrete-time method which implements polynomial forgetting and has provable convergence guarantees. We state a few assumptions:

\textbf{(H0c)} $f\in\C^3_b(\R^d,\R), \ \sigma_*^2 := \sup_{x} \|\sigma(x)\sigma(x)^T\|_s < \infty$.

The definition of $\sigma_*^2$ nicely decouples the measure of noise magnitude to the problem dimension $d$ (which will then, of course, appear explicitly in all our rates).

\textbf{(H1)} \ \ The cost $f(\cdot)$ is $L$-smooth and convex.

\textbf{(H2)} \ \ The cost $f(\cdot)$ is $L$-smooth and $\mu$-strongly convex.

We provide the proofs (under the less restrictive assumptions of weak-quasi-convexity and quadratic growth\footnote{$\tau$-weak-quasiconvexity is implied by convexity and has been shown to be of chief importance in the context of learning dynamical systems~\cite{hardt2018gradient}. Strong convexity implies quadratic growth with a unique minimizer~\cite{karimi2016linear} as well as $\tau$-weak-quasiconvexity. More details in the appendix.}), as well as an introduction to stochastic calculus, in App.~\ref{proofs}.

\subsection{EXPONENTIAL FORGETTING}
If $\m(t) = e^{\alpha t}-1$, then $\dot\m(t)/\m(t) = \frac{\alpha e^{\alpha t}}{e^{\alpha t}-1}$ which converges to $\alpha$ exponentially fast. To simplify the analysis and for comparison with the literature on HB-SDE~(which is usually analyzed under constant volatility \cite{shi2018understanding}) we consider here \ref{MG-SDE} with the approximation $\dot\m(t)/\m(t)\simeq\alpha$. In App.~\ref{proofs} we show that, under \textbf{(H1)}, the rate of convergence of $f(\cdot)$, evaluated at the Cesàro average $\bar X(t) = \int_0^t X(s)ds$ is sublinear~(see Tb. \ref{tb:summary_rates_cont}) to a ball\footnote{Note that the term "ball" might be misleading: indeed, the set $\mathcal{N}_\epsilon(x^*) = \{x\in\R^d, f(x)-f(x^*)\le\epsilon\}$ is not compact in general if $f(\cdot)$ is convex but not strongly convex.} around $x^*$ of size $d \ \sigma^2_*/2$, which is in line with known results for SGD \cite{bottou2018optimization}\footnote{Note that, by definition of $\sigma(\cdot)$ (see discussion after \ref{MG-SDE}), $\sigma^2_*$ is proportional both to the learning rate and to the largest eigenvalue of the stochastic gradient covariance}. Note that the size of this ball would change if we were to study a stochastic version of~\ref{HB-ODE} with constant volatility (i.e. $\ddot X + \alpha \dot X + \nabla f(X)$). In particular, it would depend on the normalization constant in  Eq.~\eqref{normalization_step}.  Also, in App.~\ref{proofs}, under \textbf{(H2)}, we provide a \textit{linear} convergence rate of the form $f(X(t))-f(x^*)\le \mathcal{O}(e^{-\gamma t})$ to a ball ($\gamma$ depends on $\mu$ and $\alpha$). Our result generalizes the analysis in \cite{shi2018understanding} to work with \textit{any viscosity} and with stochastic gradients. 

\vspace{-2mm}
\paragraph{Discretization.} As shown in Sec.~\ref{sec:mem_fun_grad_forg}, the discrete equivalent of~\ref{MG-SDE} with exponential forgetting is Adam without adaptive stepsizes (see App.~\ref{sec:unbiasing}). As for the continuous-time model we just studied, for a sufficiently large iteration, exponential forgetting can be approximated with the following recursive formula:
 $$x_{k+1} = x_k + \beta (x_k-x_{k-1}) - \eta(1-\beta)\nabla f(x_k),$$
which is exactly \ref{HB} with learning rate $(1-\beta)\eta$. Hence, the corresponding rates can be derived from Tb.~\ref{tb:summary_existing_rates}.  

\subsection{POLYNOMIAL FORGETTING}
The insights revealed in Sec.~\ref{nesterov_comparison} highlight the  importance of the choice $\m(t) = t^p$ in this paper. In contrast to instantaneous~\cite{mertikopoulos2018convergence} and exponential forgetting, the rate we prove in App.~\ref{proofs} for this case under \textbf{(H1)} does not involve a Cesàro average --- \textit{but holds for the last time point} (see Tb.~\ref{tb:summary_rates_cont}). This stability property is directly linked to our discussion in Sec.~\ref{nesterov_comparison} and shows that different types of memory may react to noise very differently. Also, we note that the size of the ball we found is now also proportional to $p$; this is not surprising since, as the memory becomes more focused on recent past, we get back to the discussion in the previous subsection and we need to consider a Cesàro average.
\vspace{-2mm}
\paragraph{Discretization.} Finally, we consider the burning question "\textit{Is it possible to discretize~\ref{MG-SDE} --- with polynomial forgetting --- to derive a cheap iterative algorithm with similar properties?}". In App.~\ref{poly-discre}, we build this algorithm in a non-standard way: we  reverse-engineer the proof of the rate for~\ref{MG-SDE} to get a method which is able to mimic each step of the proof. Starting from $x_{-1} = x_0$, it is described by the following recursion

\begin{tcolorbox}
\vspace{-5mm}
\begin{equation}
    \tag{MemSGD-p}
    x_{k+1} = x_k + \frac{k}{k+p} (x_k-x_{k-1}) - \frac{p}{k+p} \eta \nabla f(x_k).
    \label{MemSGD-p}
\end{equation}
\vspace{-5mm}
\end{tcolorbox}

As a direct result of our derivation, we show in Thm.~\ref{thm:discretization} (App.~\ref{poly-discre}) that \textit{this algorithm preserves exactly the rate} of its continuous-time model in the stochastic setting\footnote{Required assumptions: \textbf{(H1)}, $p\ge 2$, $\eta\le \frac{p-1}{pL}$ and $\varsigma_*^2$ bounds the gradient variance in each direction of $\R^d$.}:
$$\E[f(x_k)-f(x^*)] \le \frac{(p-1)^2\|x_0-x^*\|^2}{2 \eta p (k+p-1)} + \frac{1}{2}p d \eta \varsigma_*^2 .$$
We also show that \ref{MemSGD-p} can be written as $x_{k+1} = x_k -\eta\sum_{j=0}^k w(j,k)\nabla f(x_k)$, where $\sum_{j=0}^k w(j,k)=1$ (in analogy to the bias correction in Adam) and with $ w(\cdot,k)$ increasing as a polynomial of order $p-1$ for all $k$, again in complete analogy with the model. Fig.~\ref{fig:p_verification} shows the behaviour of different types of memory in a simple convex setting; as predicted, polynomial (in this case linear) forgetting has a much smoother trajectory then both exponential and instantaneous forgetting. Also, the reader can verify that the asymptotic noise level for MemSGD (p=2) is slightly higher, as just discussed. 

For ease of comparison with polynomial memory, we will often write MemSGD (p=e) to denote exponential forgetting (i.e. Adam without adaptive steps) and SGD (p=inf) to stress that SGD implements instantaneous forgetting.

 \begin{figure}
  \centering
    \includegraphics[width=0.65\linewidth]{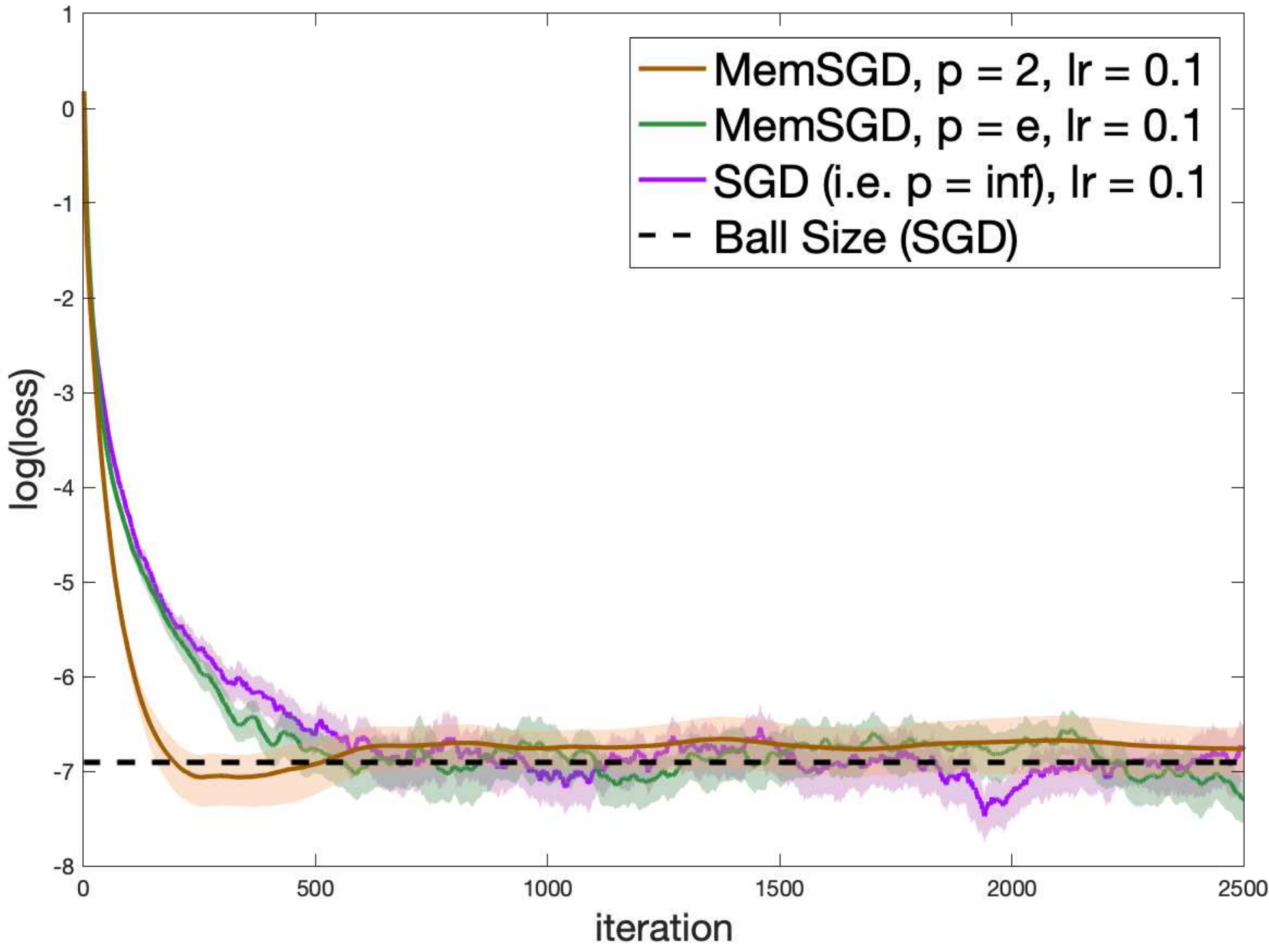}\caption{\footnotesize{Synthetic example: $f(x_1,x_2) = 0.8\times x_1^4 + 0.4\times x_2^4$ with Gaussian noise. Displayed is linear forgetting (i.e. MemSGD-2), exponential forgetting (denoted p=e) with $\beta = 0.8$ and instantaneous forgetting. Average and 95\% confidence interval for 150 runs starting from $(1,1)$.}}
    \label{fig:p_verification}
\end{figure}

\section{LARGE SCALE EXPERIMENTS} \label{sec:experiments}
\vspace{-1mm}
In order to assess the effect of different types of memory in practical settings, we benchmark MemSGD with different memory functions: from instantaneous to exponential, including various types of polynomial forgetting. As a reference point, we also run vanilla \ref{HB} with constant momentum as stated in the introduction. To get a broad overview of the performance of each method, we run experiments on a convex logistic regression loss as well as on non-convex neural networks in both a mini- and full-batch setting. Details regarding algorithms, datasets and architectures can be found in App.~\ref{sec:exp_setting}.
\label{sec:exp}
\begin{figure*}[ht]
\centering 
\begin{adjustbox}{center}
\begin{tabular}{c@{}c@{}c@{}c@{}}
        \small{Covtype Logreg} & \small{MNIST Autoencoder} & \small{FashionMNIST MLP} & \small{CIFAR-10 CNN}
          \\
            \includegraphics[width=0.245\linewidth,valign=c]{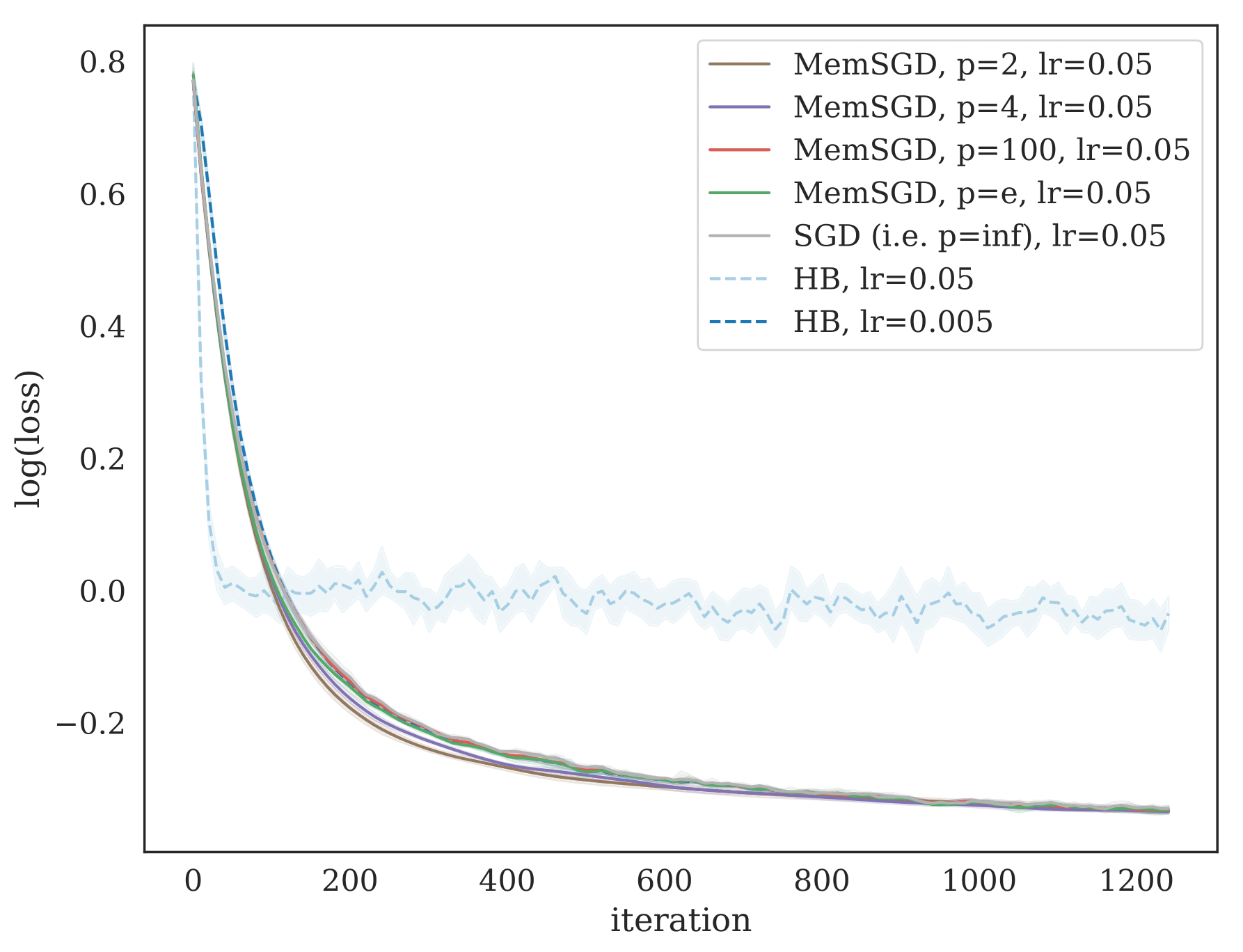}&
            \includegraphics[width=0.245\linewidth,valign=c]{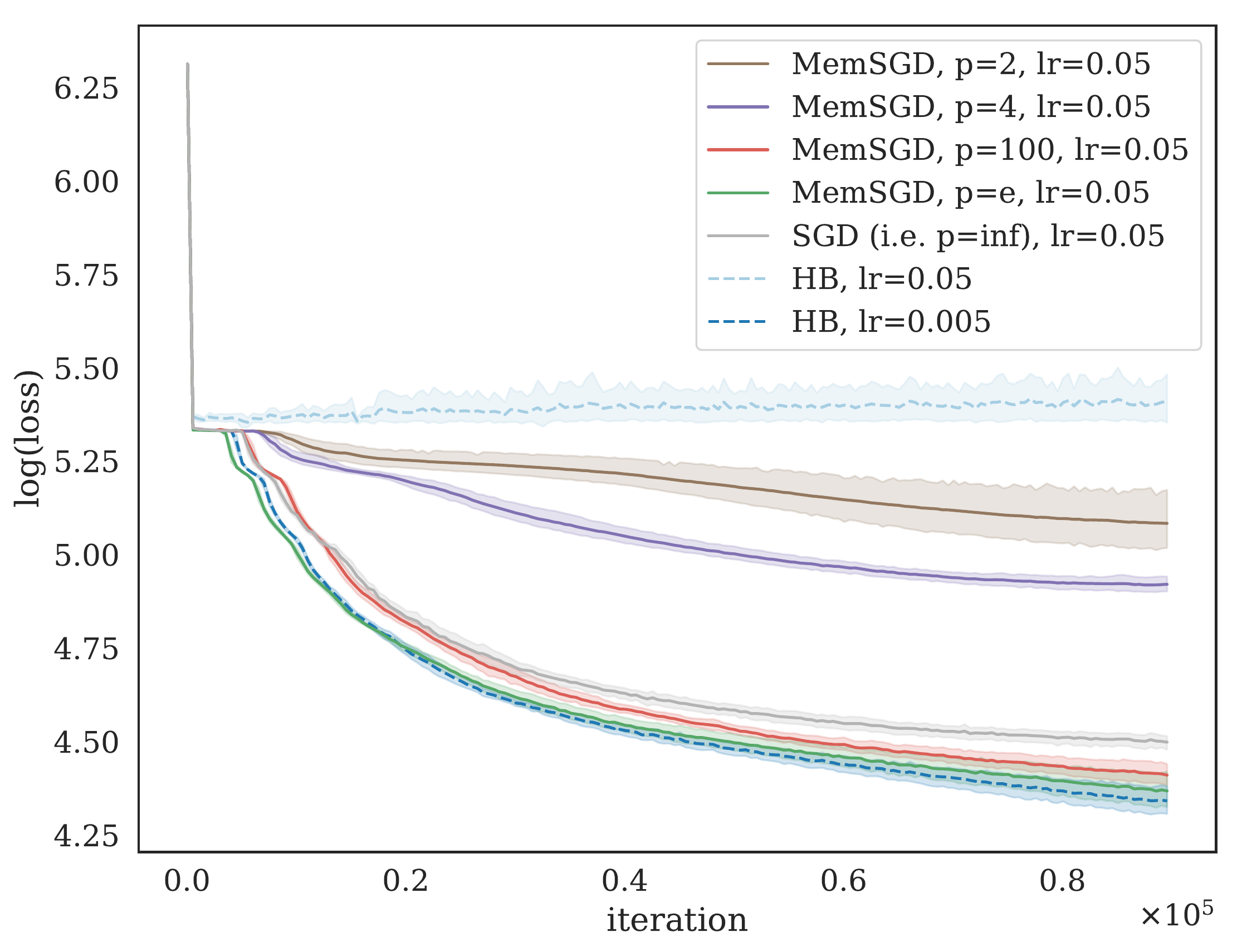}&
            \includegraphics[width=0.245\linewidth,valign=c]{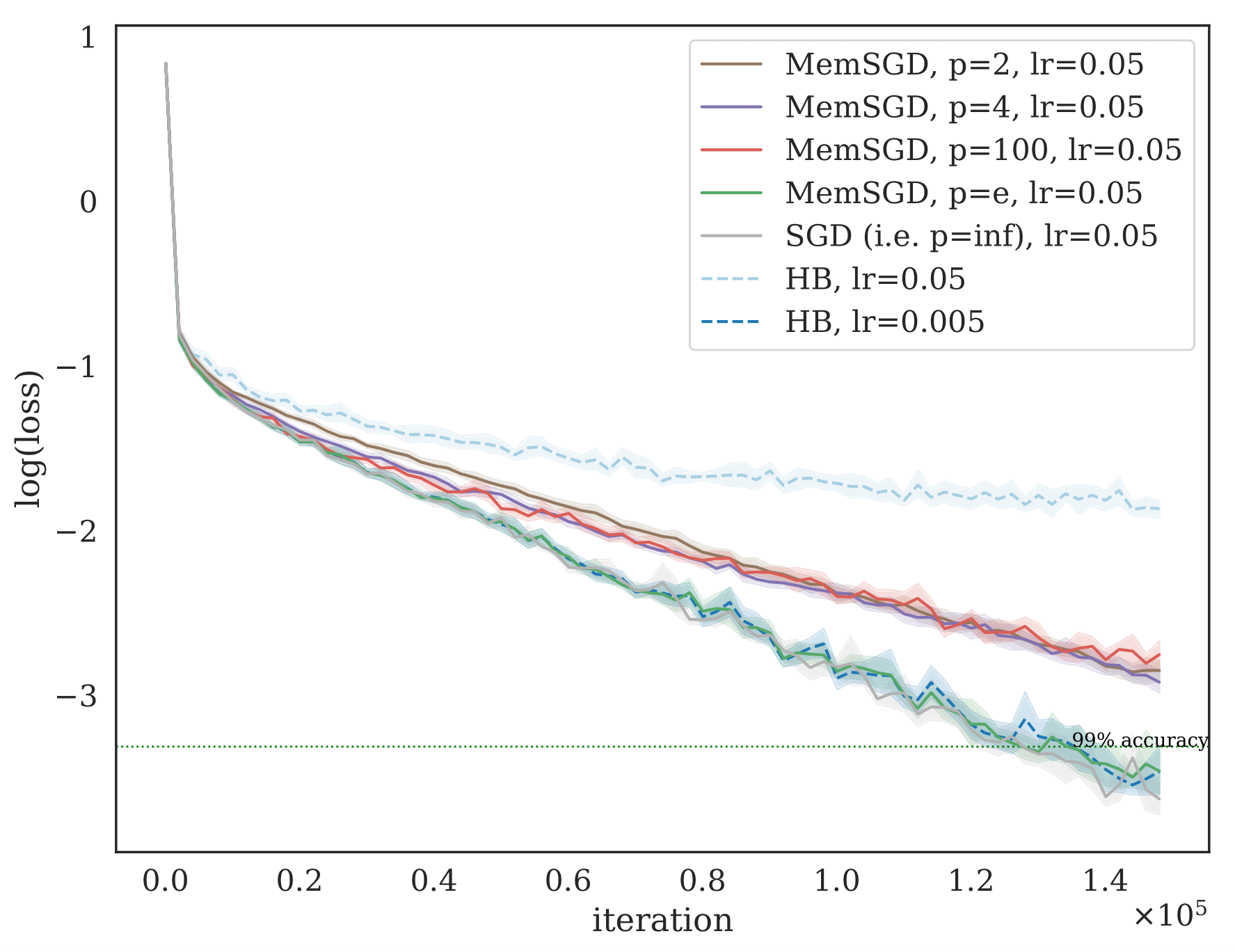}&
            \includegraphics[width=0.245\linewidth,valign=c]{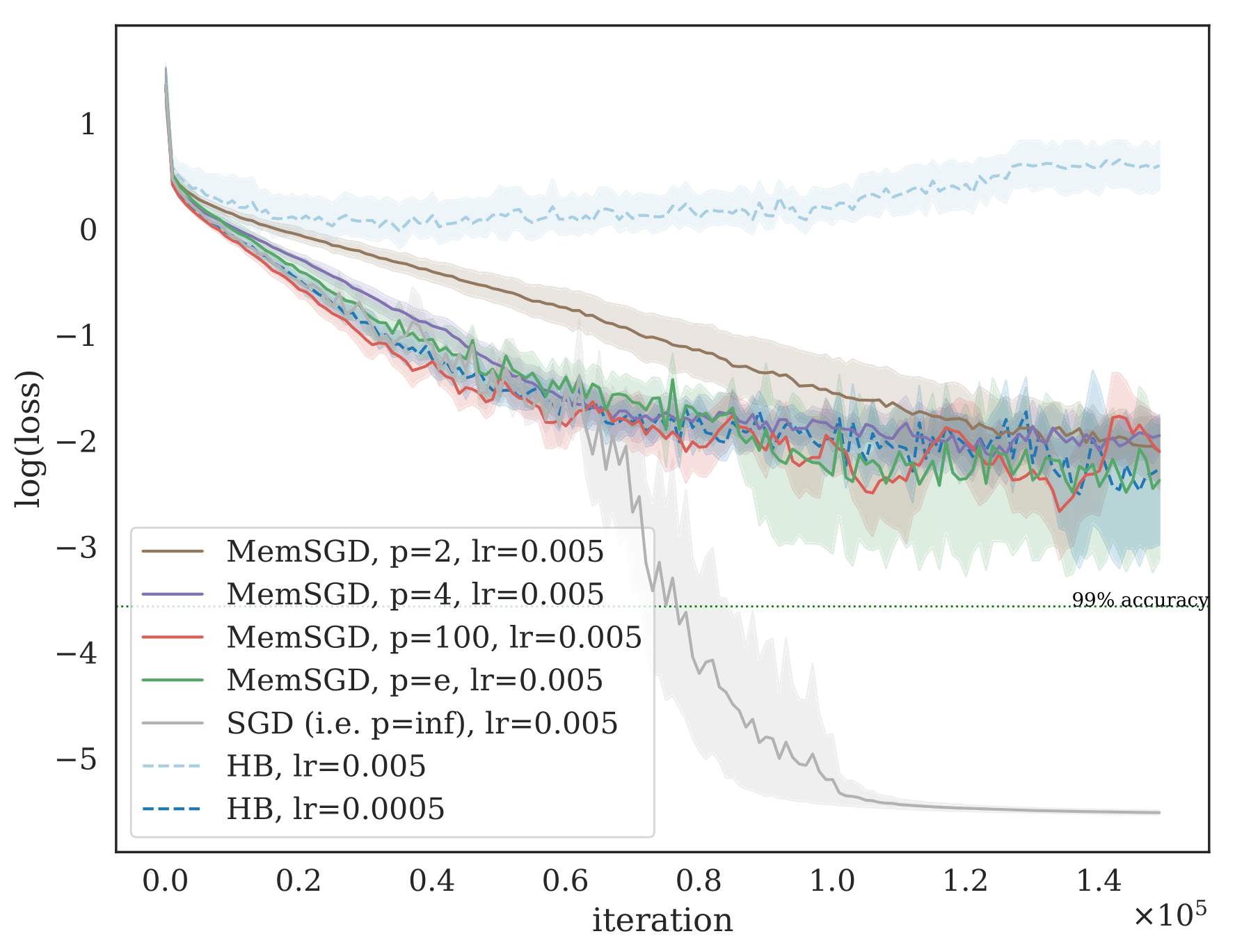}\\  
            ~
             \includegraphics[width=0.245\linewidth,valign=c]{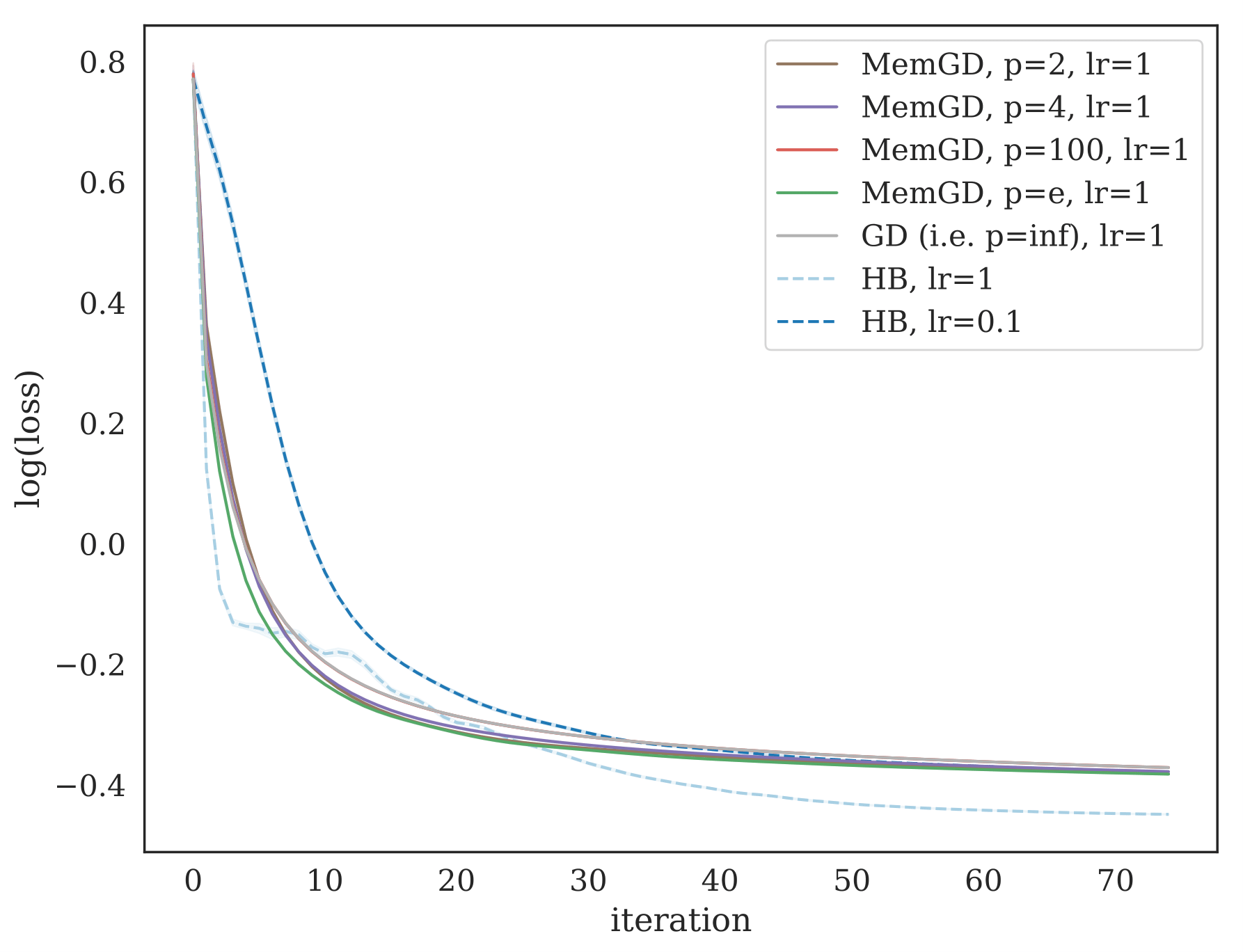}&
            \includegraphics[width=0.245\linewidth,valign=c]{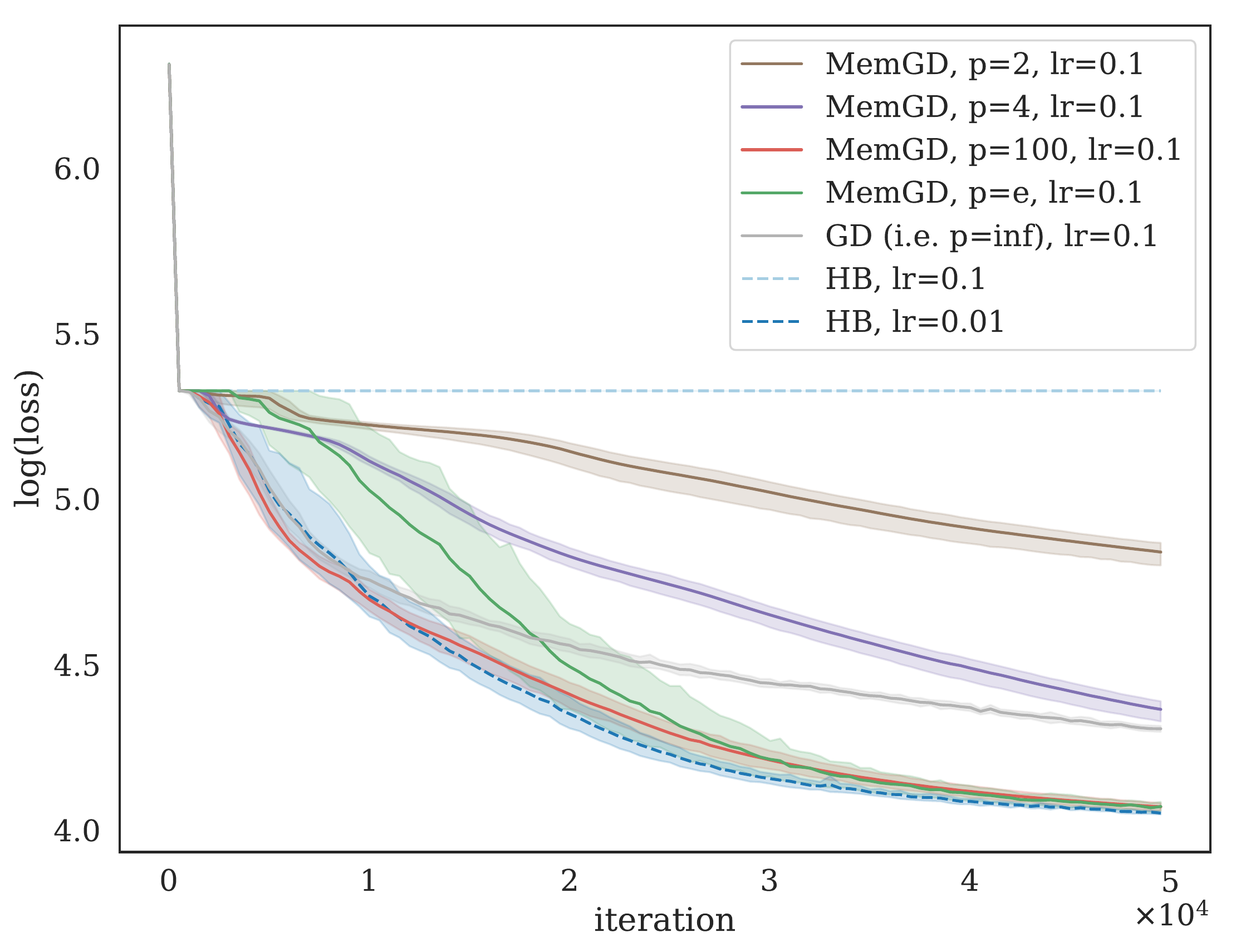}&
            \includegraphics[width=0.245\linewidth,valign=c]{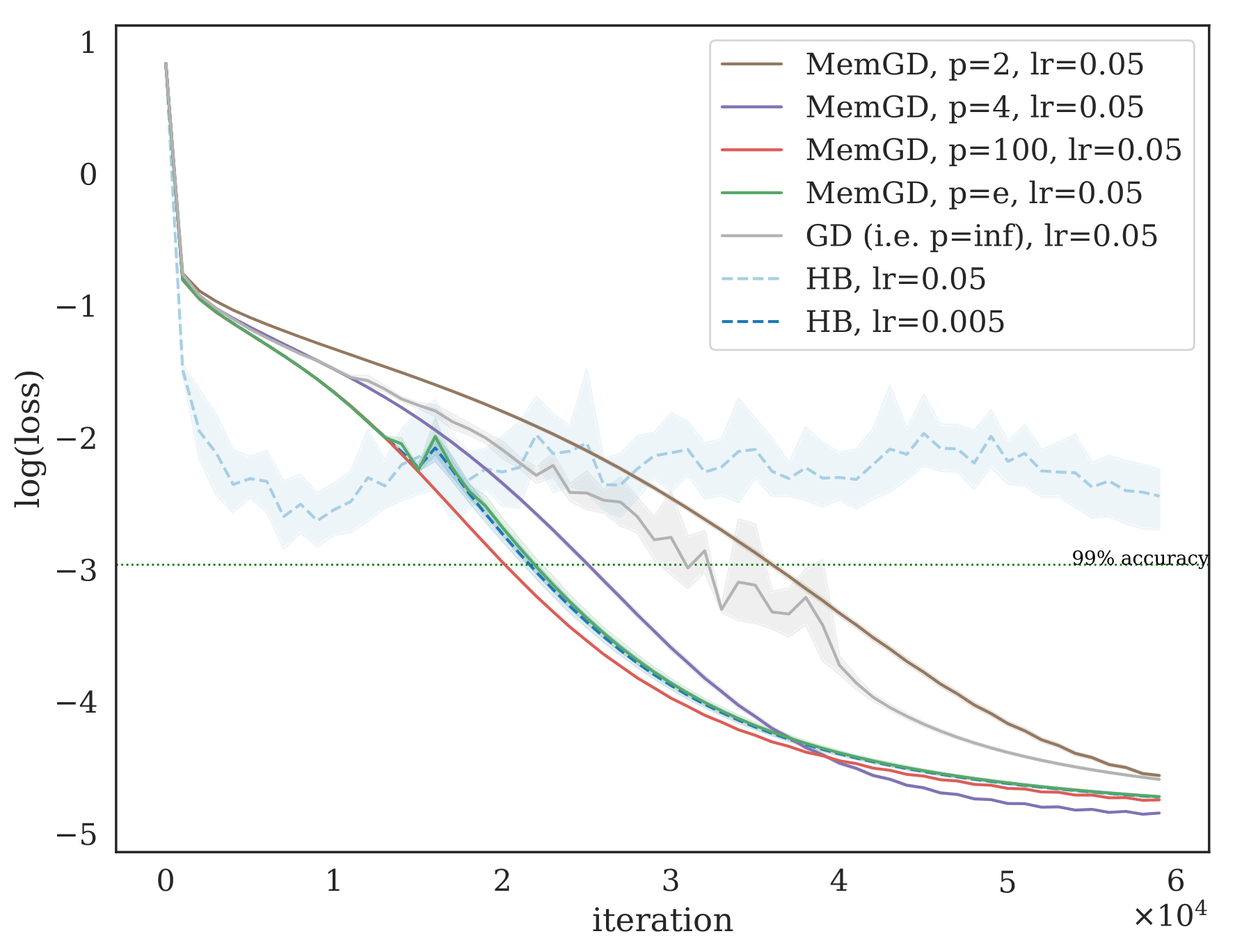}&
            \includegraphics[width=0.245\linewidth,valign=c]{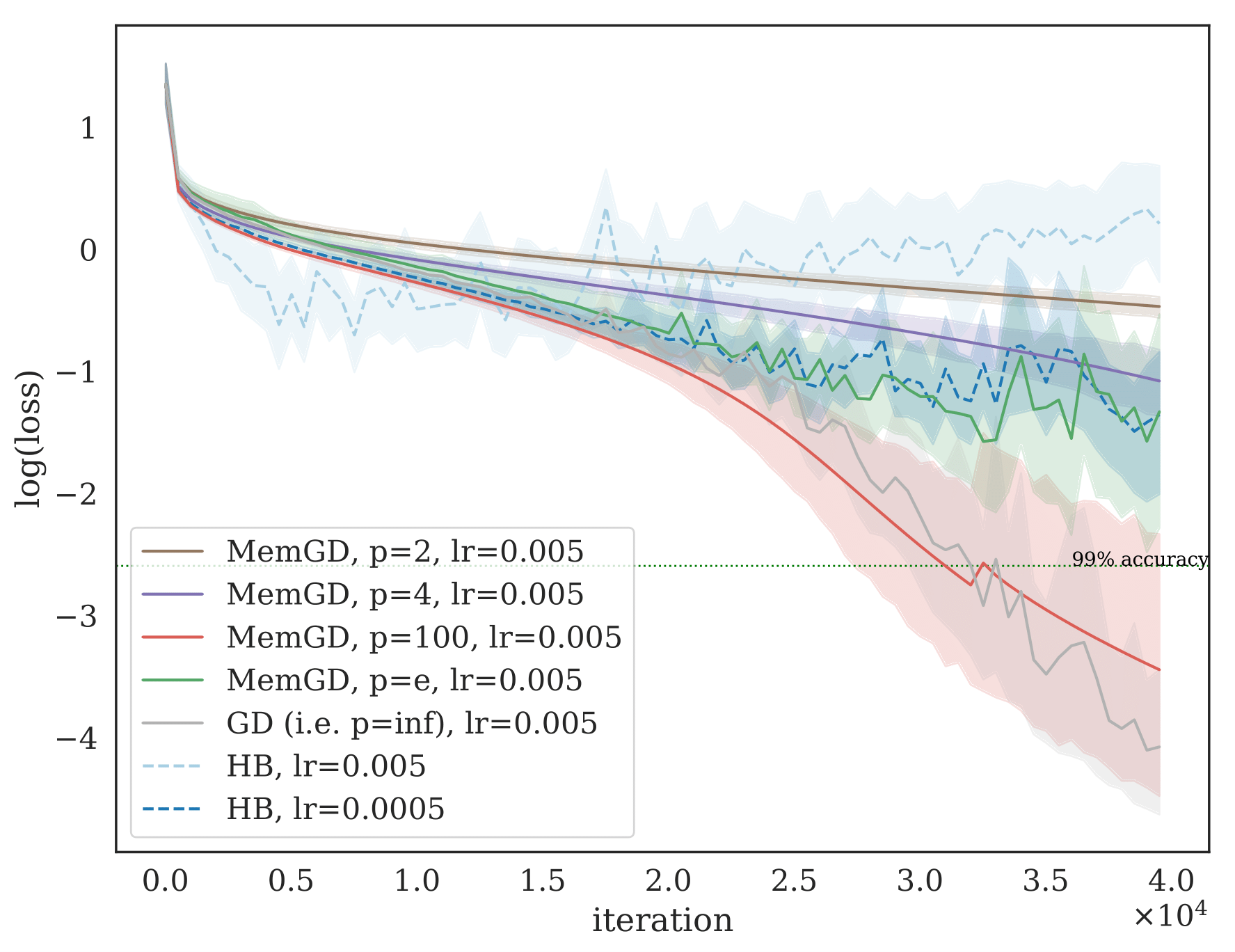}\\  
\end{tabular}
\end{adjustbox}
\vspace{-1mm}
\caption{\footnotesize{Log loss over iterations in mini- (top) and full-batch (bottom) setting. Average and $95\%$ CI of 10 random initializations.}}
\label{fig:results_iterations_main}
\end{figure*}

\vspace{-2mm}
\paragraph{Results and discussion.} Fig.~\ref{fig:results_iterations_main} summarizes our results in terms of training loss. While it becomes evident that no method is best on all problems, we can nevertheless draw some interesting conclusions. 

First, we observe that while long-term memory (especially $p=2$) is faster than SGD in the convex case, it does not provide any empirical gain in the neural network settings. This is not particularly surprising since past gradients may quickly become outdated in non-convex landscapes. Short term memory is at least as good as SGD in all cases except for the CIFAR-10 CNN, which represents the most complex of our loss landscapes in terms of curvature.

Secondly, we find that the best stepsize for \ref{HB} is always strictly smaller than the one for SGD in the non-convex setting. MemSGD, on the other hand, can run on stepsizes as large as SGD which reflects the gradient amplification of \ref{HB} as well as the unbiasedness of MemSGD. Interestingly, however, a closer look at Fig.~\ref{fig:exp_stepnorm} (appendix) reveals that \ref{HB} (with best stepsize) actually takes much smaller steps than SGD for almost all iterations. While this makes sense from the perspective that memory averages past gradients, it is somewhat counter-intuitive given the inertia interpretation of \ref{HB} which should make the method travel further than SGD. Indeed, both \cite{sutskever13} and \cite{goodfellow2016deep} attribute the effectiveness of \ref{HB} to its increased velocity along consistent directions (especially early on in the optimization process). However, our observation, together with the fact that MemSGD with fast forgetting ($p=e$ and $p=100$) is as good as HB, suggests that there is actually more to the success of taking past gradients into account and that this must lie in the altered directions that adapt better to the underlying geometry of the problem.\footnote{Note that we find the exact opposite in the convex case, where \ref{HB} does take bigger steps and converges faster.}

Finally, we draw two conclusions that arise when comparing the mini- and full batch setting. First, the superiority of \ref{HB} and fast forgetting MemSGD over vanilla SGD in the deterministic setting is indeed reduced as soon as stochastic gradients come into play (this is in line with the discussion in Sec. \ref{nesterov_comparison}). Second, we find that stochasticity per se is not needed to optimize the neural networks in the sense that all methods eventually reach very similar methods of suboptimality. That is, not even the full batch methods get stuck in any elevated local minima including the saddle found in the MNIST autoencoder which they nicely escape (given the right stepsize).
\section{MEMORY IN ADAPTIVE METHODS}
\label{padam}
\vspace{-1mm}
While the main focus of this paper is the study of the effect of different types of memory on the \emph{first moment} of the gradients, past gradient information is also commonly used to adapt stepsizes. This is the case for Adagrad and Adam which both make use of the second moment of past gradients to precondition their respective update steps. 

Of course, the use of polynomial memory generalizes directly to the second moment estimates and we thus consider a comprehensive study of the effect of long- versus short-term memory in adaptive preconditioning an exciting direction of future research. In fact, as shown in \cite{reddi2018} the non-convergence issue of Adam can be fixed by making the method forget past gradients less quickly. For that purpose the authors propose an algorithm called AdamNC that essentially differs from Adam by the choice of $\beta_2=1-1/k$, which closely resembles Adagrad with constant memory. Interestingly, the memory framework introduced in this paper allows to interpolate between the two extremes of constant- and exponential memory (i.e. Adagrad and Adam) in a principled way. Indeed, by tuning the additional parameter $p$ --- which specifies the degree of the polynomial memory function --- one can equip Adam with any degree of short- to long-term memory desired. As a proof of concept, Fig.~\ref{fig:padam} shows that Adam equipped with a polynomial memory of the squared gradients (PolyAdam) can in fact be faster than both Adam and Adagrad.

\label{sec:adaptive}
\vspace{-2mm}
\begin{figure}[H]
\centering 
\begin{adjustbox}{center}
\begin{tabular}{c@{}c@{}}
\includegraphics[width=0.49\linewidth,valign=c]{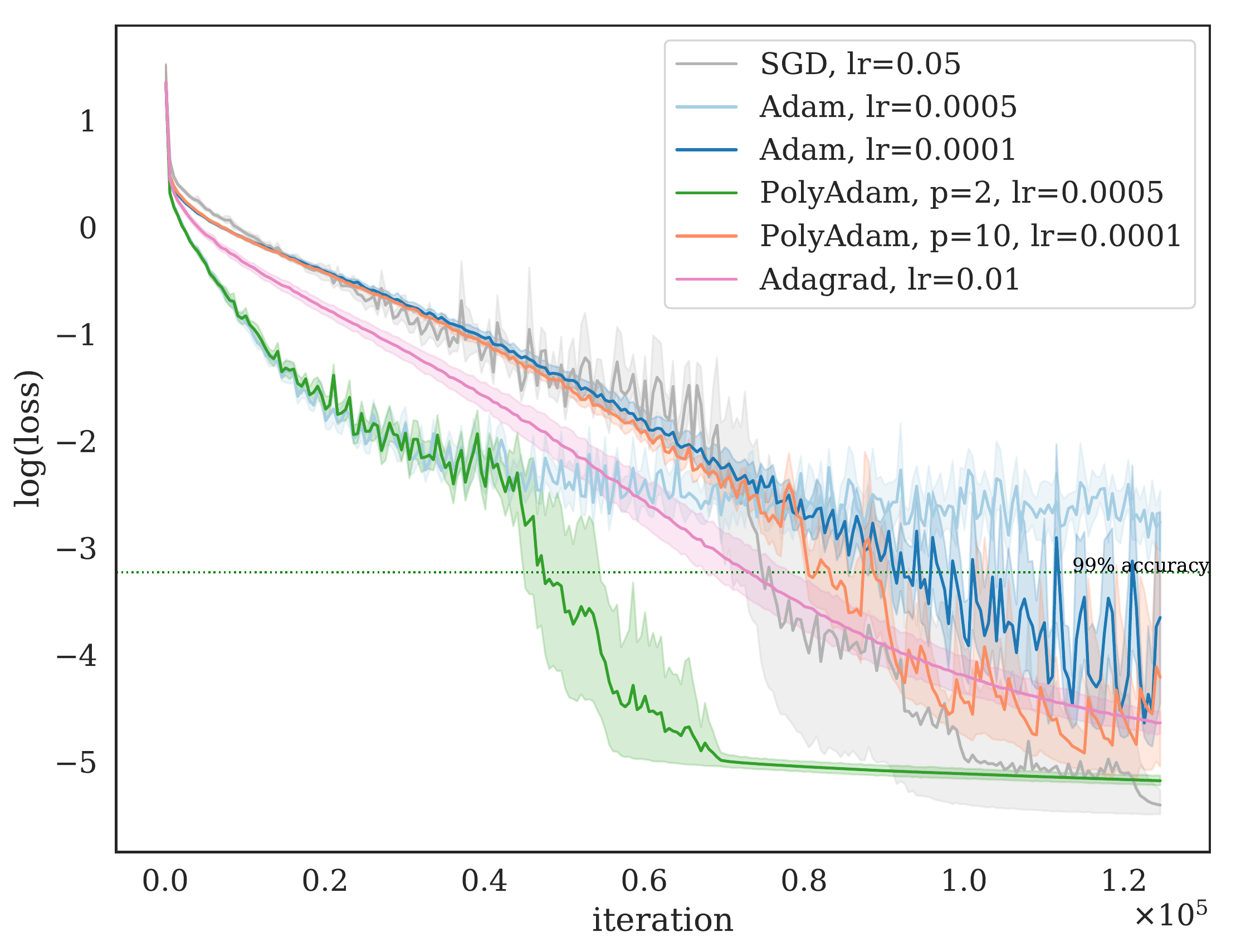}&
\includegraphics[width=0.49\linewidth,valign=c]{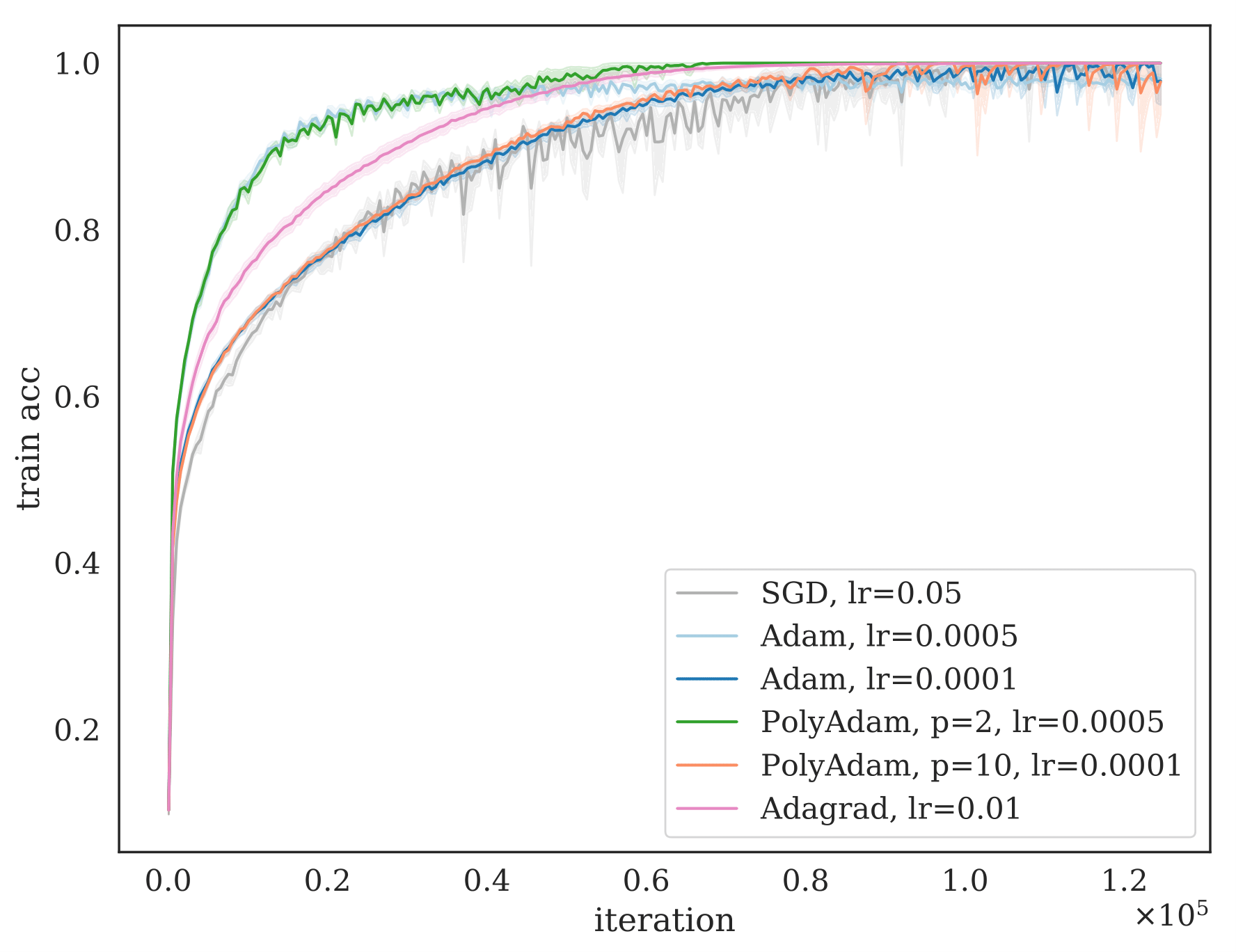}\\
\end{tabular}
\end{adjustbox}
\vspace{-2mm}
 \caption{\footnotesize{Cifar-10 CNN: Log loss over iterations (left) and training accuracy (right). Average and $95\%$ confidence interval of 10 runs with random initialization.}}\label{fig:padam}
\end{figure}
\section{CONCLUSION}
\vspace{-1mm}
We undertook an extensive theoretical study of the role of memory in (stochastic) optimization. We provided convergence guarantees for memory systems as well as for novel algorithms based on such systems. This study led us to derive novel insights on momentum methods. We complemented these findings with empirical results, both on simple functions as well as more complex functions based on neural networks. There, long- and short-term memory methods exhibit a different behaviour, which suggests further investigation is needed to better understand the interplay between the geometry of neural networks losses, memory and gradient stochasticity. On a more theoretical side, an interesting direction of future work is the study of the role of memory in state-of-the art momentum methods such as algorithms that include primal averaging, increasing gradient sensitivity or decreasing learning rates (see e.g. \cite{krichene2017acceleration}).


\bibliographystyle{apalike}
\bibliography{refs}
\newpage
\appendix
\onecolumn
{\Huge \textbf{Appendix}}

\section{Basic Definitions and Notation}
\label{notation}
In this paper we work in $\R^d$ with the metric induced by the Euclidean norm, which we denote by $\|\cdot\|$. We say that $f:\R^d\to\R$ is $\C^m(\R^d,\R)$ if it is $m$ times continuously differentiable and we say that it is $L$-Lipschitz if $\|\nabla f(x) - \nabla f(y)\|\le L\|x-y\|$ for all $x,y\in\R^d$. We say that $f(\cdot)$ is $\mu$-strongly convex if $f(y)\ge f(x) + \langle\nabla f(x), y-x \rangle + \frac{\mu}{2} \|y-x\|^2$ for all $x,y\in \R^n$. We call a function convex if it is $0$-strongly convex. An equivalent definition involves the Hessian: a twice differentiable function $f(\cdot)$ is $\mu$-strongly convex and $L$-smooth if and only if, for all $x\in \R^d$ , $\mu I_d \preceq \|\nabla^2 f(x)\|_{s}\preceq L I_d$, where $I_d$ is the identity matrix in $\R^d$ and $\|\cdot\|_{s}$ denotes the operator norm of a matrix in Euclidean space: $\|\nabla^2 f(x)\|_{s} = \sup_{\|y\|= 1} \|\nabla^2 f(x) y\|$. For a symmetric matrix, the operator norm is the norm of the maximum positive eigenvalue.

\section{Background on Momentum and Adam}
\label{sec:app_background}
All the algorithms/ODEs mentioned in this appendix are reported in Sec.~\ref{sec:MGSDE}.
\subsection{Discretization of~\ref{HB-ODE}}
\label{discretization_HB_proof}

The next theorem provides a strong link between \ref{HB} and \ref{HB-ODE}, and in partially included in~\cite{shi2019acceleration}.
\begin{theorem}
The Heavy-Ball is the result of some semi-implicit Euler integration on \eqref{HB-ODE_PS}.
\label{thm:HB_discretization}
\end{theorem}
\begin{proof}
Semi-implicit integration with stepsize $h$, at iteration $k$ and from the current integral approximation $(x_k,v_k)\simeq\left(X(kh),V(kh)\right)$, computes $(x_{k+1},v_{k+1})\simeq \left(X(h(k+1)), V(h(k+1))\right)$ as follows:

\begin{equation}
    \begin{cases}
    v_{k+1} = v_k + h(-a_k v_k -\nabla f(x_k))\\
    x_{k+1} = x_k + h v_{k+1}
    \end{cases},
    \label{eq:HB_int}
\end{equation}

where $a_k = a(hk)$. Notice that $v_{k+1} = \frac{x_{k+1}-x_k}{h}$ and
\begin{equation}
x_{k+1} = x_{k}-(1-a_k h) (x_k-x_{k-1}) - h^2 \nabla f(x_k),
\label{HB_link_continuous}
\end{equation}

which is exactly an Heavy Ball iteration, with $\beta_k = 1-h a_k$ and $\eta = h^2$.
\end{proof}

\begin{remark}
The semi-implicit Euler method ---when applied to an Hamiltonian system--- is \textit{symplectic}, meaning that is preserves some geometric properties of the true solution \cite{hairer2006geometric}. However, as also pointed out in \cite{zhang2018direct}, continuous time models of momentum methods are energy-dissipative ---hence not Hamiltonian. Therefore, as opposed to \cite{betancourt2018symplectic, shi2019acceleration}, we avoid using this misleading moniker.
\label{rmk:symplectic}
\end{remark}

Moreover, the last result also has an integral formulation.
\begin{proposition}
The differential equation \ref{HB-ODE-INT-C} is the continuous time limit of \ref{HB-SUM}.
\end{proposition}
\begin{proof}
Recalling that $\beta = 1-h\alpha$ (where $h$ is the stepsize of the semi-implicit Euler integration defined in Thm.~\ref{thm:HB_discretization}) and choosing $t = kh$ ($k > 0$) and $s = jh$ ($j > 0$),
$$\lim_{h\to 0} \beta^{k-j} = \lim_{h\to 0} (1-h\alpha)^{(t-s)/h}=\lim_{h\to 0} \left((1-h\alpha)^{1/h}\right)^{(t-s)}= \left(e^{\alpha}\right)^{(t-s)} = e^{\alpha (t-s)}.$$
Moreover, since $\eta=h^2$, taking one $h$ inside the summation, we get a Riemann sum which then rightfully converges to the integral in the limit. 
\end{proof}

\subsection{Unbiasing \ref{HB-SUM} under constant momentum: the birth of Adam}
\label{sec:unbiasing}

\cite{kingma2014adam} noticed that, in the limit case where true gradients are constant, the averaging procedure in \ref{HB-SUM} (defined in the Introduction section of the main paper) is biased: let $i_j\in\{1,\dots, N\}$ be the data-point selected at iteration $j$ and $\nabla f_{i_j}$ the corresponding stochastic gradient; if we define $\epsilon_j:= \nabla f(x_j)-\nabla f_{i_j}(x_j)$ and pick constant momentum $\beta_j = \beta$, we have
$$\E\left[\sum_{j=0}^{k} \beta^{k-j} \nabla f_{i_j}(x_j)\right] = \sum_{j=0}^{k} \beta^{k-j} \E[\nabla f + \epsilon_j] = \frac{\beta^{k+1}-1}{\beta -1} \nabla f. $$
Therefore ---to ensure an unbiased update, at least for this simple case--- \cite{kingma2014adam} normalize the sum above  by $\frac{\beta^{k+1}-1}{\beta -1}$, showing significant benefits in the experimental section. Indeed, such normalization retains all the celebrated geometric properties of momentum (see Introduction), while improving statistical accuracy  ---a crucial feature of SGD. For convenience of the reader, we report below the full Adam algorithm.

 \begin{tcolorbox}
 Initialize $m_0=v_0=0$ and choose initial estimate $x_0$. Let "$\circ$" denote the element-wise product.
 \begin{equation}
 \tag{ADAM}
     \begin{cases} 
      m_{k+1} &= \beta_1 m_k + (1-\beta_1) \nabla f(x_k) \\
      v_{k+1} &= \beta_2 v_k + (1-\beta_2) \nabla f_k^{\circ 2}(x_k) \\
        \hat m_{k+1} &= m_{k+1}/ (1-\beta_1)^{k+1} \\
            \hat v_{k+1} &= v_{k+1}/ (1-\beta_2)^{k+1} \\
                       x_{k+1} &= x_{k} - \eta \frac{\hat m_k}{\sqrt{\hat v_k +\epsilon}} \\
   \end{cases}
   \label{eq:adam}
\end{equation}
 \end{tcolorbox}

In addition, such normalization also performs variance reduction. Indeed, let $\Sigma := \var\left[\nabla f_{i}(x)\right]$ be the gradient covariance, which we assume constant and finite for simplicity. Then,
\begin{align*}
    \var\left[\frac{\beta-1}{\beta^{k+1}-1}\sum_{j=0}^{k} \beta^{k-j} \nabla f_{i_j}(x_j)\right]&= \frac{(\beta-1)^2}{(\beta^{k+1}-1)^2}\sum_{j=0}^{k} \beta^{2(k-j)} \Sigma\\ &= \frac{(\beta-1)^2}{(\beta^{k+1}-1)^2}\frac{\beta^{2(k+1)}-1}{\beta^2-1}\Sigma\\&= \frac{\beta-1}{(\beta^{k+1}-1)}\frac{\beta^{k+1}+1}{\beta+1}\Sigma \preceq \Sigma.\\
\end{align*}

Note that if $k$ increases, the covariance monotonically decreases until reaching the minimum $\frac{1-\beta}{\beta+1}\Sigma$ at infinity. We also note that, in case $\beta = 0$ (Gradient Descent) we have no variance reduction, as expected.

Supported by the empirical success of Adam\cite{kingma2014adam}, we suggest in the main paper to modify the standard Heavy Ball (\ref{HB-SUM}) with constant momentum to match the Adam update:
 \begin{equation}
    x_{k+1} = x_k - \eta \frac{\beta-1}{\beta^{k+1}-1} \sum_{j=0}^{k} \beta^{k-j} \nabla f(x_j).
     \label{HB-sum_modified}
 \end{equation}
 
 Which can be written recursively using 3 variables (see again \cite{kingma2014adam}). Notice that, after a relatively small number of iterations, we converge to the simpler update rule
 $$x_{k+1} = x_k + \beta (x_k-x_{k-1}) - \eta(1-\beta)\nabla f(x_k),$$
 
 which is also the starting point in some recent elaboration on Heavy Ball \cite{ma2018quasi}.

\section{Time-warping, acceleration and gradient amplification}

\subsection{Counterexample for existence of a solution of \ref{MG-ODE} starting integration at 0}
\label{app:counterexample}
Consider the one dimensional dynamics with gradients always equal to one and $\m(t) = t^3$. The \ref{MG-ODE} (see main paper) reads $\ddot X(t) + \frac{3}{t} \dot X + \frac{3}{t} = 0$. It is easy to realize using a Cauchy-Euler argument that all solutions are of the form $X(t) = \frac{c_1}{t^2} + c_2 - t$. Unfortunately, the constraint $X(0)=x_0$ fixes both the degrees of freedom and fixes $X(t) = x_0-t$, so that we necessarily have $\dot X(0) = -1$.

  \subsection{Variational point of view}
  \label{lagrangian}
Let $X\in\C^1([t_1,t_2],\R^d)$ be a curve; the action (see \cite{arnol2013mathematical} for the precise definition) associated with a Lagrangian $\mathcal{L}:\R^d\times\R^d\times\R\to\R$ is $\int_{t_1}^{t_2}\mathcal{L}(X(s),\dot X(s), s) ds$. The fundamental result in variational analysis states that the curve $X$ is a stationary point for the action only if it solves the Euler-Lagrange equation $\frac{d}{dt}\partial_{\dot X} \mathcal{L}(X(t),\dot X(t), t) = \partial_{X} \mathcal{L}(X(t),\dot X(t), t)$. We define the \textit{Memory Lagrangian} as
$$\boxed{\mathcal{L}_\m(X,\dot X,t) := \frac{1}{2}\m(t)\|\dot X\|^2-\dot\m(t) f(X)}.$$
It is straightforward to verify that the associated Euler-Lagrange equations give rise to \ref{MG-ODE}. We cannot help but noticing the striking simplicity of this Lagrangian when compared to others arising from momentum methods (see e.g. the Bregman Lagrangian in \cite{wibisono2016variational}) . For the sake of delivering other points in this paper, we leave the analysis of  the symmetries of $\mathcal{L}_\m$ to future work.

\subsection{Time-warping of memory: a general correspondence to \ref{HB-ODE}}
\label{time}
Consider the ODE $\boxed{\ddot X(t) + \frac{\dot\m(t)}{\m(t)}\dot X(t) + \frac{\dot\m(t)}{\m(t)}\nabla f(X)=0}$ and the time change $\tau(t)$. 

Let $Y(t) = X(\tau(t))$, following \cite{wibisono2016variational} we have, by the chain rule, 
\begin{equation*}
    \dot Y(t) = \dot X(\tau(t))\dot\tau(t);\ \ \ \ \ \ \ \
    \ddot Y(t) = \dot X(\tau(t))\ddot\tau(t) + (\dot\tau(t))^2\ddot X(\tau(t)).\\
\end{equation*}

Therefore, we have
\begin{equation*}
    \dot X(\tau(t)) = \frac{\dot Y(t)}{\dot\tau(t)};\ \ \ \ \ \ \ \
    \ddot X(\tau(t)) = \frac{\ddot Y(t)}{(\dot\tau(t))^2}-\frac{\ddot \tau(t)}{(\dot\tau(t))^3} \dot Y(t).\\
\end{equation*}

Next, we construct a new ODE for $Y$ using the previous formulas:
$$\frac{1}{(\dot\tau(t))^2}\ddot Y(t) + \left(\frac{\dot\m(\tau(t))}{\dot\tau(t) \m(\tau(t))} - \frac{\ddot \tau(t)}{(\dot\tau(t))^3}\right)\dot Y(t) + \frac{\dot \m(\tau(t))}{\m(\tau(t))}\nabla f(X(\tau(t)));$$

Next, we multiply everything by $\m(\tau(t))/\dot\m(\tau(t))$, in order to eliminate the coefficient in front of the gradient:

$$\frac{\m(\tau(t))}{\dot\m(\tau(t))(\dot\tau(t))^2}\ddot Y(t) + \left(\frac{1}{\dot\tau(t)} - \frac{\m(\tau(t))\ddot \tau(t)}{\dot\m(\tau(t))(\dot\tau(t))^3}\right)\dot Y(t) + \nabla f(Y(t))$$

If we also want the coefficient in front of $\ddot Y(t)$ to be one, we need $\tau(t)$ to satisfy the following differential equation:
$$\m(\tau(t)) = \dot\m(\tau(t))(\tau(t))^2.$$

For simplicity, let us consider $\m(t) = t^p$, for some $p$ (polynomial forgetting), the equation reduces to $\tau(t) = p (\dot \tau(t))^2$, which has the general solution 
$$\tau(t) = \frac{1}{4p}\left(-2\sqrt{2} c_1 t + 2c_1^2 + t^2\right).$$
To be a valid time change, we need to have $\tau(0)=0$; hence, --- \textit{only one choice is possible}: $\tau(t) = \frac{t^2}{4p}$. Plugging in this choice into the differential equation, we get 
$$\boxed{\ddot Y(t) + \frac{2p-1}{t}\dot Y(t) + \nabla f(Y(t))=0},$$
which is of the form of \ref{HB-ODE}.

In all this, it is the time change is \textit{fixed to be} $\frac{t^2}{4p}$. For this reason, we postulate that this time change analysis has deep links to the general mechanism of acceleration. We verify this change of time/ODE formula in App. \ref{app:ODEs_quadratic}.

\subsection{Behaviour of Nesterov's SDE compared to the quadratic forgetting SDE}
\label{app:nesterov_vs_quadratic_noise}

We compare here the variance of the Nesterov's SDE to the variance of the quadratic forgetting SDE.

\subsubsection{A general result}
First, we need to prove a result in stochastic integration~(for the definition of integral w.r.t. a Brownian Motion the reader can check \cite{mao2007stochastic}).

\begin{lemma}
Let $\{B\}_{t\ge 0}$ be a $d-$dimensional Brownian Motion,
$$\var\left(\int_0^t s^p dB(s)\right) = \frac{t^{2p+1}}{2p+1} I_{d}.$$
\label{lemma:isometry}
\end{lemma}

\begin{proof}
First, notice that $\E\left[\int_0^t s^p dB(s)\right]=0$; therefore the variance is equal to the second moment:
$$\var\left(\int_0^t s^p dB(s)\right) = \E\left[\left(\int_0^ts^p dB(s)\right)^2\right] .$$
By the It\^{o} isometry (see e.g. \cite{mao2007stochastic}).
$$\E\left[\left(\int_0^ts^p dB(s)\right)^2\right] = \int_0^t s^{2p} dt = \frac{t^{2p+1}}{2p+1}.$$
\end{proof}

We want to compare the SDEs below.
\begin{equation*}
    \tag{quadratic forgetting SDE}
    \boxed{
    \begin{cases}
    dX(t) = V(t) dt\\ 
    dV(t) = -\frac{3}{t}V(t)dt -\frac{3}{t}\left[\nabla f(X(t))dt +\sigma(X(t)) dB(t)\right]
    \end{cases}}
\end{equation*}

\begin{equation*}
    \tag{Nesterov's SDE}
    \boxed{
    \begin{cases}
    dX(t) = V(t) dt\\ 
    dV(t) = -\frac{3}{t}V(t)dt -\left[\nabla f(X(t))dt +\sigma(X(t)) dB(t)\right]
    \end{cases}} \ \ \ \ \ \ \ \ \ \ \ \ \ \ \ \ \ \
\end{equation*}

We have the following result, which is included in Sec.~\ref{nesterov_comparison} of the main paper.
\begin{proposition}
Assume persistent volatility $\sigma(X(t))= \sigma$.
Let $\{X_N(t), V_N(t)\}_{t\ge0}$ be the stochastic process which solves Nesterov's SDE. The infinitesimal update direction $V_N(t)$ of the position $X_N(t)$ can be written as
\begin{equation*}
    V_{N}(t) =  -\int_0^t \frac{s^{3}}{t^3}\nabla f(X_N(s))ds  +\zeta_{N}(t),
\end{equation*}
where $\zeta_{N}(t)$ is a random vector with $\E[\zeta_{N}(t)]=0$ and $\var[\zeta_{N}(t)]=\frac{1}{7} t\sigma\sigma^T$. In contrast, the solution $\{X_{\m2}(t), V_{\m2}(t)\})_{t\ge 0}$ of \ref{MG-SDE} with quadratic forgetting satisfies
\begin{equation*}
    V_{\m2}(t) =  -\int_0^t \frac{3s^{2}}{t^3}\nabla f(X_{\m2}(s))ds  +\zeta_{\m2}(t),
\end{equation*}
where $\zeta_{\m2}(t)$ is a random vector with $\E[\zeta_{\m2}(t)]=0$ but $\var[\zeta_{\m2}(t)]=\frac{9}{5t}\sigma\sigma^T$.
\label{prop:var}
\end{proposition}

\ \\
\begin{remark} [implications of the proposition]
Note that, in the result above, the time integrals themselves are random variables; therefore, in general, $\var[V_{\m2}(t)] \ne \var[\zeta_{\m2}(t)]$ and $\var[V_{N}(t)] \ne\var[\zeta_{N}(t)]$. Therefore, the result does \underline{not} directly imply that $\var[V_{N}(t)]$ explodes. However, if gradients are constant, then clearly $\var[V_{N}(t)]$ diverges since the integrals are deterministic (see Fig.\ref{fig:verif_noise}). A more careful analysis, presented in App.~\ref{sec:quadratic_variance}, shows that this fact also holds in the quadratic convex case.
\end{remark}

\begin{proof}
 Let us consider first the Quadratic forgetting SDE. Define the function $Q(v,t) = t^3 v$, which has Jacobian $\partial_vQ(v,t) = t^3 I_d$. Using It\^{o}'s Lemma (Eq.~\eqref{ITO}) coordinate-wise, we get
\begin{align}
    dQ(V(t),t) &= \partial_t Q(V(t),t) dt+ \left\langle\partial_v Q(V(t),t),-\frac{3}{t}V(t)dt -\frac{3}{t}\nabla f(X(t))dt \right\rangle + \left\langle\partial_v Q(V,t),-\sigma dB(t) \right\rangle\nonumber\\
    &= \cancel{3t^{2} V(t) dt}-\cancel{3t^{2}V(t)dt} -3t^{2}\nabla f(X(t))dt  -3t^{2} \sigma dB(t),\label{middle_divergence_proof}
\end{align}
which implies, after taking the stochastic integral,
$$t^3V(t) = Q(V(t),t) - Q(V(0),0) = -\int_0^t 3s^{2}\nabla f(X(s))ds  - \int_0^t 3s^{2} \sigma dB(s).$$

Therefore, for any $t\ge 0$,
$$V_{2\m}(t) = -\int_0^t \frac{3s^{2}}{t^3}\nabla f(X(s))ds  - \int_0^t \frac{3s^{2}}{t^3} \sigma dB(s).$$

Let us call $\zeta_{2\m}(t)$ the stochastic integral, by Lemma~\ref{lemma:isometry}, 
$$\var\left[\zeta_{2\m}(t)\right] = \frac{9}{t^6}\frac{t^5}{5}\sigma\sigma^T = \frac{9}{5 t}\sigma\sigma^T.$$

If we apply the same procedure to Nesterov's SDE, we instead get
$$V_{N}(t) = -\int_0^t \frac{s^{3}}{t^3}\nabla f(X(s))ds  - \int_0^t \frac{s^3}{t^3} \sigma dB(s) = -\int_0^t \frac{s^{3}}{t^3}\nabla f(X(s))ds  - \zeta_{2\m}(t),$$
where $\var\left(\zeta_{2\m}(t)\right) = \frac{t}{7}\sigma\sigma^T.$
\end{proof}
\begin{figure}
  \centering
    \includegraphics[width=0.52\linewidth]{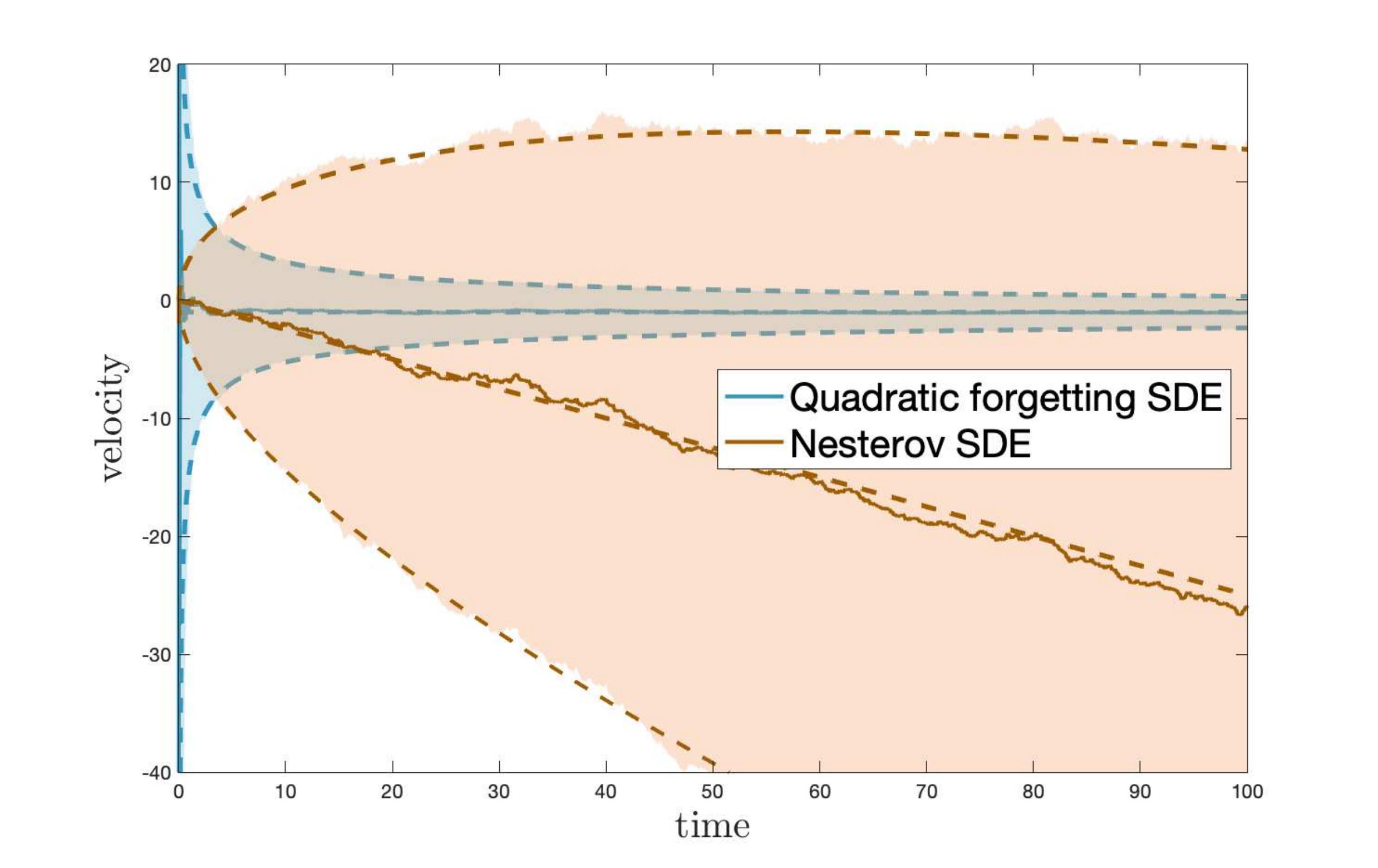} 
    \caption{Simulation using Milstein scheme \cite{mil1975approximate} (stepsize $10^{-3}$, equivalent for this case to Euler-Maruyama) to verify gradient amplification and covariance explosion under a constant gradient equal to 1, $\sigma = 10$. Plotted is the empirical mean and standard deviation of the velocity variable using 1000 runs. The dashed lines indicates the prediction from Prop.~\ref{prop:var}. \ref{MG-SDE} with quadratic forgetting quickly settles to the value $-1$ (the true negative gradient) with decreasing variance. Nesterov's SDE amplifies linearly the negative gradient, and such amplification makes the velocity noisy, with exploding variance.}
    \label{fig:verif_noise}
\end{figure}

\begin{remark}[effect of starting integration after 0]
Effect of starting integration after 0. We take the chance here to explain what changes if we start integration at $\epsilon>0$ with $V_\epsilon(\epsilon)=0$ and $X_\epsilon(\epsilon) = x_0$. Starting from Eq.~\eqref{middle_divergence_proof}, which is still valid, we have to integrate on $[\epsilon, t].$

we have, after taking the stochastic integral,
$$t^3V_\epsilon(t) - \epsilon^3V_\epsilon(\epsilon) = -\int_{\epsilon}^t 3s^{2}\nabla f(X_\epsilon(s))ds  - \int_{\epsilon}^t 3s^{2} \sigma dB(s).$$

Notice that $V_{\epsilon}(\epsilon)=0$; therefore

$$V(t) = -\int_{\epsilon}^t \frac{3s^{2}}{t^3}\nabla f(X(s))ds  - \int_{\epsilon}^t 3s^{2} \sigma dB(s).$$

Notice that, for all $t\ge 0$, the integral dependency on $\epsilon$ vanishes as $\epsilon$ goes to zero. More explicitly, this can be seen using a change of variable in the integral and considering \textit{any} fixed function $X(s)$.
\label{rmk1_exi}
\end{remark}

\subsubsection{Variance divergence of Nesterov's SDE in the quadratic case}
\label{sec:quadratic_variance}

Consider $f(x) = \frac{1}{2}\langle x-x^*,H(x-x^*)\rangle$ for some positive semidefinite $H$. Without loss of generality, we can assume $H$ to be diagonal and $x^* = 0_d$, the $\R^d$ vector of all zeros. Then, in the quadratic-forgetting SDE and Nesterov's SDE, each direction in the original space evolves separately and is decoupled from the others. In other words, the problem becomes linear and two dimensional. 

Let us perform the analysis for Nesterov first. We call $\{X_N(t), V_N(t)\}_{t\ge0}$, the solution to Nesterov's SDE and denote by $X_N^i(t)$ and $V_N^i(t)$ the i-th coordinates of the space and velocity variables and by $\lambda_i\in\R_+$ the eigenvalue of $H$ in the i-th eigendirection. We want to show that $\E[(V_N^i(t))^2]$ explodes. The pair $(X_N^i(t), V_N^i(t))$ evolves with the SDE

\begin{equation*}
    \begin{cases}
    dX_N^i(t) = V_N^i(t) dt\\ 
    dV_N^i(t) = -\frac{3}{t}V_N^i(t)dt - \lambda_i X_N^i dt +\sigma dB(t),
    \end{cases}
\end{equation*}
where $\{B(t)\}_{t\ge 0}$ is a one-dimensional Brownian Motion. This SDE is linear and can be written in matrix form:

\begin{equation*}
    \begin{pmatrix}dX_N^i(t)\\
    dV_N^i(t)\end{pmatrix} = \begin{pmatrix} 0 & 1\\ -\lambda_i & -3/t\end{pmatrix}\begin{pmatrix}X_N^i(t)\\
    V_N^i(t)\end{pmatrix}dt + \begin{pmatrix} 0 \\ \sigma\end{pmatrix} dB(t).
\end{equation*}

By the stochastic variations-of-constants formula (see e.g. Sec. 3.3 in~\cite{mao2007stochastic}), the matrix of second moments (uncentered covariance)
$$\begin{pmatrix} \E[(X_N^i(t))^2]& \E[X_N^i(t)V_N^i(t)]\\ \E[X_N^i(t)V_N^i(t)] & \E[(V_N^i(t))^2]\end{pmatrix} =: \begin{pmatrix}p_1(t) & p_2(t) \\ p_2(t)& p_3(t) \end{pmatrix}$$
 evolves with the following matrix ODE
 
 $$\begin{pmatrix}\dot p_1(t) & \dot p_2(t) \\ \dot p_2(t)& \dot p_3(t) \end{pmatrix} = \begin{pmatrix} 0 & 1\\ -\lambda_i & -3/t\end{pmatrix}\begin{pmatrix}p_1(t) & p_2(t) \\ p_2(t)& p_3(t) \end{pmatrix} + \begin{pmatrix}p_1(t) & p_2(t) \\ p_2(t)& p_3(t) \end{pmatrix} \begin{pmatrix} 0 & 1\\ -\lambda_i & -3/t\end{pmatrix}^T + \begin{pmatrix} 0 & 0\\ 0&  \sigma^2\end{pmatrix},$$
 
 with initial condition 
 $\begin{pmatrix}1 & 0 \\ 0& 0 \end{pmatrix}$. This translates in a system of linear time-dependent ODEs

\begin{equation}
    \begin{cases}
        \dot p_1(t) = 2p_2(t)\\
        \dot p_2(t) = -\lambda_i p_1(t) -\frac{3}{t} p_2(t)  +p_3(t)\\
        \dot p_3(t) = -2\lambda_i p_2(t)  -\frac{6}{t} p_3(t)  + \sigma^2\\
    \end{cases}.
    \label{eq:var-nag}
\end{equation}

This system can be easily tackled numerically with accurate solvers such as MATLAB \texttt{ode45}. We show the integrated variables in Fig.\ref{fig:quadratic_variance}, and compare them with the ones relative to quadratic-forgetting, which can be shown to solve a similar system:
\begin{equation}
    \begin{cases}
        \dot p_1(t) = 2p_1(t)\\
        \dot p_2(t) = -\frac{3\lambda_i}{t} p_1(t) -\frac{3}{t} p_2(t)  +p_3(t) \\
        \dot p_3(t) = -\frac{6\lambda_i}{t} p_2(t)  -\frac{6}{t} p_3(t)  + \frac{3}{t}\sigma^2\\
    \end{cases}.
    \label{eq:var-qf}
\end{equation}

\begin{figure}
  \centering
    \includegraphics[width=0.7\linewidth]{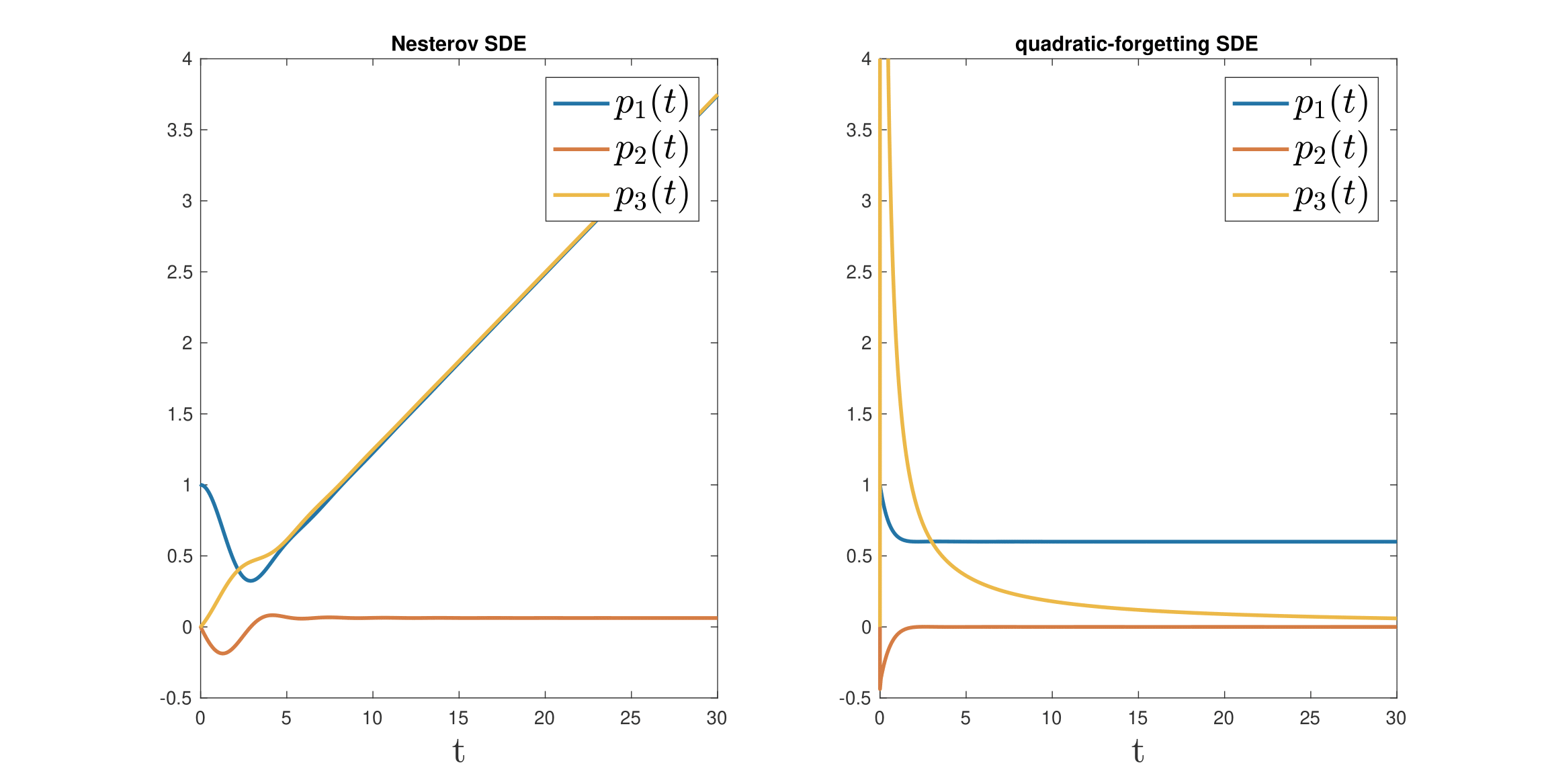} 
    \caption{Numerical solution with MATLAB \texttt{ode45} to Eq.~\eqref{eq:var-nag} (on the left) and Eq.~\eqref{eq:var-qf} (on the right). We choose $\lambda = \sigma = 1$, but the results do not change qualitatively (only scale the axis) for $\sigma,\lambda_i\in\R_+$.}
    \label{fig:quadratic_variance}
\end{figure}

\begin{figure}
  \centering
    \includegraphics[width=0.48\linewidth]{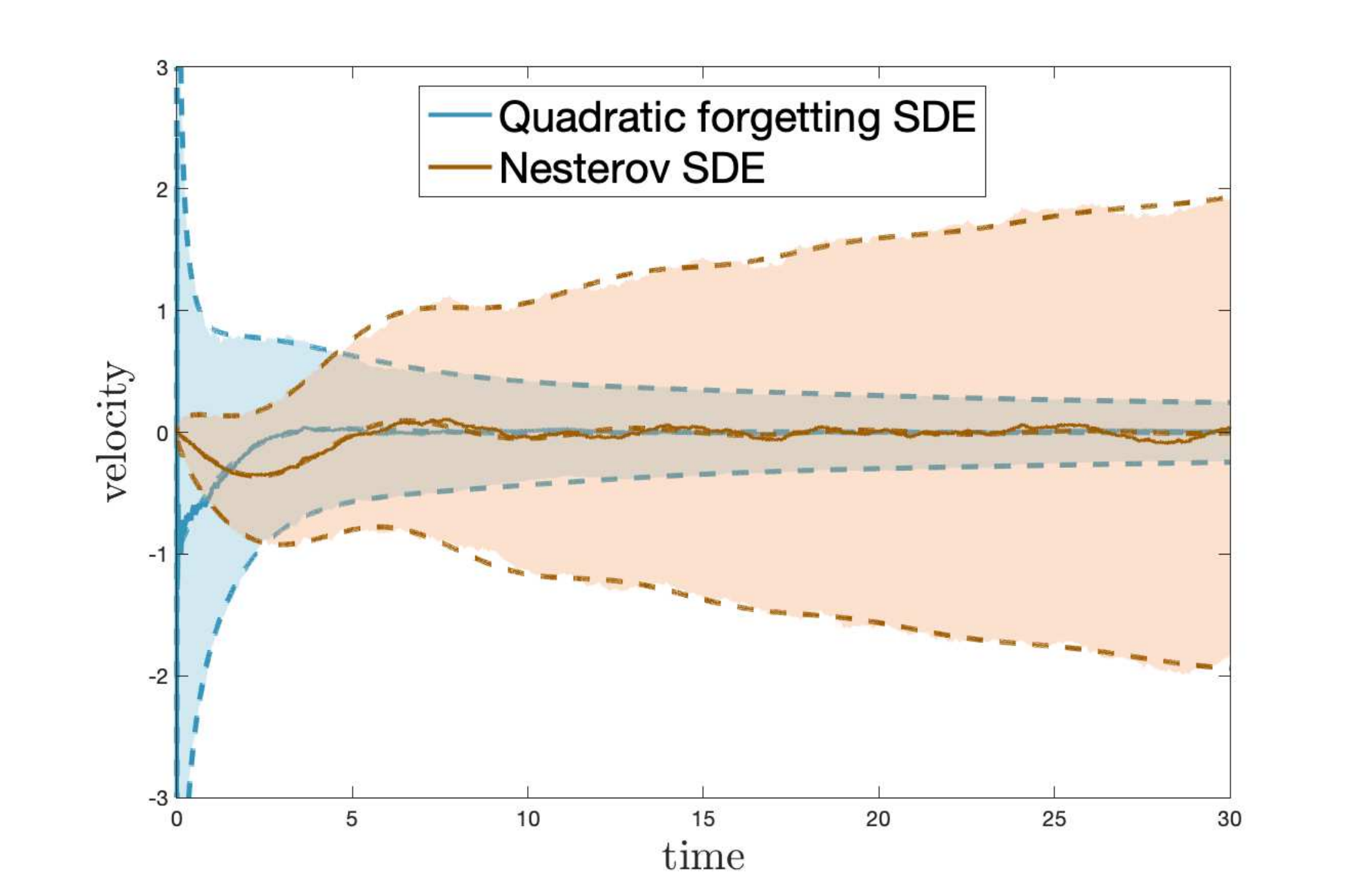} 
    \caption{Simulation using Milstein scheme \cite{mil1975approximate} (stepsize $10^{-3}$, equivalent for this case to Euler-Maruyama) to verify covariance explosion of Nesterov's SDE under a quadratic one dimensional cost $f(x) = x^2/2$. Plotted is the empirical mean and standard deviation of the velocity variable using 1000 runs. The dashed lines indicate the predictions from Eq.~\eqref{eq:var-nag} and Eq.~\eqref{eq:var-qf} (both solved numerically with MATLAB \texttt{ode45} along with standard Nesterov's ODE to get the mean). \ref{MG-SDE} with quadratic forgetting has decreasing variance. Nesterov's SDE amplifies linearly the negative gradient, and such amplification makes the velocity noisy, with exploding variance.}
    \label{fig:verif_noise_quadratic}
\end{figure}

From the simulation results we conclude that, in the convex quadratic setting, the covariance of $V_N(t)$ explodes, while the one of $V_{\m2}(t)$ does not: it converges. Indeed, as we show formally in App.~\ref{proofs}, for quadratic forgetting we get convergence to a ball around the solution. We validate these findings with a numerical simulation in Fig.~\ref{fig:verif_noise_quadratic}.

\newpage
\section{Proofs of convergence rates for the memory SDE}
\label{proofs}

We start by refreshing the reader's knowledge in stochastic calculus. Basic definitions (SDEs, stochastic integrals, etc) can be found in~\cite{mao2007stochastic}.

\subsection{Stochastic calculus for the memory SDE}
Consider the memory SDE

\begin{equation*}
    \begin{cases}
    dX(t) = V(t) dt\\ 
    dV(t) = -\frac{\dot \m(t)}{\m (t)}V(t)dt -\frac{\dot \m(t)}{\m (t)}\nabla f(X(t))dt - \frac{\dot \m(t)}{\m (t)}\sigma(X(t),t) dB(t)
    \end{cases}.
\end{equation*}

We can rewrite this in vector notation ($0_{d\times d}$ is the $d\times d$ of all zeros)
\begin{align}
    \begin{pmatrix}
    dX(t)\\
    dV(t)
    \end{pmatrix} &=
    \begin{pmatrix}
    V(t)\nonumber\\
    -\frac{\dot \m(t)}{\m (t)}V(t)-\frac{\dot \m(t)}{\m (t)}\nabla f(X(t))
    \end{pmatrix}dt + 
    \begin{pmatrix}
    0_{d\times d} & 0_{d\times d}\\
     0_{d\times d} & -\frac{\dot \m(t)}{\m (t)}\sigma(X(t),t)
    \end{pmatrix} dB(t)\nonumber\\
    &= b(X(t),V(t),t) dt + \xi(X(t),V(t),t) dB(t), \label{eq:memory_SDE_general}
\end{align}

where $\{B(t)\}_{t\ge 0}$ is a d-dimensional Brownian Motion. We write this for simplicity as $b(t) dt + \xi(t) dt$.

Let $\EE:\R^d\times\R^d\times\R\to \R$ be twice continuously differentiable jointly in the first two variables (indicated as $x$ and $v$) and continuously differentiable in the last (which we indicate as $t$). Then, by It\^{o}'s lemma \cite{mao2007stochastic},  the stochastic process $\{\EE(X(t),V(t),t)\}_{t\ge0}$ satisfies the following SDE:
\begin{multline}
d\EE(X(t),V(t),t) = \partial_t \EE (X(t),V(t), t)) dt + \langle\partial_{(x,v)} \EE(X(t),V(t),t), b(t)\rangle dt\\
        +  \frac{1}{2}\tr\left(\xi(t)\xi(t)^T\partial^2_{(x,v)}\EE(X(t),V(t),t) \  \right)dt+ \langle\partial_{(x,v)}\EE(X(t),V(t),t),\xi(t) dB(t)\rangle.
        \label{ITO}
\end{multline}

where $\partial_{(x,v)}$ is the partial derivative with respect to $(x,v)$ and $\partial^2_{(x,v)}$ the matrix of second derivatives with respect to $(x,v)$. Notice that, in the deterministic case $\xi(t)=0$, the equation reduces to standard differentiation using the chain rule:

 $$\frac{d\EE(X(t),V(t),t)}{dt} = \partial_t \EE (X(t),V(t) t)) + \langle\partial_{(x,v)} \EE(X(t),V(t),t), b(t)\rangle. $$

Following \cite{mao2007stochastic}, we introduce the \textbf{It\^{o} diffusion differential operator} $\mathscr{A}$ associated with Eq.~\eqref{eq:memory_SDE_general} acting on a scalar function :
\begin{equation}
\mathscr{A}(\cdot)= \partial_t(\cdot) + \langle\partial_{(x,v)} (\cdot), b(t)\rangle + \frac{1}{2}\tr\left(\xi(t)\xi(t)^T\partial^2_{(x,v)}(\cdot)\right).
    \label{eqn:SDE_diffusion_op}
\end{equation}

It is then clear that, thanks to It\^{o}'s lemma,
$$d\EE(X(t),V(t),t) = \mathscr{A}\EE(X(t),V(t),t) dt + \langle\partial_{(x,v)}\EE(X(t),V(t),t),\xi(t)dB(t)\rangle.$$

The It\^{o} diffusion differential operator generalizes the concept of derivative: in fact, with slight abuse of notation\footnote{The expectation of a differential is not well defined. However, if by $\E[d\EE(t)]$ we mean a small change in $\E[\EE(t)]$ over a "small" period of time $dt$, we can write understand the intuition behind this writing using Dynkin's formula (Eq.~\eqref{eq:dynkin}), presented below.}

$$\frac{\E\left[d\EE(X(t),V(t),t)\right]}{dt} = \mathscr{A}\EE(X(t),V(t),t).$$
Moreover, by the definition of the solution to an SDE (see \cite{mao2007stochastic}), we know that at any time $t>0$,
$$\EE(X(t),V(t),t) = \EE(x_0,v_0,0) + \int_0^t \mathscr{A}\EE(X(s),V(s),s) ds + \int_0^t \partial_{x}\EE(X(s),V(s),s)^T\sigma(s) dB(s).$$

Taking the expectation the stochastic integral vanishes\footnote{see e.g. Thm. 1.5.8 \cite{mao2007stochastic}} and we have

\begin{equation}
    \E[\EE(X(t),V(t),t)]-\EE(x_0,0) = \E \left[\int_0^t \mathscr{A}\EE(X(s),V(s),s) ds\right].
    \label{eq:dynkin}
\end{equation}

This result is known as \textbf{Dynkin's formula} and generalizes the fundamental theorem of calculus to the stochastic setting.

\subsection{Convergence rates for polynomial forgetting}

We recall our assumptions below.

\textbf{(H0c)} \ \ \  $f\in\C^3_b(\R^d,\R), \ \sigma_*^2 := \sup_{x} \|\sigma(x)\sigma(x)^T\|_s < \infty$.

\textbf{(H1')} \ \ \ $f:\R^d\to\R$ is $L$-smooth and there exist $x^*\in\R^d$ s.t. for all $x\in\R^d$, $\langle \nabla f(x), x-x^*\rangle\ge \tau (f(x)-f(x^*))$.

The last condition is known as weak-quasi-convexity, and generalizes convexity (convex functions are wqc with constant 1 ~\cite{hardt2018gradient}).
The next fundamental lemma can also be found in \cite{krichene2017acceleration}.
\begin{lemma}
Consider two symmetric $d-$dimensional square matrices $P$ and $Q$. We have
$$\tr(PQ) \le  d\cdot \|P\|_s\cdot\|Q\|_s,$$
where $\|\cdot\|_s$ denotes the spectral norm.
\label{lemma:CONV_trace_product}
\end{lemma}
\begin{proof}
Let $P_j$ and $Q_j$ be the j-th row(column) of $P$ and $Q$, respectively.
\begin{align}
    \tr(PQ) &= \sum_{j =1}^d P_j^TQ_j \le \sum_{j =1}^d \|P_j\|\cdot\|Q_j\| \nonumber \le \sum_{j =1}^d \|P\|_s\cdot\|Q\|_s = d\cdot \|P\|_s\cdot\|Q\|_s,
\end{align}
where we first used the  Cauchy-Schwarz inequality, and then the following inequality:
$$\|A\|_s = \sup_{\|z\|\le 1}\|Az\|\ge \|Ae_j\| = \|A_j\|,$$
where $e_j$ is the j-th vector of the canonical basis of $\mathbb{R}^d$.
\end{proof}

We start with a general result about convergence of the memory SDE (see Sec.~\ref{sec:MGSDE} of the main paper) for arbitrary memory $\m(\cdot)$ in the stochastic setting. We will then specialize this result to polynomial forgetting.

\begin{lemma}[Continuous-time Convex Master Inequality] Assume \textbf{(H0)} and \textbf{(H1')} hold. Let $\{X(t),V(t)\}_{t\ge 0}$ be a solution to \ref{MG-SDE} with memory $\m(\cdot)$, define $\lambda(t) := -\m(t)\int\frac{1}{\m(t)}dt$ (where $\int$ denotes the antiderivative\footnote{Equivalently, $\lambda$ is such that $\dot \lambda(t) = \frac{\dot\m(t)}{\m(t)}\lambda(t)-1$.}) and $r(t) := \frac{\dot\m(t)}{\m(t)}\lambda(t)^2$. If $\m(\cdot)$ is such that $\dot r(t)\le \tau \lambda(t)\frac{\dot\m(t)}{\m(t)}$, then for any $t> 0$
$$\E[f(X(t))-f(x^*)] \le \frac{r(0)(f(x_0)-f(x^*)) + \frac{1}{2}\|x_0-x^*\|^2}{r(t) } + \ \frac{d\sigma^2_*}{2} \  \frac{\int_0^t \lambda(s)^2 \left(\frac{\dot\m}{\m}(s)\right)^2 ds}{r(t)}.$$
\label{lemma:app_master_inequality}
\end{lemma}

\begin{proof}
Consider the following Lyapunov function, inspired  from \cite{su2014differential}:
$$\EE(x,v,t) = r(t) (f(x)-f(x^*)) + \frac{1}{2}\|x-x^* + \lambda(t) v\|^2,$$
where $r:\R\to\R$ and $\lambda:\R\to\R$ are two differentiable functions which we will fix during the proof.
First, we find a bound on the infinitesimal diffusion generator of the stochastic process $\{\EE(X(t),V(t),t)\}_{t\ge0}$. Ideally, we want this bound to be independent of the dynamics (i.e the solution $\{(X(t), V(t))\}_{t\ge 0}$) of the problem, so that we can integrate it  and get a rate. 

By It\^{o}'s lemma, we know that
\begin{align*}
\mathscr{A}\EE(X(t),V(t),t) &= \partial_t \EE (X(t),V(t), t)) dt + \langle\partial_{(x,v)} \EE(X(t),V(t),t), b(t)\rangle dt\\& \ \ \ \ +  \frac{1}{2}\tr\left(\xi(t)\xi(t)^T\partial^2_{x}\EE(X(t),V(t),t) \  \right)dt.\\
\end{align*}

Hence, plugging in the SDE definition and the definition of $\EE$,

\begin{align*}
    \mathscr{A}\EE(X(t),V(t),t)&= \partial_t \EE (X(t),V(t), t)) dt\\  & \ \ \ \ +  \langle\partial_{x} \EE(X(t),V(t),t), V\rangle dt\\
    & \ \ \ \ + \left\langle\partial_{v} \EE(X(t),V(t),t), -\frac{\dot \m(t)}{\m (t)}V(t)-\frac{\dot \m(t)}{\m (t)} \nabla f(X(t)) \right\rangle dt\\
    & \ \ \ \ +  \frac{1}{2}\tr\left(\xi(t)\xi(t)^T\partial^2_{(x,v)}\EE(X(t),V(t),t) \  \right)dt\\
    & = \dot r(t) (f(X(t))-f(x^*))dt + 2 \langle X-x^* + \lambda(t) V, \dot \lambda(t) V\rangle dt\\
    & \ \ \ \ + r(t) \langle \nabla f(X), V\rangle dt +  \langle X-x^* + \lambda(t) V, V\rangle dt\\
    & \ \ \ \ +  \lambda(t)\left\langle X-x^* + \lambda(t) V, -\frac{\dot \m(t)}{\m (t)}V(t)-\frac{\dot \m(t)}{\m (t)} \nabla f(X)\right\rangle dt\\
    & \ \ \ \ +  \frac{1}{2}\tr\left(\xi(t)\xi(t)^T\partial^2_{(x,v)}\EE(X(t),V(t),t) \  \right)dt.
\end{align*}
Next, we group some terms together,
\begin{align*}
    \mathscr{A}\EE(X(t),V(t),t) &= \dot r(t) (f(X)-f(x^*))dt - \lambda(t)\frac{\dot\m(t)}{\m(t)}\langle \nabla f(X), X-x^*\rangle dt \\
    & \ \ \ \ + \left(\dot \lambda(t) + 1 - \lambda(t) \frac{\dot\m(t)}{\m(t)}\right)\left\langle X-x^* + \lambda(t) V,V\right\rangle dt\\
     & \ \ \ \ + \left(r(t)-\lambda(t)^2\frac{\dot \m(t)}{\m(t)}\right)\left\langle\nabla f(X), V\right\rangle dt\\
    & \ \ \ \ +  \frac{1}{2}\tr\left(\xi(t)\xi(t)^T\partial^2_{(x,v)}\EE(X(t),V(t),t) \  \right)dt.\\
    \end{align*}
    Using \textbf{(H1')}, we conclude
    \begin{align*}
    \mathscr{A}\EE(X(t),V(t),t) &\le \ \ \ \left(\dot r(t)-\tau \lambda(t)\frac{\dot\m(t)}{\m(t)}\right) (f(X)-f(x^*))dt\\
    & \ \ \ \ + \left(\dot \lambda(t) + 1 - \lambda(t) \frac{\dot\m(t)}{\m(t)}\right)\left\langle X-x^* + \lambda(t) V,V\right\rangle dt\\
     & \ \ \ \ + \left(r(t)-\lambda(t)^2\frac{\dot \m(t)}{\m(t)}\right)\left\langle\nabla f(X), V\right\rangle dt\\
    & \ \ \ \ +  \frac{1}{2}\tr\left(\xi(t)\xi(t)^T\partial^2_{(x,v)}\EE(X(t),V(t),t) \  \right)dt.
\end{align*}
Under the hypotheses of this lemma, since $\dot \lambda(t) = \frac{\dot\m(t)}{\m(t)}\lambda(t)-1$ if and only if $\lambda(t) = -\m(t)\int\frac{1}{\m(t)}dt$, we are left with
\begin{align*}
    \mathscr{A}\EE(X(t),V(t),t) &\le \frac{1}{2}\tr\left(\xi(t)\xi(t)^T\partial^2_{(x,v)}\EE(X(t),V(t),t) \  \right) dt\\
    &= \frac{d}{2}\left(\frac{\dot\m(t)}{\m(t)}\right)^2\tr\left(\sigma(t)\sigma(t)^T\partial^2_{v}\EE(X(t),V(t),t)\right) dt\\
    &\le\frac{d}{2}\left(\frac{\dot\m(t)}{\m(t)}\right)^2\|\sigma(t)\sigma(t)^T\|_s \|\partial^2_{v}\EE(X(t),V(t),t)\|_s dt\\
    &\le\frac{d}{2}\sigma^2_* \lambda(t)^2 \left(\frac{\dot\m(t)}{\m(t)}\right)^2 dt,
\end{align*}

where in the first inequality we used Lemma \ref{lemma:CONV_trace_product} and the definition of $\sigma^2_*$ in \textbf{(H0c)}. Finally, by Dynkin's formula
$$\E[\EE(X(t),V(t),t)]-\EE(x_0,0) \le \frac{d\sigma^2_*}{2} \int_0^t  \lambda(s)^2 \left(\frac{\dot\m}{\m}(s)\right)^2 ds;$$

therefore
\begin{multline}
    r(t) \E[f(X(t))-f(x^*)] + \E\left[\frac{1}{2}\|X(t)-x^* + \lambda(t) V\|^2\right]\\\le r(0)(f(x_0)-f(x^*)) + \frac{1}{2}\|x_0-x^*\|^2 + \frac{d\sigma^2_*}{2}\int_0^t  \lambda(s)^2 \left(\frac{\dot\m}{\m}(s)\right)^2 ds,
    \label{eq_for_exi}
\end{multline}
which implies
$$r(t) \E[f(X)-f(x^*)] \le r(0)(f(x_0)-f(x^*)) + \frac{1}{2}\|x_0-x^*\|^2 + \frac{d\sigma^2_*}{2}\int_0^t  \lambda(s)^2 \left(\frac{\dot\m}{\m}(s)\right)^2 ds.$$
\end{proof}

\begin{theorem}
Under the conditions of Lemma~\ref{lemma:app_master_inequality}, let $\m(t) = t^p$, with $p\ge 1+\frac{1}{\tau}$. Then, for any $t> 0$,
$$\E[f(X(t))-f(x^*)] \le \underbrace{\frac{(p-1)^2\|x_0-x^*\|^2}{2pt}}_{\text{rate of convergence to suboptimal sol.}} + \ \underbrace{\frac{p \ d \ \sigma^2_*}{2}}_{\text{suboptimality}} .$$
\label{prop:memory_continuous_WQC}
\end{theorem}

\begin{proof}
This is a simple application of Lemma \ref{lemma:app_master_inequality}. Since $p\ne 1$, we have $\lambda(t) = -\m(t) \int\frac{1}{\m(t)}dt = -t^p\int\frac{1}{t^p}dt = - t^p \left(\frac{t^{1-p}}{1-p}-C\right)$. Let us choose $C=0$, then $\lambda(t) = \frac{t}{p-1}$ and $r(t) = \frac{\dot \m(t)}{\m(t)} \lambda(t)^2 = \frac{p}{t}\frac{t^2}{(p-1)^2} = \frac{pt}{(p-1)^2}$. Thanks to the Lemma, we get a rate if $\dot r(t)\le\tau \lambda(t)\frac{\dot \m(t)}{\m(t)}$; that is, $\frac{p}{(p-1)^2} \le \tau \frac{t}{p-1} \frac{p}{t}$, which is true if and only if $p\ge 1+\frac{1}{\tau}$. Plugging in the functions $h, \m$ and $r$, we get the desired rate.
\end{proof}

\begin{remark}[Effect of starting integration after 0]
We start again from the rate in Lemma~\ref{lemma:app_master_inequality}, in the setting of Thm.~\ref{prop:memory_continuous_WQC}. We consider integration on an interval $[\epsilon, t]$ with $X_\epsilon(\epsilon) = x_0$ and $V_\epsilon(\epsilon) = 0$:

$$\E[f(X_\epsilon(t))-f(x^*)] \le \frac{\frac{p\epsilon}{(p-1)^2}(f(x_0)-f(x^*)) + \frac{1}{2}\|x_0-x^*\|^2}{r(t) } + \ \frac{p d\sigma^2_*}{2} \frac{t-\epsilon}{t}\ .$$

Trivially, as $\epsilon$ goes to $0$ and for any $t\ge 0$, we converge to the solution of Thm.~\ref{prop:memory_continuous_WQC}, which has to be intended as a limit case.
\label{rmk2_exi}
\end{remark}

\subsection{Convergence rates for exponential forgetting}
We recall again our assumptions below.

\textbf{(H0c)} \ \ \  $f\in\C^3_b(\R^d,\R), \ \sigma_*^2 := \sup_{x\in\R^d} \|\sigma(x)\sigma(x)^T\|_s < \infty$.

\textbf{(H1)} \ \ \ \ $f:\R^d\to\R$ is $L$-smooth and convex.

\textbf{(H2)} \ \ \ \ $f:\R^d\to\R$ is $\mu$-strongly convex.

\textbf{(H1')} \ \ $f:\R^d\to\R$ is $L$-smooth and there exists $x\in\R^d$ s.t. for all $x\in\R^d$, $\langle \nabla f(x), x-x^*\rangle\ge \tau (f(x)-f(x^*)).$

\textbf{(H2')}\ \ \ $f:\R^d\to\R$ is such that, for all $x\in\R^d$ we have quadratic growth w.r.t $x^*\in\R^d$: $f(x)-f(x^*)\ge\frac{\mu}{2}\|x-x^*\|^2$.

\begin{remark}
Convexity (i.e. \textbf{(H1)}) implies weak-quasi-convexity \textbf{(H1')} with $\tau = 1$ ~\cite{hardt2018gradient}. Similarly, $\mu$- strong-convexity (i.e. \textbf{(H2)}) implies \textbf{(H1')} with $\tau = 1$ and \textbf{(H2')} with the same $\mu$~\cite{karimi2016linear}.
\label{rmk:weak}
\end{remark}

In this subsection, we are going to study the simplified\footnote{More precisely, we should study the case $\m(t) = e^{\alpha t}-1$ (see Tb.~\ref{tb:memories} in the main paper), which leads to $\dot \m(t) / \m(t) = \frac{\alpha e^{\alpha t}}{e^{\alpha t}-1}$. However, this function converges very quickly to $\alpha$.} stochastic exponential forgetting system

\begin{equation}
    \tag{MG-SDE-exp}
    \begin{cases}
    dX(t) = V(t) dt\\ 
    dV(t) = -\alpha V(t)dt-\alpha\left[\nabla f(X(t))dt +\sigma(X(t)) dB(t)\right]
    \end{cases}.
    \label{MG-SDE-exp}
\end{equation}

We can rewrite this in vector notation
\begin{align*}
    \begin{pmatrix}
    dX(t)\\
    dV(t)
    \end{pmatrix} &=
    \begin{pmatrix}
    V(t)\nonumber\\
    -\alpha V(t)-\alpha \nabla f(X(t))
    \end{pmatrix}dt + 
    \begin{pmatrix}
    0_{d\times d} & 0_{d\times d}\\
     0_{d\times d} & -\alpha\sigma(X(t))
    \end{pmatrix} dB(t)\nonumber\\
    &= b(X(t),V(t),t) dt + \xi(X(t)) dB(t), 
\end{align*}

which we write for simplicity as $b(t) dt + \xi(t) dB$. 
Existence and uniqueness of the solution to this SDE follows directly from \cite{mao2007stochastic}.

\subsubsection{Result under weak-quasi-convexity}
Below is the main result for this subsection.

\begin{theorem} Assume \textbf{(H0c)} and \textbf{(H1')} hold. Let $\{X(t),V(t)\}_{t\ge0}$ be the stochastic process which solves \ref{MG-SDE-exp} for $t\ge 0$, starting from $X(0)=x_0$ and $V(0)=0$. Then, for $t>0$ we have
\begin{equation*}
    \E\left[f(X(\tilde t))-f(x^*)\right] \le \underbrace{\frac{\left(f(x_0)-f(x^*)\right)+\frac{\alpha}{2}\|x_0-x^*\|^2}{\alpha \tau t}}_{\text{rate of convergence to suboptimal sol.}}+ \underbrace{\frac{d}{2\tau }\sigma^2_*}_{\text{suboptimality}},
\end{equation*}
where $\tilde t$ is sampled uniformly in $[0,t]$. Moreover, if \textbf{(H1)} holds (and therefore $\tau = 1$), we can replace $X(\tilde t)$ with the Cesàro average $\bar X(t) = \int_0^tX(s)ds$.
\label{prop:CONV_HBF_WQC_app}   
\end{theorem}

\begin{remark}
Notice that the suboptimality does not depend on the viscosity $\alpha$ and is exactly the same as for the SGD model \cite{orvieto2018continuous}. This would \textit{not} happen if we erase  $\alpha$ in front of the gradient (i.e. the standard Heavy Ball model $\ddot X + \alpha \dot X + \nabla f(X) =0$ \cite{yang2018physical}).
\end{remark}

\begin{proof}
Consider $\EE(x,t) :=\alpha\left(f(x)-f(x^*)\right) + \frac{1}{2}\|v + \alpha (x-x^*)\|^2$. Using Eq.~\eqref{ITO} we compute an upper bound on the infinitesimal diffusion generator of $\EE$.

\begin{align*}
    \mathscr{A}\EE(X(t),V(t),t) &= \cancel{\partial_t \EE (X(t),V(t), t)) dt} + \langle\partial_{(x,v)} \EE(X(t),V(t),t), b(t)\rangle dt\\ 
    & \ \ \ \ +  \frac{1}{2}\tr\left(\xi(t)\xi(t)^T\partial^2_{(x,v)}\EE(X(t),V(t),t) \  \right)dt\\
    &= \langle\partial_{x} \EE(X(t),V(t),t), V\rangle dt\\
    & \ \ \ \ + \left\langle\partial_{v} \EE(X(t),V(t),t), -\alpha V(t)-\alpha \nabla f(X(t)) \right\rangle dt\\
    & \ \ \ \ +  \frac{1}{2}\tr\left(\xi(t)\xi(t)^T\partial^2_{(x,v)}\EE(X(t),V(t),t) \  \right)dt\\
    &= \langle\alpha\nabla f(X(t)) + \alpha( V(t) + \alpha(X(t)-x^*)), V(t)\rangle dt\\
    & \ \ \ \ + \left\langle V(t) + \alpha(X(t)-x^*), -\alpha V(t)-\alpha \nabla f(X(t)) \right\rangle dt\\
    & \ \ \ \ +\frac{1}{2}\tr\left(\alpha^2\sigma(t)\sigma(t)^T\partial^2_{(v,v)}\EE(X(t),V(t),t) \  \right)dt \\
    &= \alpha\langle\nabla f(X(t)), V(t)\rangle dt +  \alpha^2 \langle X(t)-x^*, V(t)\rangle + \alpha\|V(t)\|^2\\
    & \ \ \ \ -\alpha \|V(t)\|^2 - \alpha \langle \nabla f(X(t)), V(t)\rangle - \alpha^2\langle V(t), X(t)-x^*\rangle - \alpha^2\langle\nabla f(X(t)), X(t)-x^*\rangle\\
    & \ \ \ \ +  \frac{d}{2}\alpha^2\sigma^2_* \\
    &= - \alpha^2\langle\nabla f(X(t)), X(t)-x^*\rangle + \frac{d}{2}\alpha^2\sigma^2_*\\& \le - \alpha^2 \tau (f(X(t))-f(x^*)) + \frac{d}{2}\alpha^2\sigma^2_*. 
\end{align*}

where in the first inequality we used Lemma \ref{lemma:CONV_trace_product} and the definition of $\sigma^2_*$ in \textbf{(H0c)} and in the last we used \textbf{(H1')}. Finally, by Eq.~\eqref{eq:dynkin}

\begin{equation*}
    \E\left[\EE(X(t),t)\right]-\EE(x_0,0) \le - \alpha^2 \tau \int_0^t\E[f(X(s))-f(x^*)]ds + \frac{d}{2}\alpha^2\sigma^2_*t.
\end{equation*}

Since $\E\left[\EE(X(t),t)\right]\ge0$, we have
\begin{equation*}
    \alpha^2 \tau \int_0^t\E[f(X(s))-f(x^*)]ds \le \EE(x_0,0)+ \frac{d}{2}\alpha^2\sigma^2_*t;
\end{equation*}
Therefore
\begin{equation*}
    \E\left[\int_0^t \frac{1}{t} f(X(s))-f(x^*)ds\right] \le \frac{\alpha\left(f(x_0)-f(x^*)\right)+\frac{\alpha^2}{2}\|x_0-x^*\|^2}{\alpha^2 \tau t}+ \frac{d}{2\tau }\sigma^2_*
\end{equation*}
Now, if $f(\cdot)$ is convex, we can use Jensen's inequality and get the result directly. Otherwise, if $f(\cdot)$ is under \textbf{(H1')}, we can view the integral as the expectation of $f(X(\tilde t))-f(x^*)$, where $\tilde t$ is sampled uniformly over $[0,t]$.
\end{proof}
\subsubsection{Result under weak-quasi-convexity and quadratic growth}
In the next proposition, we are going the consider a slightly generalized SDE

\begin{equation}
    \tag{MG-SDE-exp-G}
    \begin{cases}
    dX(t) = V(t) dt\\ 
    dV(t) = -\alpha V(t)dt - \nabla \tilde f(X(t))dt + \tilde\sigma(X(t)) dB(t)
    \end{cases},
    \label{MG-SDE-exp-G}
\end{equation}
where $\tilde f(x) = C f(x)$ and $\tilde \sigma(x) = C \sigma(x)$. We introduce this generalization to provide rates also for the stochastic counterpart of the standard Heavy-ball model $\ddot X + \alpha \dot X + \nabla f(X) = 0$ generalizing \cite{shi2018understanding} to \textit{any} viscosity $\alpha$. 

\paragraph{Further notation.} Under assumptions \textbf{(H0c)}, \textbf{(H1')}, \textbf{(H2')}, we define the constants $\tilde\sigma_*^2$ and $\tilde \mu$ to match our assumptions: we have $\tilde\sigma_*^2 := \sup_{x\in\R^d}\|\tilde \sigma(x)\tilde \sigma(x)^T\|_s = C^2 \sigma_*^2 $, and $\tilde f(x)-\tilde f(x^*) \ge \frac{\tilde\mu}{2}\|x-x^*\|^2 := \frac{C\mu}{2}\|x-x^*\|^2$.

Notice that \ref{MG-SDE-exp-G} is exacly \ref{MG-SDE-exp} for $\tilde f(\cdot) = \alpha f(\cdot)$ and $\tilde \sigma(\cdot) = \alpha\sigma(\cdot)$. However, it also contains the SDE of standard Heavy Ball in the case $\tilde f(\cdot) = f(\cdot)$ and $\tilde \sigma(\cdot) = \sigma(\cdot)$.

\begin{proposition} Assume \textbf{(H0c)}, \textbf{(H1')}, \textbf{(H2')} hold. Let $\{X(t),V(t)\}_{t\ge0}$ be the stochastic process which solves \ref{MG-SDE-exp-G} for $t\ge 0$ starting from $X(0)=x_0$ and $V(0)=0$. Then, for $t\ge0$ we have

    \begin{equation}
    \E[\tilde f(X(t))-\tilde f(x^*)]  \le e^{-\gamma t}\left(\tilde f(x_0)-\tilde f(x^*) + \frac{(\alpha-\gamma)^2}{2}\|(x_0-x^*)\|^2\right) + \frac{d}{2\gamma}\tilde\sigma_*^2,
\end{equation}
    where
    \begin{equation*}
    \begin{cases}
        \gamma = \frac{\tau}{\tau +2}\alpha & \text{for }\alpha\le \alpha_{max}\\
        \gamma = \frac{1}{2}\left(\alpha-\sqrt{\alpha^2-2\tilde\mu\tau}\right) & \text{for }\alpha\ge\alpha_{max} \\
    \end{cases}
\end{equation*}
and $\alpha_{max}$ is the optimal viscosity parameter
$$\alpha_{max} = \frac{\tau+2}{2}\sqrt{\tilde\mu}.$$

\label{prop:CONV_HBF_WQC_QG_app}   
\end{proposition}
We first need a lemma, of which we leave the proof to the reader.

\begin{lemma}
Let $\tau,\beta,\tilde\mu$ be positive real numbers then 
\begin{equation*}
    \frac{\tau\beta}{1 + \frac{\beta^2}{\tilde\mu}}\ge
    \begin{cases}
    \frac{\tau}{2}\beta & \text{if } \beta\le \sqrt{\tilde\mu}\\
    \frac{\tau\tilde\mu}{2\beta} & \text{if } \beta\ge \sqrt{\tilde\mu}
    \end{cases}.
\end{equation*}
\vspace{-3mm}
\label{lemma:CONV_HBF_WQC_QG}
\end{lemma}
We illustrate this result graphically in Fig. \ref{fig:bound_beta} and proceed with the proof of the rate.

\begin{figure}
  \centering
    \includegraphics[width=0.5\textwidth]{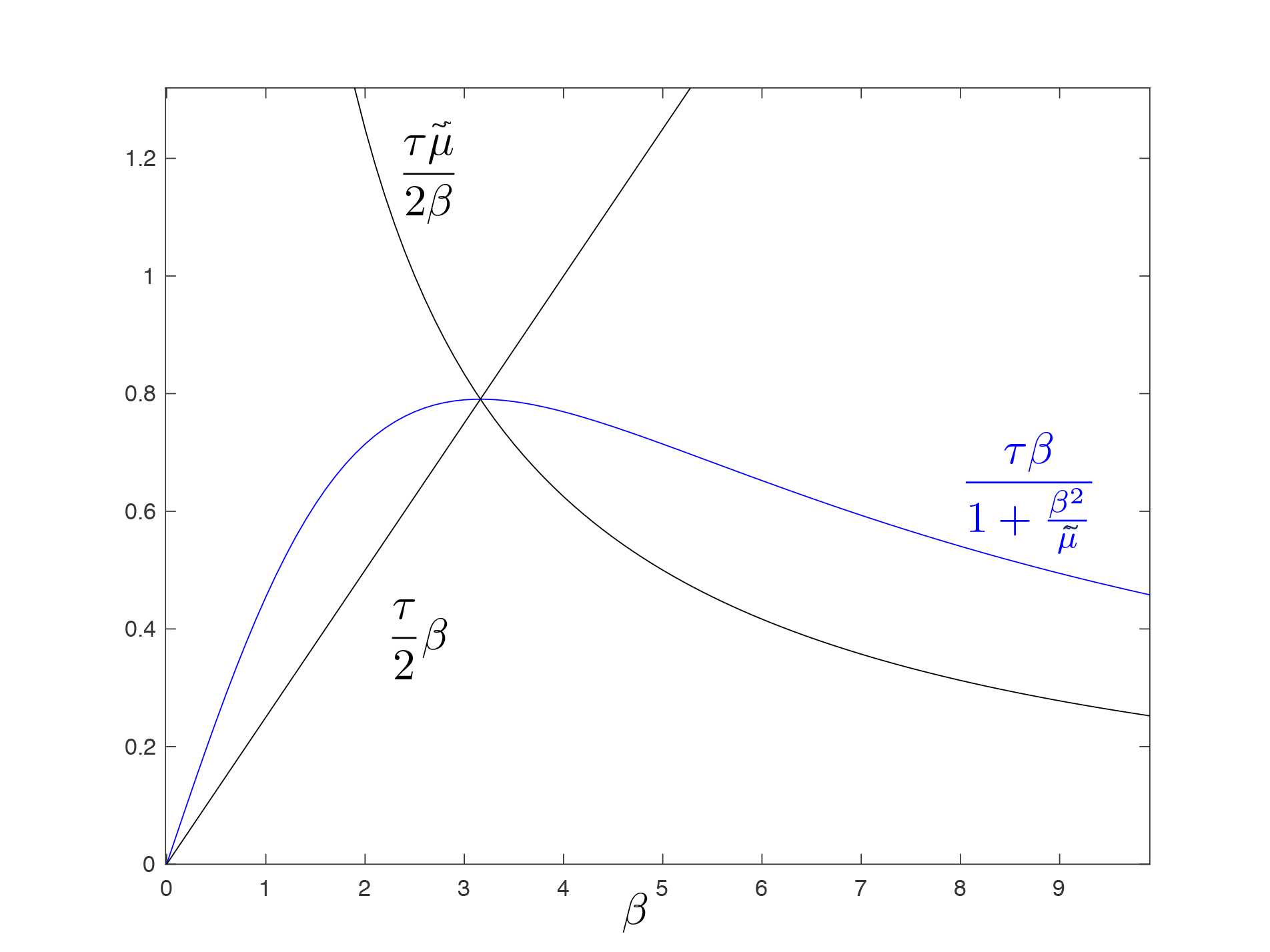}
     \caption{Bounds from Lemma~\ref{lemma:CONV_HBF_WQC_QG} for $\tau = 0.5$ and $\tilde\mu = 10$. The bounds overlap at $\sqrt{10} \simeq 3.13$.}
	\label{fig:bound_beta}
\end{figure}
\begin{proof}(of Prop. \ref{prop:CONV_HBF_WQC_QG_app})
Define the following parametric energy

$$\EE_{\gamma,\beta}(X(t),t) := e^{\gamma t}\left(\tilde f(X(t))-\tilde f(x^*) + \frac{1}{2}\|V(t) + \beta (X(t)-x^*)\|^2\right),$$
where $\beta$ and $\gamma$ are positive real numbers. The infinitesimal diffusion generator of $\EE_{\gamma,\beta}(X(t),t)$ is 

\begin{align*}
     \mathscr{A}\EE(X(t),V(t),t) &=\underbrace{\partial_t \EE(X(t),V(t),t) dt}_{\circled{1}} + \underbrace{\langle\partial_{x} \EE(X(t),V(t),t), V\rangle dt}_{\circled{2}}\\
    & \ \ \ \ + \underbrace{\left\langle\partial_{v} \EE(X(t),V(t),t), -\alpha V(t)-\alpha \nabla \tilde f(X(t)) \right\rangle dt}_{\circled{3}}\\
    & \ \ \ \ +   \underbrace{\frac{1}{2}\tr\left(\xi(t)\xi(t)^T\partial^2_{(x,v)}\EE(X(t),V(t),t) \  \right)dt}_{\circled{4}}.
    \end{align*}
Using the definition of $\tilde\sigma_*^2$ and Lemma \ref{lemma:CONV_trace_product}, we can upper bound the generator with
    \begin{align*}
       \mathscr{A}\EE(X(t),V(t),t) & \le \underbrace{\gamma e^{\gamma t}\left[\tilde f(X(t))-\tilde f(x^*) + \frac{1}{2}\|V(t) + \beta (X(t)-x^*)\|^2\right]dt}_{\circled{1}}\\ 
        & \ \ \   + \underbrace{e^{\gamma t} \langle\nabla \tilde f + \beta V + \beta^2 (X-x^*),V\rangle dt}_{\circled{2}}\ \ \ + \underbrace{e^{\gamma t} \langle V+\beta(X-x^*),-\alpha V -\nabla \tilde f \rangle dt}_{\circled{3}}+ \underbrace{\frac{d}{2}e^{\gamma t}\tilde\sigma_*^2dt}_{\circled{4}}.\\
\end{align*}

Next, we multiply both sides by $\frac{e^{-\gamma t}}{dt}$ and proceeding with the calculations we get

\begin{equation*}
    \begin{split}
        \frac{e^{-\gamma t}\mathscr{A}\EE}{dt} & = \gamma \tilde f(X(t))-\tilde f(x^*) + \frac{\gamma}{2}\|V(t) + \beta (X(t)-x^*)\|^2\\ 
        & \ \ \ + \frac{d}{2}\tilde\sigma_*^2 +  \langle\nabla \tilde f + \beta V + \beta^2 (X-x^*),V\rangle\\
        & \ \ \ + \langle V+\beta(X-x^*),-\alpha V -\nabla \tilde f \rangle\\
        & = \gamma \tilde f(X(t))-\tilde f(x^*) + \frac{\gamma}{2}\|V\|^2 + \frac{\beta^2\gamma}{2} \|X-x^*\|^2 + \beta\gamma\langle V, X-x^*\rangle\\
        & \ \ \ + \frac{d}{2}\tilde\sigma_*^2 + \langle\nabla \tilde f,V\rangle +  \beta \|V\|^2 + \beta^2 \langle V,X-x^* \rangle\\
        & \ \ \ -\alpha \|V\|^2 -\langle \nabla \tilde f, V\rangle -\alpha \beta \langle V,X-x^*\rangle -\beta\langle \nabla \tilde f, X-x^*\rangle.\\ 
        & \le \left(\gamma + \frac{\beta^2\gamma}{\tilde\mu} - \beta \tau\right) \tilde f(X(t))-\tilde f(x^*) + \frac{d}{2}\tilde\sigma_*^2\\
        & \ \ \  + \left(\frac{\gamma}{2} + \beta -\alpha \right)\|V\|^2 + (\beta\gamma + \beta^2 -\alpha\beta)\langle V, X-x^*\rangle,
    \end{split}
\end{equation*}
where in the inequality we used \textbf{(H1')} and \textbf{(H2')}. Next, since we a priori don't have a way to bound the term $\langle V,X-x^*\rangle$ in general, we require
$$\gamma\beta + \beta^2 -\alpha\beta = 0,$$
which holds if and only if $\gamma + \beta = \alpha$. This also implies the coefficient multiplying $\|V\|^2$, that is $\left(\frac{\gamma}{2} + \beta - \alpha\right)$, is negative (because $\gamma$ and $\beta$ are positive).

Assume for now that $\left(\gamma + \frac{\beta^2\gamma}{\tilde\mu} - \beta \tau\right)\le0$ (we will deal with this at the end of this proof), we then get $\mathscr{A}\EE \le \frac{d}{2}\tilde\sigma_*^2e^{\gamma t}.$

Next, using Eq.~\eqref{eq:dynkin},

$$\E\left[\EE_{\gamma,\beta}(X(t),t)\right] - \EE_{\gamma,\beta}(x_0,0) \le \int_0^t\frac{d}{2}\tilde\sigma_*^2e^{\gamma s} ds.$$

Using the definition of $\EE_{\gamma,\beta}$, we get
\begin{equation*}
    e^{\gamma t}\E\left[\tilde f(X(t))-\tilde f(x^*) + \frac{1}{2}\|V(t) + \beta (X(t)-x^*)\|^2\right] \\ \le \tilde f(x_0)-\tilde f(x^*) + \frac{\beta^2}{2}\|(x_0-x^*)\|^2 + \int_0^t\frac{d}{2}\tilde\sigma_*^2e^{\gamma s} ds.
\end{equation*}

Hence, discarding the positive term $\frac{1}{2}\|V(t) + \beta (X(t)-x^*)\|^2$ and multiplying both sides by $ e^{-\gamma t}$ and solving the integral
\begin{align*}
    \E[\tilde f(X(t))-\tilde f(x^*)]  &\le e^{-\gamma t}\left(\tilde f(x_0)-\tilde f(x^*) \frac{\beta^2}{2}\|(x_0-x^*)\|^2\right) + \frac{d}{2}\tilde\sigma_*^2\int_0^te^{-\gamma (t-s)} ds\\ &\le e^{-\gamma t}\left(\tilde f(x_0)-\tilde f(x^*) \frac{\beta^2}{2}\|(x_0-x^*)\|^2\right) + \frac{d}{2}\tilde\sigma_*^2\frac{1-e^{-\gamma t}}{\gamma}\\
    &\le e^{-\gamma t}\left(\tilde f(x_0)-\tilde f(x^*) \frac{\beta^2}{2}\|(x_0-x^*)\|^2\right) + \frac{d}{2\gamma}\tilde\sigma_*^2,
\end{align*}

which is the statement of the theorem. Now, we just need to choose $\beta$ and $\gamma$ such that $\left(\gamma + \frac{\beta^2\gamma}{\tilde\mu} - \beta \tau\right)<0$ while $\gamma + \beta = \alpha$. Since we have an inequality, we expect that the choice of $\beta$ and $\gamma$ is not unique. However, since $\gamma$ directly influences the rate, we would like it to be as large as possible. This can be formulated with the following linear program with nonlinear constrains.

\begin{equation*}
   (\gamma^*, \beta^*) = \begin{cases}
    \max_{\gamma, \beta} \ \ \ \gamma\\
    \text{s.t} \ \ \ \ \ \ \ \ \ \ \  \gamma = \alpha - \beta\\
    \ \ \ \ \ \ \ \ \ \ \ \ \ \ \  \gamma\le\frac{\tau\beta}{1 + \frac{\beta^2}{\tilde\mu}}\\
    \ \ \ \ \ \ \ \ \ \ \ \ \ \ \  \gamma, \beta > 0
    \end{cases}.
\end{equation*}

Using Lemma \ref{lemma:CONV_HBF_WQC_QG}, we shrink the feasible region at the cost of having a suboptimal, yet simpler, solution

\begin{equation*}
   (\gamma^*, \beta^*) = \begin{cases}
    \max_{\gamma, \beta} \ \ \ \gamma\\
    \text{s.t} \ \ \ \ \ \ \ \ \ \ \  \gamma = \alpha - \beta \ \ \ \ \ \ \ \ \ \ \ \ \ \ \ \ \  \ \ \ \ \ (\bigcirc)\\
    \ \ \ \ \ \ \ \ \ \ \ \ \ \ \  \gamma\le \min\left\{\frac{\tau}{2}\beta, \frac{\tau\tilde\mu}{2\beta}\right\} \ \ \ \ \ \ \ (\msquare) \\
    \ \ \ \ \ \ \ \ \ \ \ \ \ \ \  \gamma, \beta > 0
    \end{cases}.
\end{equation*}

\begin{figure}
  \centering
    \includegraphics[width=0.6\textwidth]{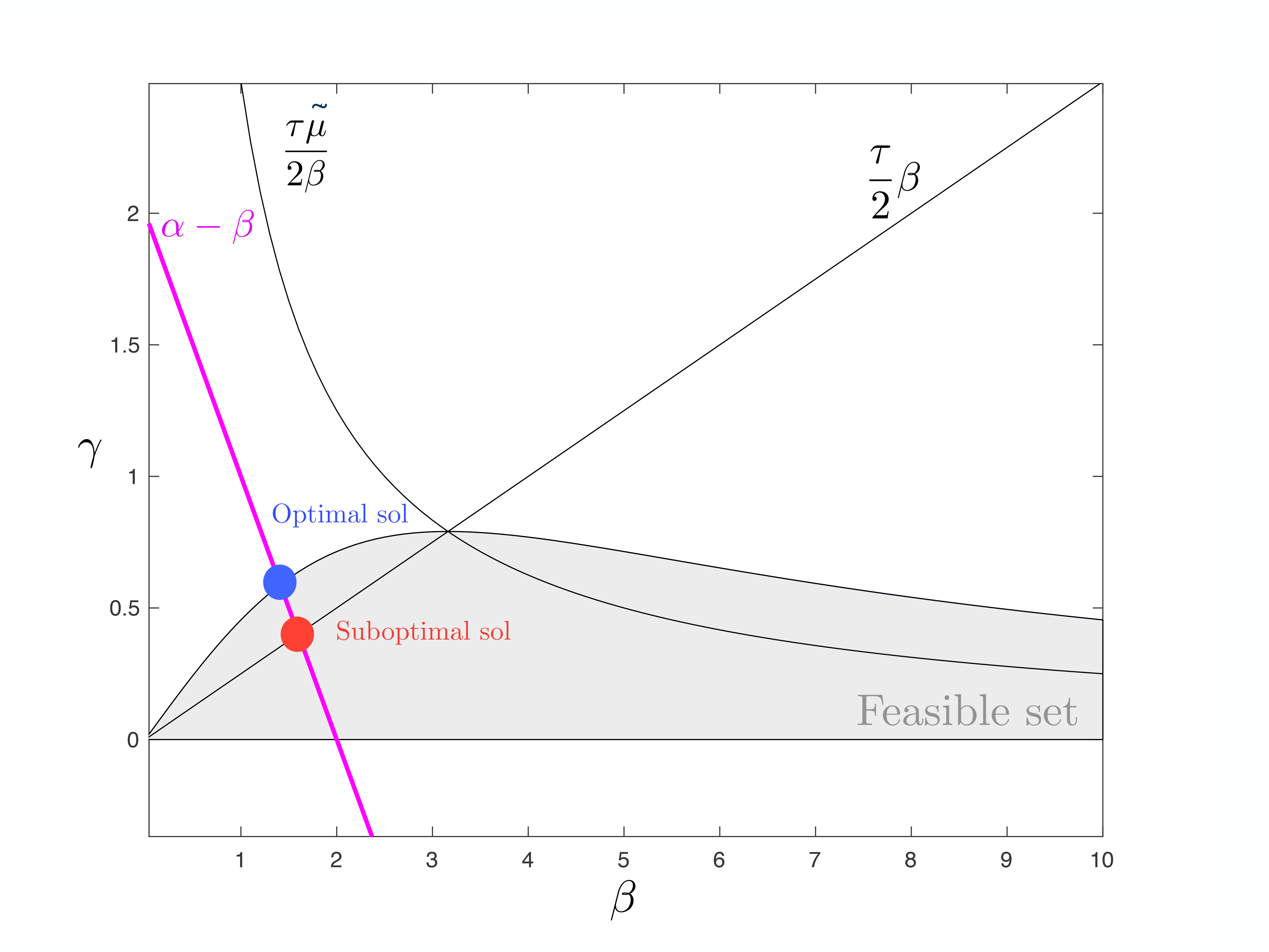}
    \vspace{-3mm}
     \caption{Bounds provided in Lemma \ref{lemma:CONV_HBF_WQC_QG}. Shown is the feasible region for the $\gamma$ maximization problem along with the optimal and approximate solution.}
	\label{fig:bound_beta_2}
\end{figure}

For any fixed $\alpha>0$, the RHS of $(\msquare)$ is positive and the RHS of $(\bigcirc)$ is a line which starts in $(\alpha, 0)$ and eventually exist the set $\{\gamma, \beta) \ : \ \gamma\ge0\}$. Hence, the constrain $(\msquare)$ is binding. The reader can find an helpful representation in Fig. \ref{fig:bound_beta_2}. Notice now that we only have one feasible point, hence

\begin{equation*}
   (\gamma^*, \beta^*) = \begin{cases}
    \gamma = \alpha - \beta\\
    \gamma = \min\left\{\frac{\tau}{2}\beta, \frac{\tau\tilde\mu}{2\beta}\right\}\\
    \gamma, \beta > 0
    \end{cases}.
\end{equation*}

Next, we ask when $\frac{\tau}{2}\beta = \frac{\tau\tilde\mu}{2\beta}$. This clearly happens at $\beta = \sqrt{\tilde\mu}$. This is the maximum of the continuous curve $\min\left\{\frac{\tau}{2}\beta, \frac{\tau\tilde\mu}{2\beta}\right\}$ with respect to $\beta$. The maximum achievable rate is therefore $\frac{\tau}{2}\sqrt{\tilde\mu}$ and it is reached for $\frac{\tau}{2}\sqrt{\tilde\mu} = \alpha -\sqrt{\tilde\mu} \Rightarrow \alpha = \frac{\tau+2}{2}\sqrt{\tilde\mu}$. Moreover, since $\frac{\tau\tilde\mu}{2\beta}$ and $\frac{\tau}{2}\beta$ only touch in one point, it is clear that
\begin{equation*}
    \gamma^* = 
    \begin{cases}
            \frac{\tau}{2}\beta & \alpha\le\frac{\tau+2}{2}\sqrt{\tilde\mu}\\
            \frac{\tau\tilde\mu}{2\beta} & \alpha\ge\frac{\tau+2}{2}\sqrt{\tilde\mu}
    \end{cases}.
\end{equation*}

In any of these two regions, one just has equate the corresponding equation with $\gamma^* = \alpha - \beta^*$, to get
\begin{equation*}
    \gamma^* =\begin{cases}
         \frac{\tau}{\tau +2}\alpha & \text{for }\alpha\le\frac{\tau+2}{2}\sqrt{\tilde\mu} \\
        \frac{1}{2}\left(\alpha-\sqrt{\alpha^2-2\tilde\mu\tau}\right) & \text{for }\alpha\ge\frac{\tau+2}{2}\sqrt{\tilde\mu} \\
    \end{cases}.
\end{equation*}
\end{proof}

From this general proposition, we derive results form the memory ODE and for the standard Heavy-ball ODE.

\begin{theorem}Assume \textbf{(H0c)}, \textbf{(H1)}, \textbf{(H2)} hold. Let $\{X(t),V(t)\}_{t\ge0}$ be the stochastic process which solves \ref{MG-SDE-exp} for $t\ge 0$ starting from $X(0)=x_0$ and $V(0)=0$. Then, for $t\ge0$ we have
    \begin{equation}
    \E[f(X(t))-f(x^*)]  \le e^{-\gamma t}\left((f(x_0)-f(x^*)) + \frac{(\alpha-\gamma)^2}{2\alpha}\|(x_0-x^*)\|^2\right) + \frac{d\alpha}{2\gamma}\sigma_*^2.
\end{equation}
    where
    \begin{equation*}
    \begin{cases}
        \gamma = \frac{1}{3}\alpha & \text{for }\alpha\le \alpha_{max}\\
        \gamma = \frac{1}{2}\left(\alpha-\sqrt{\alpha^2-2\alpha\mu}\right) & \text{for }\alpha\ge\alpha_{max} \\
    \end{cases}
\end{equation*}
and $\alpha_{max}$ is the optimal viscosity parameter
$$\alpha_{max} = \frac{9}{4}\mu.$$
\label{cor:SC}   
\end{theorem}

\begin{proof}
Follows directly from Prop.~\ref{prop:CONV_HBF_WQC_QG_app} plugging in $\tau = 1$, $\tilde f(x) = \alpha f(x)$, $\tilde\mu = \alpha\mu$ and $\tilde\sigma_*^2 = \alpha^2\sigma_*^2$.
\end{proof}
\begin{theorem}Assume \textbf{(H0c)}, \textbf{(H1)}, \textbf{(H2)} hold. Let $\{X(t),V(t)\}_{t\ge0}$ be the stochastic process which solves HB-SDE\footnote{HB-SDE is defined, similarly \ref{MG-SDE}, by augmenting the phase space representation with a volatility : $dX(t) = V(t) dt; \  dV(t) = -\alpha V(t) dt - \nabla f(X(t))dt + \sigma(X(t))dB(t)$, i.e. the stochastic version of the ODE in \cite{shi2018understanding}.} for $t\ge 0$, starting from $X(0)=x_0$ and $V(0)=0$. Then, for $t\ge0$ we have
    \begin{equation}
    \E[f(X(t))-f(x^*)]  \le e^{-\gamma t}\left(f(x_0)-f(x^*) + \frac{(\alpha-\gamma)^2}{2}\|(x_0-x^*)\|^2\right) + \frac{d\alpha}{2\gamma}\sigma_*^2.
\end{equation}
    where
    \begin{equation*}
    \begin{cases}
        \gamma = \frac{1}{3}\alpha & \text{for }\alpha\le \alpha_{max}\\
        \gamma = \frac{1}{2}\left(\alpha-\sqrt{\alpha^2-2\mu}\right) & \text{for }\alpha\ge\alpha_{max} \\
    \end{cases}
\end{equation*}
and $\alpha_{max}$ is the optimal viscosity parameter
$$\alpha_{max} = \frac{3}{2}\sqrt{\mu}.$$
\label{cor:SC_1}   
\end{theorem}

\begin{proof}
Follows directly from Prop.~\ref{prop:CONV_HBF_WQC_QG_app} plugging in $\tilde f(x) = f(x)$, $\tau=1$, $\tilde \mu = \mu$ and $\tilde\sigma_*^2 = \sigma_*^2$.
\end{proof}

The optimal rates are exponential for both SDEs; however, for \ref{MG-SDE-exp} the constant in such exponential is proportional to $\sqrt{\mu}$, while for HB-SDE it is proportional to $\mu$. Note that this difference might be big if $\mu\ll 1$, leading to a faster rate for HB-SDE. However, this difference \textit{only comes because of the discretization procedure} and will not be present in the algorithmic counterparts. Indeed, it is possible to show (we leave this exercise to the reader) that the optimal discretization stepsize $h$ for the first SDE (see Sec.~\ref{sec:analysis_discretization} in the main paper) is $1/\sqrt{\mu L}$, while for the second SDE it is $2/(\sqrt{L}+\sqrt{\mu}) = \mathcal{O}(1/\sqrt{L})$. Therefore, since we have the correspondence $t = kh$, in both cases we actually end up with $e^{\gamma_{max} t}\approx (1-C\sqrt{\mu/L})^k$, which is the well known accelerated rate found by~\cite{polyak1964some}. This interesting difference and the link to numerical integration deserves to be explored in a follow-up work.

\section{Provably convergent discrete-time polynomial forgetting}
\label{poly-discre}
The algorithm below (a generalization of Heavy Ball) builds a sequence of iterates $\{x_k\}_{k\in\N}$ as well as moments $\{m_k\}_{k\in\N}$ starting from $m_0=0$ and using the recursion

\begin{align}
        m_{k+1} &= \beta_k(x_k-x_{k-1}) - \delta_k \eta \nabla f_{i_k}(x_k)\label{eq:discrete_time_m}\\
        x_{k+1} &= x_k + m_{k+1}\label{eq:discrete_time_x}
\end{align}

where $i_k\in\{1,\dots,n\}$ is the index of a random data point sampled at iteration $k$,  $\beta_k$ is an iteration-dependent positive momentum parameter and $\delta_k$ is a positive iteration-dependent discount on the learning rate $\eta$. Trivially, for each $x\in\R^d$, $f_{i_k}(x)$ is a random variable with mean zero and finite covariance matrix, which we call $\Sigma(x)$.

We list below two important assumptions

\textbf{(H0d)} \ \ \ $\varsigma_*^2 := \sup_{x\in\R^d} \|\Sigma(x)\|_s<\infty$.

\textbf{(H1)} \ \ \ $f:\R^d\to\R$ is convex and $L$-smooth.

We start with a rather abstract result inspired from \cite{ghadimi2015global}, which we will then use for algorithm design.

\begin{lemma}[Discrete-time Master Inequality] Assume \textbf{(H1)} and \textbf{(H0d)} hold. Let $\{\lambda_k\}_{k\in\mathbb{N}}$ be any sequence such that $\lambda_k\le k$ for all $k$ and define $r_k= \eta(\lambda_k +1)$. If $\beta_k = \frac{\lambda_k}{\lambda_{k+1}+1}$, $\delta_k = \frac{1}{\lambda_{k+1}+1}$ and $\eta\le 1/L$. Then we have for  all iterates $\{x_k\}_{k\in\N}$ given by Eq.~\eqref{eq:discrete_time_m} and \eqref{eq:discrete_time_x} that
$$\E[f(x_k)-f(x^*)] \le \frac{\|x_0-x^*\|^2}{2 r_k} + \frac{d \eta^2 \varsigma_*^2}{2} \frac{k}{r_k}.$$
\label{lemma:master_discrete}
\end{lemma}

\begin{proof}
Consider the Lyapunov function inspired by the continuous-time setting in Lemma \ref{lemma:unbiased}.

\begin{equation}\label{eq:lyapunov_discrete}
    \EE_{k} = r_k (f(x_k)-f(x^*)) + \frac{1}{2}\|x_{k+1}-x^* + \lambda_k m_{k+1}\|^2
\end{equation}

First, notice that

\begin{equation}
    \begin{aligned}
     x_{k+1} -x^* + \lambda_{k+1}m_{k+1} &\overset{\eqref{eq:discrete_time_x}}{=} x_k-x^*+(1+\lambda_{k+1})m_{k+1} \\
     &\overset{\eqref{eq:discrete_time_m}}{=} x_k-x^*+(1+\lambda_{k+1})(\beta_k m_k-\delta_k\eta \nabla f_{i_k}(x_k))\\
     &= x_{k} -x^* + \lambda_k m_k - \eta\nabla f_{i_k}(x_k)),
    \end{aligned}
\end{equation}

where in the last line we chose $\lambda_k = (\lambda_{k+1}+1)\beta_k$ and $\delta_k =\frac{1}{\lambda_{k}+1}$.

Consider $\zeta_k := \nabla f_{i_k}(x_k)-\nabla f(x_k)$, then 
\begin{align*}
\E\left[\| x_{k+1}- x^* + \lambda_{k+1}m_{k+1} \|^2\right] 
=& \E\left[\| x_{k} -x^* + \lambda_k m_k - \eta\nabla f(x_k)) +\eta \zeta_k\|^2\right]\\
=& \E\left[\| x_{k} -x^* + \lambda_k m_k\|^2\right] + \E\left[\|\eta\nabla f(x_k) + \eta\zeta_k\|^2\right]\\
& -2\eta \E\left[\langle\nabla f(x_k), x_{k} -x^* + \lambda_k m_k\rangle\right]\\
=& \E\left[\| x_{k} -x^* + \lambda_k m_k\|^2\right] + \eta^2\E\left[\|\nabla f(x_k)\|^2\right] + \eta^2\E\left[\|\zeta_k\|^2\right]\\
&  -2\eta \E\left[\langle\nabla f(x_k), x_{k} -x^*\rangle\right]  -2\eta \lambda_k \E\left[\langle f(x_k),m_k\rangle\right].
\end{align*}
Since $f(\cdot)$ is convex and smooth, it follows from Thm. 2.1.5 in \cite{nesterov2018lectures} that
\begin{equation*}
\begin{aligned}
 \dfrac{1}{L}\|\nabla f(x_k)\|^2&\leq \langle x_k - x^* , \nabla f(x_k) \rangle,\\
 f(x_k)-f(x^*)+\dfrac{1}{2L}\|\nabla f(x_k)\|^2&\leq \langle x_k - x^* , \nabla f(x_k) \rangle,\\
f(x_k)-f(x_{k-1}) &\leq \langle x_k - x_{k-1} , \nabla f(x_k)) \rangle
\end{aligned}
\end{equation*}

 for all $x_k$. Next, notice that $\zeta_k$ is a random variable with mean $0$ and we denote its covariance by $\Sigma(x_k)\geq 0$. Moreover, $$\E\left[\|\zeta_k\|^2\right] = \E\left[\tr(\zeta_k \zeta_k^T)\right] =\tr(\E\left[\zeta_k \zeta_k^T\right]) = \tr(\Sigma(x_k)) \le d\|\Sigma(x_k)\|_s = d\varsigma_*^2.$$ Let us assume $k\ge 1$, then
 
 \begin{align*}
\E\left[\| x_{k+1}- x^* + \lambda_{k+1}m_{k+1} \|^2\right] =& \E\left[\| x_{k} -x^* + \lambda_k m_k\|^2\right] + \eta^2\E\left[\|\nabla f(x_k)\|^2\right] + d\eta^2\varsigma_*^2\\
&  -2\eta \E\left[\langle\nabla f(x_k), x_{k} -x^*\rangle\right]  -2\eta \lambda_k \E\left[\langle\nabla f(x_k), x_k -x_{k-1}\rangle\right]\\
\le& \E\left[\| x_{k} -x^* + \lambda_k m_k\|^2\right] + \eta^2\E\left[\|\nabla f(x_k)\|^2\right] + d\eta^2\varsigma_*^2\\
&  -2\eta \E\left[ f(x_k)-f(x^*)+\dfrac{1}{2L}\|\nabla f(x_k)\|^2
\right]  -2\eta \lambda_k \E\left[f(x_k)-f(x_{k-1})\right]\\
\le& \E\left[\| x_{k} -x^* + \lambda_k m_k\|^2\right] + \eta\left(\eta-\dfrac{1}{L}\right)\E\left[\|\nabla f(x_k)\|^2\right]+ d\eta^2\varsigma_*^2\\
&  -2\eta (1+\lambda_k) \E\left[f(x_k)-f(x^*)\right] + 2\eta \lambda_k\E\left[f(x_{k-1})-f(x^*)\right].\\
 \end{align*}
 
Then, let $\eta\le 1/L$, note that $\E\left[f(x_k)-f(x^*)\right]\geq 0, \forall k$ and assume $\lambda_k \le \lambda_{k-1}+1$. As a result, we have
\begin{multline*}
    \E\left[\eta (1+\lambda_k) (f(x_k)-f(x^*)) + \frac{1}{2}\| x_{k+1}- x^* + \lambda_{k+1}m_{k+1} \|^2\right]\\ \le \E\left[\eta (1+\lambda_{k-1})(f(x_{k-1})-f(x^*)) + \frac{1}{2}\| x_{k} -x^* + \lambda_k m_k\|^2\right]+ \frac{d\eta^2\varsigma_*^2}{2}.
\end{multline*}

Recalling the definition of our Lyapunov function in Eq.~\eqref{eq:lyapunov_discrete} and choosing $r_k := 1+\lambda_k$, the last inequality reads as $\E[\EE_k-\EE_{k-1}]\le \frac{d\eta^2}{2}\varsigma_*^2$ for $k\ge 1$ and by a telescoping sum argument we obtain \begin{equation*}
 \E[\EE_k-\EE_{0}]\le \frac{d\eta^2\varsigma_*^2}{2} (k-1) .  
\end{equation*}

That is, 
\begin{multline*}
    \E\left[\eta (1+\lambda_k) (f(x_k)-f(x^*))\right]\le\E\left[\eta (1+\lambda_k) (f(x_k)-f(x^*)) + \frac{1}{2}\| x_{k+1}- x^* + \lambda_{k+1}m_{k+1} \|^2\right]\\ \le \E\left[\eta (1+\lambda_{0})(f(x_{k-1})-f(x^*)) + \frac{1}{2}\| x_{1} -x^* + \lambda_1 m_1\|^2\right]+ \frac{d\eta^2\varsigma_*^2}{2} (k-1).
\end{multline*}

Following the same procedure as above, we get
$$
\E\left[\frac{1}{2}\| x_{1}- x^* + \lambda_{1}m_{1} \|^2\right] \le \E\left[\frac{1}{2}\| x_{0} -x^*\|^2\right] + \frac{d\eta^2}{2}\varsigma_0^2-\eta \E\left[f(x_0)-f(x^*)\right],
$$
which can also be written as
$$
\eta \E\left[f(x_0)-f(x^*)\right] + \E\left[\| x_{1}- x^* + \lambda_{1}m_{1} \|^2\right] \le \frac{1}{2}\| x_{0} -x^*\|^2 + \frac{d\eta^2\varsigma_*^2}{2}.
$$

The proof is concluded once we set $\lambda_0=0$
\end{proof}

We now apply the lemma above to get an algorithm and a convergence rate. In particular, we want to implement polynomial memory of past gradients and still get the rate found in Thm.~\ref{prop:memory_continuous_WQC}.

\begin{theorem}
Assume \textbf{(H1)} and \textbf{(H0d)} hold. Consider the following iterative method

\begin{equation*}
        x_{k+1} = x_k + \frac{k}{k+p}(x_k-x_{k-1}) - \frac{p}{k+p} \eta \nabla f_{i_k}(x_k)
        \label{eq:poly_mem_app}
\end{equation*}

with $p\ge 2$ and $\eta\le \frac{p-1}{pL}$. We have
$$\E[f(x_k)-f(x^*)] \le \underbrace{\frac{(p-1)^2\|x_0-x^*\|^2}{2 \eta p (k+p-1)} }_{\text{rate of convergence to suboptimal sol.}} + \underbrace{\frac{1}{2}d \eta \varsigma_*^2 p}_{\text{suboptimality}}.$$
Moreover, the resulting update direction can be written as $x_{k+1}-x_k = -\eta\sum_{i=0}^k w(i,k)\nabla f(x_i)$ with $\sum_{i=0}^k w(i,k) = 1$ and $w(\cdot,k):\{0,\dots,k\}\to \R$ behaving like a $(p-1)$-th order polynomial, that is $w(i,k) \sim i^{p-1}$, for all $k$.
\label{thm:discretization}
\end{theorem}

\begin{proof}
In the context of Lemma \ref{lemma:master_discrete}, pick the continuous-time inspired (see definition of $\lambda(t)$ in the proof of Prop.~\ref{prop:memory_continuous_WQC}) function $\lambda_k = \frac{k}{p-1}\le k$. If $r_k= \eta(\lambda_k +1) = \eta\frac{k+p-1}{p-1}$, $\beta_k = \frac{\lambda_k}{\lambda_{k+1}+1} = \frac{k}{k+p}$, $\delta_k = \frac{1}{\lambda_{k+1}+1} = \frac{p-1}{k+p}$, the iterates defined by
$$x_{k+1} = x_k + \frac{k}{k+p}(x_k-x_{k-1}) - \frac{p-1}{k+p} \eta \nabla f_{i_k}(x_k)$$
are such that, under $\eta\le 1/L$,
$$\E[f(x_k)-f(x^*)] \le  \frac{\|x_0-x^*\|^2}{2 r_k}  + \frac{d \eta^2 \varsigma_*^2}{2} \frac{k}{r_k} =\frac{(p-1)\|x_0-x^*\|^2}{2 \eta(k+p-1)} + \frac{d \eta \varsigma_*^2 (p-1)}{2} \frac{k}{(k+p-1)}.$$

The algorithm above is similar to Eq~\eqref{eq:poly_mem_app}. Indeed, if we define $\tilde\eta = \frac{p-1}{p}\eta$, then, under $\tilde\eta\le\frac{p-1}{pL}$, the iterates defined by Eq~\eqref{eq:poly_mem_app} are such that

$$\E[f(x_k)-f(x^*)] \le \frac{(p-1)^2\|x_0-x^*\|^2}{2 \tilde\eta p (k+p-1)} + \frac{1}{2}d \tilde\eta \varsigma_*^2 p \frac{k}{(k+p-1)} \le \frac{(p-1)^2\|x_0-x^*\|^2}{2 \tilde\eta p (k+p-1)}  + \frac{1}{2}d \tilde\eta \varsigma_*^2 p.$$

Let us rename $\tilde\eta$ to $\eta$; it's easy to realize by induction that
$$x_{k+1}-x_k = - \eta\sum_{j=0}^{k-1} \left(\prod_{h=j+1}^k\frac{h}{h+p}\right) \frac{p}{j+p} \nabla f(x_j) - \eta\frac{p}{k+p}\nabla f(x_k) = -\eta\sum_{j=0}^k w(j,k)\nabla f(x_j),$$
therefore, using some simple formulas \footnote{The reader can check the formula in Wolphram Alpha\textregistered :  \texttt{http://tinyurl.com/y3quchcd}} from number theory
\begin{align*}
    \sum_{j=0}^k w(j,k) &= \sum_{j=0}^{k-1} \left(\prod_{h=j+1}^k\frac{h}{h+p}\right) \frac{p}{j+p}+\frac{p}{k+p}\\
    &= \sum_{j=0}^{k-1} \frac{(j+1)(j+2)\cdots\cancel{(j+p)}}{(k+1)(k+2)\cdots(k+p)} \frac{p}{\cancel{(j+p)}} + \frac{p}{k+p}\\
    &=  \frac{p}{(k+1)(k+2)\cdots(k+p)}\sum_{j=0}^{k}(j+1)(j+2)\cdots(j+p-1),\\
\end{align*}
From the last formula, we see that indeed the weights behave like a $(p-1)$-order polynomial.
To conclude, notice that
$$\sum_{j=0}^{k}(j+1)(j+2)\cdots(j+p-1) = \frac{(k+1)(k+2)\cdots(k+p)}{p}.$$
\end{proof}

\section{Numerical simulation of memory and momentum ODEs}
\label{app:ODEs_quadratic}
We recall below some ODEs we introduced throughout the paper.
\begin{table}[ht]
\label{sample-table}
\begin{center}
\begin{tabular}{l|l|l}
Forgetting & corresponding viscosity in \ref{HB-ODE} & \ref{MG-ODE} \\
\hline
&&\\
Constant    &$\ddot X + \frac{1}{t}\dot X +\nabla f(X)=0$ & $\ddot X + \frac{1}{t}\dot X +\frac{1}{t}\nabla f(X)=0$ \\
&&\\
Linear      &$\ddot X + \frac{2}{t}\dot X +\nabla f(X)=0$ & $\ddot X + \frac{2}{t}\dot X +\frac{2}{t}\nabla f(X)=0$ \\
&&\\
Quadratic   &$\ddot X + \frac{3}{t}\dot X +\nabla f(X)=0$ & $\ddot X + \frac{3}{t}\dot X +\frac{3}{t}\nabla f(X)=0$ \\
&&\\
Cubic   &$\ddot X + \frac{4}{t}\dot X +\nabla f(X)=0$ & $\ddot X + \frac{4}{t}\dot X +\frac{4}{t}\nabla f(X)=0$ \\
&&\\
Exponential($\alpha$)   &$\ddot X + \frac{\alpha e^{\alpha t}}{e^{\alpha t}-1}\dot X +\nabla f(X)=0$ & $\ddot X + \frac{\alpha e^{\alpha t}}{e^{\alpha t}-1}\dot X +\frac{\alpha e^{\alpha t}}{e^{\alpha t}-1}\nabla f(X)=0$ \\
&$\approxeq\ddot X + \alpha\dot X +\nabla f(X)=0$ & $\approxeq\ddot X + \alpha\dot X +\alpha \nabla f(X)=0$ \\
&&\\
\hline
\end{tabular}
\end{center}
\label{tb:memory_momentum_odes}
\end{table}

Our setting is identical to the one outlined in Fig. 1 from~\cite{su2014differential}: we consider the quadratic objective $f(x_1,x_2) = 2\times10^{-2} x_1^2 + 5\times10^{-3} x_2^2$ and the numerical solution to the ODEs above starting from $X_0 = (1,1)$ and $\dot X(0) = (0,0)$. The solution is computed using the MATLAB \texttt{ode45} function, with absolute tolerance $10^{-10}$ end relative tolerance $10^{-5}$. We start integration at $t_0 =2.2204\times 10^{-16}$ (machine precision). For numerical stability, when using exponential momentum and memory, we pick $\alpha = 10$ and use the approximation $\frac{\alpha e^{\alpha t}}{e^{\alpha t}-1} \approxeq \alpha$ for large values of $t$. In the right part of the subplots, the dashed lines indicate the slope of the rates $1/t$, $1/t^2$, $1/t^3$, etc. Some comments follow.
\begin{enumerate}
    \item As noted by~\cite{su2014differential} in Thm. 8 of their paper , such second order equations can exhibit fast sublinear rates for quadratic objectives once the constant $p$ in the coefficient $p/t$ is increased.\\
    
    \item To realize the worst-case rate in the convex setting, the reader can look at the slope of on the right plot before the first inverse peak: for instance, we see in the HB-ODE with $3/t$ viscosity that the slope is aligned with $1/t^2$, while linear forgetting is aligned with $1/t$. This is predicted by the time-warping presented in App. \ref{time}.\\

    \item As predicted again by App. \ref{time}, $1/t$ viscosity has the same path as constant forgetting and $3/t$ viscosity goes along the same path as linear forgetting.
\end{enumerate}

\begin{figure}[ht]
\centering
\begin{minipage}{.5\textwidth}
  \centering
  \includegraphics[width=0.9\linewidth]{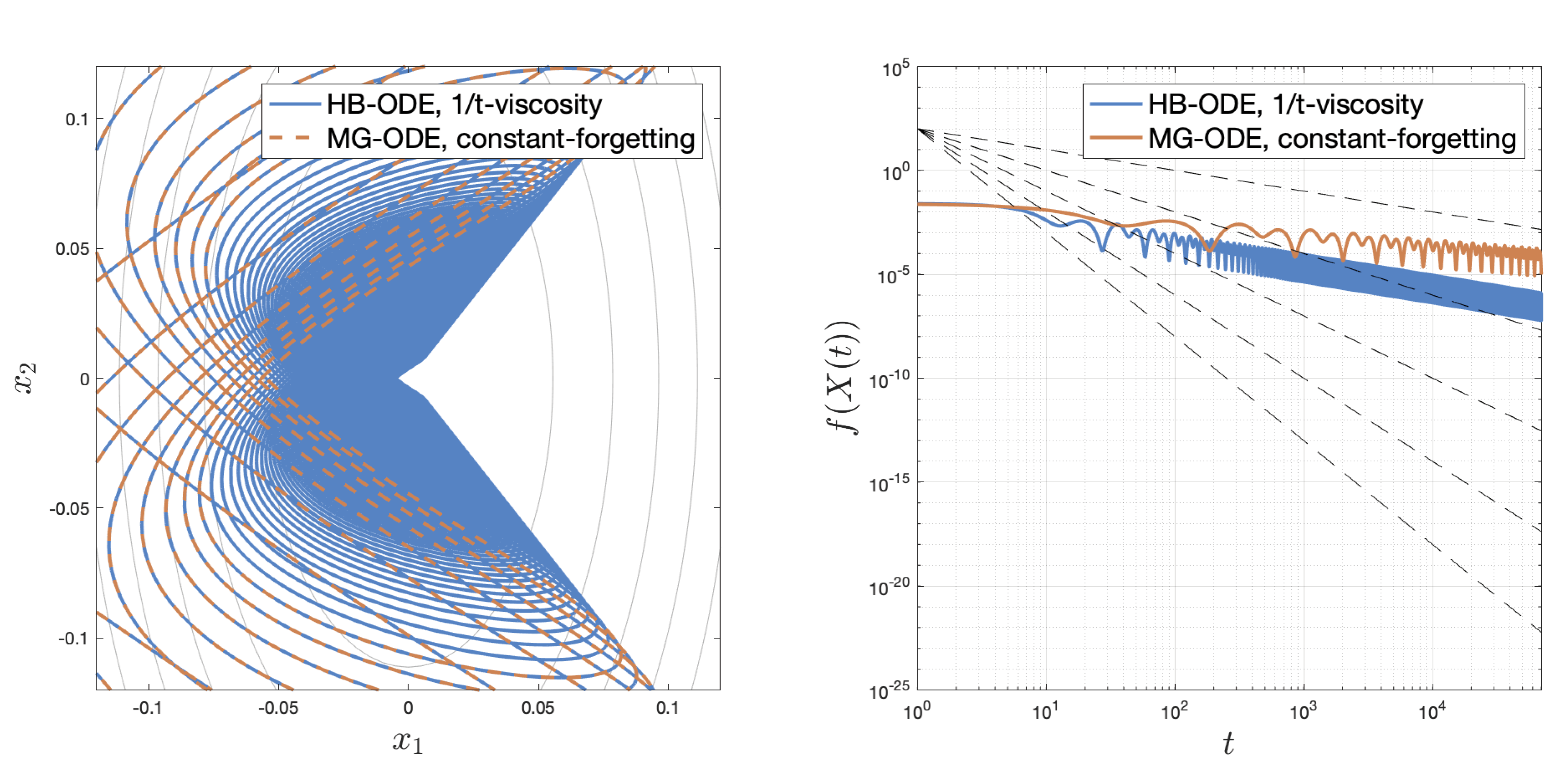}
\end{minipage}%
\begin{minipage}{.5\textwidth}
  \centering
  \includegraphics[width=0.9\linewidth]{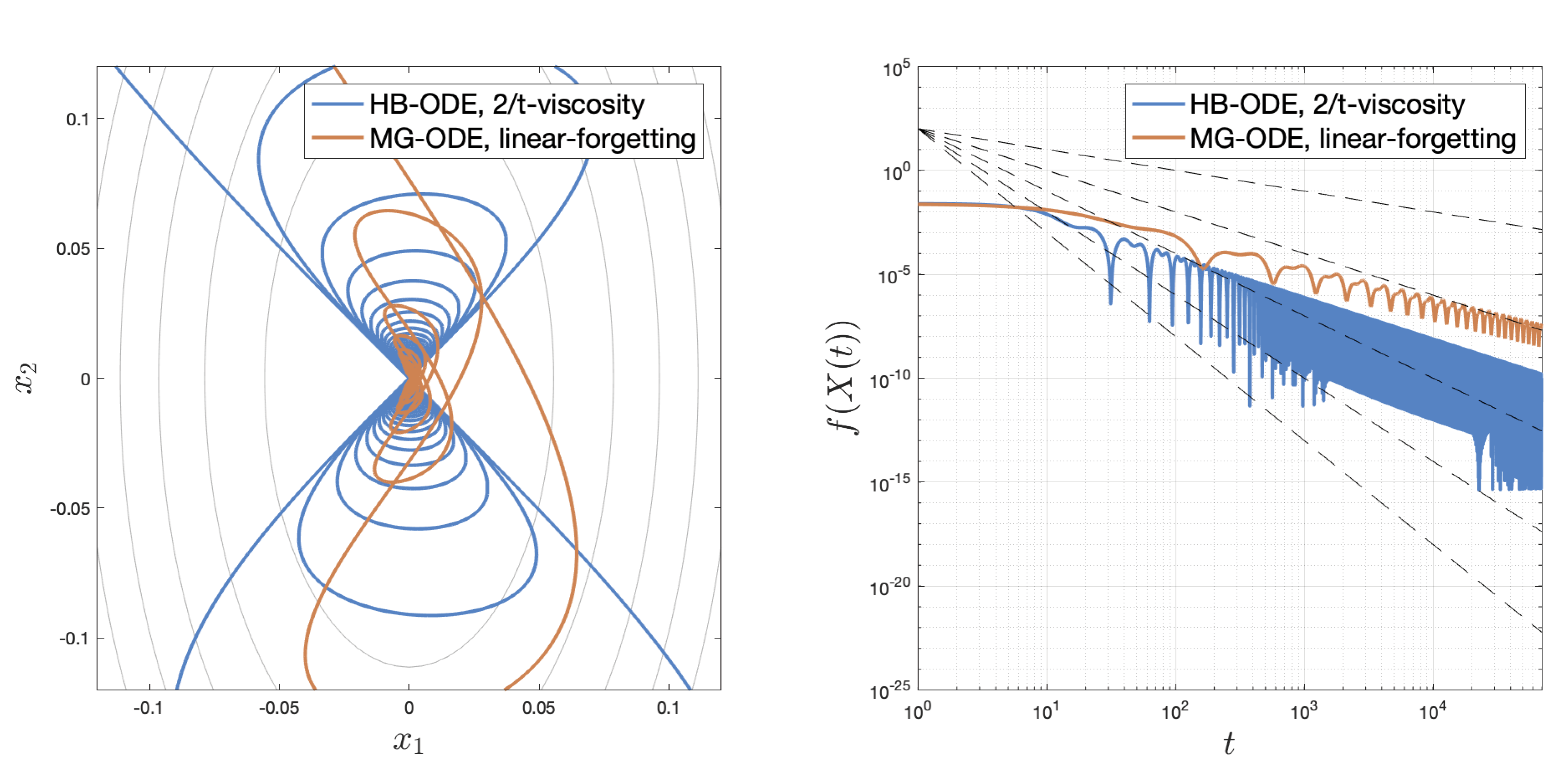}
  \end{minipage}%
\end{figure}

\begin{figure}[ht]
\centering
\begin{minipage}{.5\textwidth}
  \centering
  \includegraphics[width=0.9\linewidth]{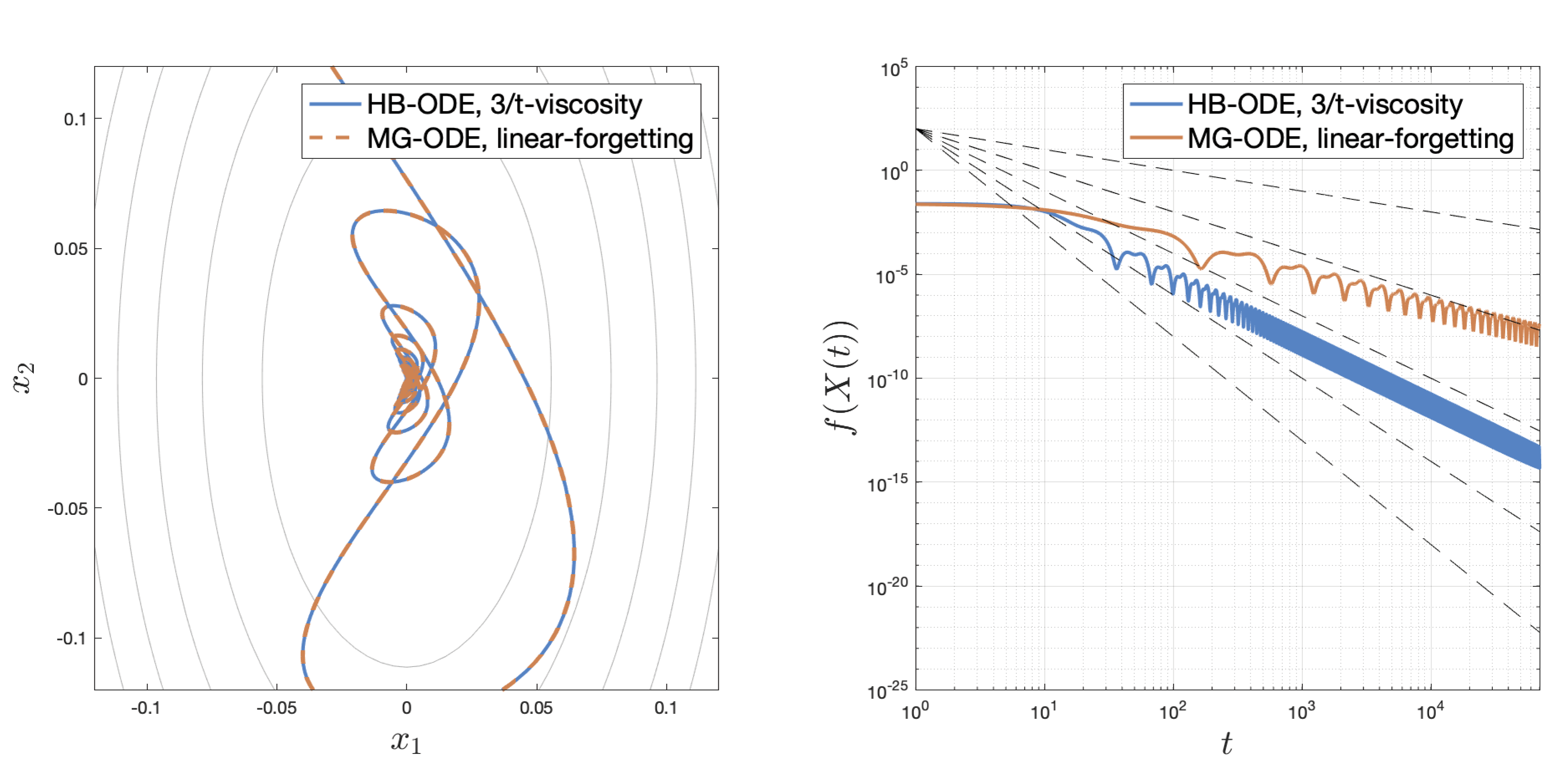}
\end{minipage}%
\begin{minipage}{.5\textwidth}
  \centering
  \includegraphics[width=0.9\linewidth]{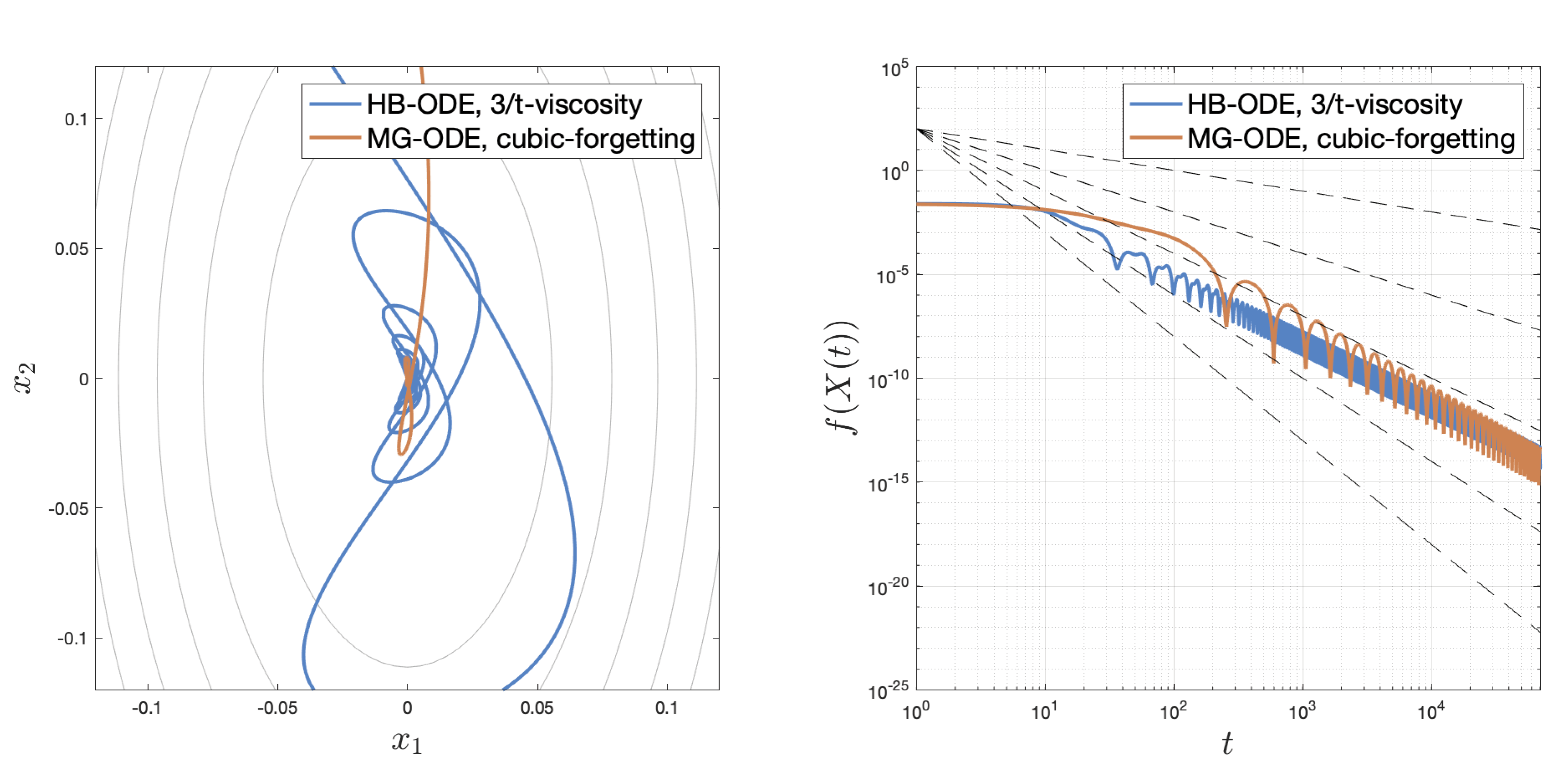}
  \end{minipage}%
\end{figure}

\section{Experiments on real-world datasets}
\subsection{Experimental setting}\label{sec:exp_setting}

\paragraph{Datasets and architecture}
We run four different sets of experiments. First, we optimize a logistic regression model on the covtype dataset ($n=464'809$) from the popular LIBSVM library. This model is strongly convex and has $d=432$ parameters. 

Secondly, we train an autoencoder introduced by \cite{hinton2006reducing} on the MNIST hand-written digits dataset ($n=60'000$). The encoder structure is $784-1000-500-250-30$ and the decoder is mirrored. Sigmoid activations are used in all but the central layer. The reconstructed images are fed pixelwise into a binary cross entropy loss. The network has a total of $2'833'000$ parameters.

Third, we optimize a simple feed-forward network on the Fashion-MNIST dataset ($n=60'000$). The network structure is $784-128-10$ with tanh activations in the hidden layer, cross entropy loss and a total of $101'770$ parameters.

Finally, we train a fairly small convnet on the CIFAR-10 dataset ($n=50'000$) taken from the official PyTorch tutorials (see here \texttt{https://pytorch.org/tutorials/beginner/blitz/cifar10\_tutorial.html}.). The total number of parameters in this network amounts to $62'006$.

Note that all neural networks models have sufficient expressive power to reach full training accuracy or (in the autoencoder case) reconstruct images very accurately. The linear model on covtype, however, achieves at most $\approx 70\%$ accuracy (see Fig.s \ref{fig:exp_acc} and \ref{fig:exp_rec}).

\paragraph{Algorithms.}
We benchmark several types of memory: (i) linear forgetting (p=2), (ii) cubic forgetting (p=4), higher degree polynomial forgetting (p=100) as well as exponential (p=e) and instantaneous forgetting (p=inf). 
Note that the last two algorithms exactly resemble Adam without adaptive preconditioning and vanilla SGD respectively. Furthermore, we also benchmark the classical Polyak Heavy Ball method (HB) as a reference point.
\paragraph{Parameters.} We run HB with a fixed momentum parameter $\beta=0.9$ across all experiments. For both SGD and HB we grid search the stepsize $\eta \in \{1,0.1,0.05,0.01,0.005,0.001,0.0005,0.0001,0.00005,0.00001\}$ and pick the best stepsize for each problem instance (reported in the corresponding legends). Interestingly, the best stepsize for HB is almost always one tenth of the SGD stepsize. The only exception is the convex logistic regression on covtype, where $\alpha_{SGD}=1$ was also the best stepsize for HB. There, we report $\alpha_{HB2}=0.1$ just for the sake of consistency. All versions of MemSGD simply run with the same stepsize as SGD, i.e. we did not grid-search the stepsize for MemSGD.

In order to assess the impact of stochasticity on SGD methods with memory, we run all algorithms in a large and a small batch setting. The large batch setting simply takes all training data points available in each dataset. For the small batch setting we chose the batch sizes as small as possible while still being able to train the networks in reasonable time. In particular, we take a mini-batch size of 16 for covtype and mnist. Fashion-mnist and Cifar-10 are run with 32 and 128 samples per iterations respectively\footnote{In fact training took much longer with smaller batch sizes for those datasets. We suspect that this is partly due to more complex optimization landscapes and partly due to less monotonicity across the data points compared to mnist and covtype.}.

 All of our experiments are run using the PyTorch library \cite{paszke2017automatic}.
 \ \\ \ \\
\begin{figure}[ht]
\centering 
          \begin{tabular}{c@{}c@{}c@{}c@{}}
          Covtype Logreg & MNIST Autoencoder & FashionMNIST MLP & CIFAR-10 CNN
          \\
            \includegraphics[width=0.23\linewidth,valign=c]{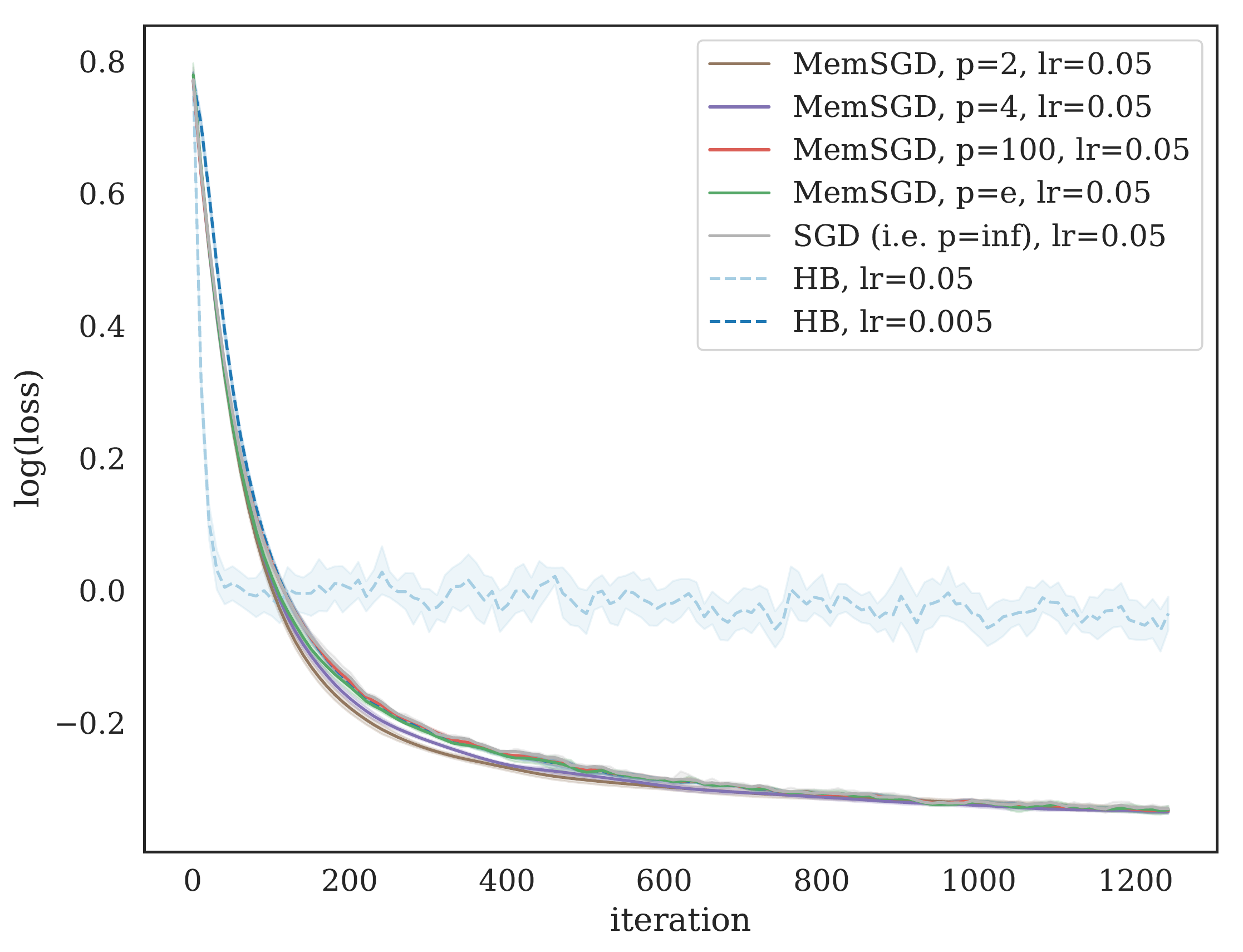}&
            \includegraphics[width=0.23\linewidth,valign=c]{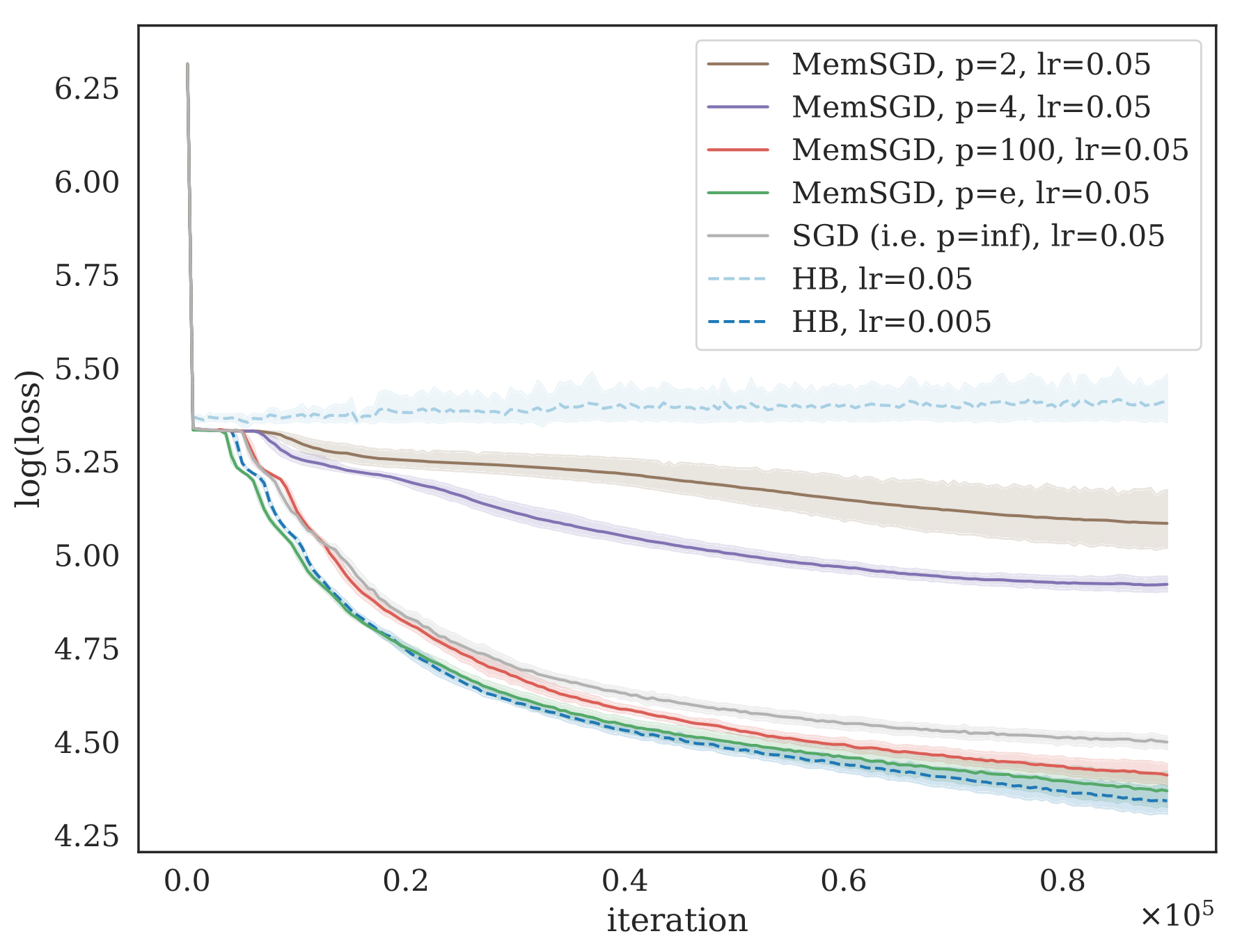}&
            \includegraphics[width=0.23\linewidth,valign=c]{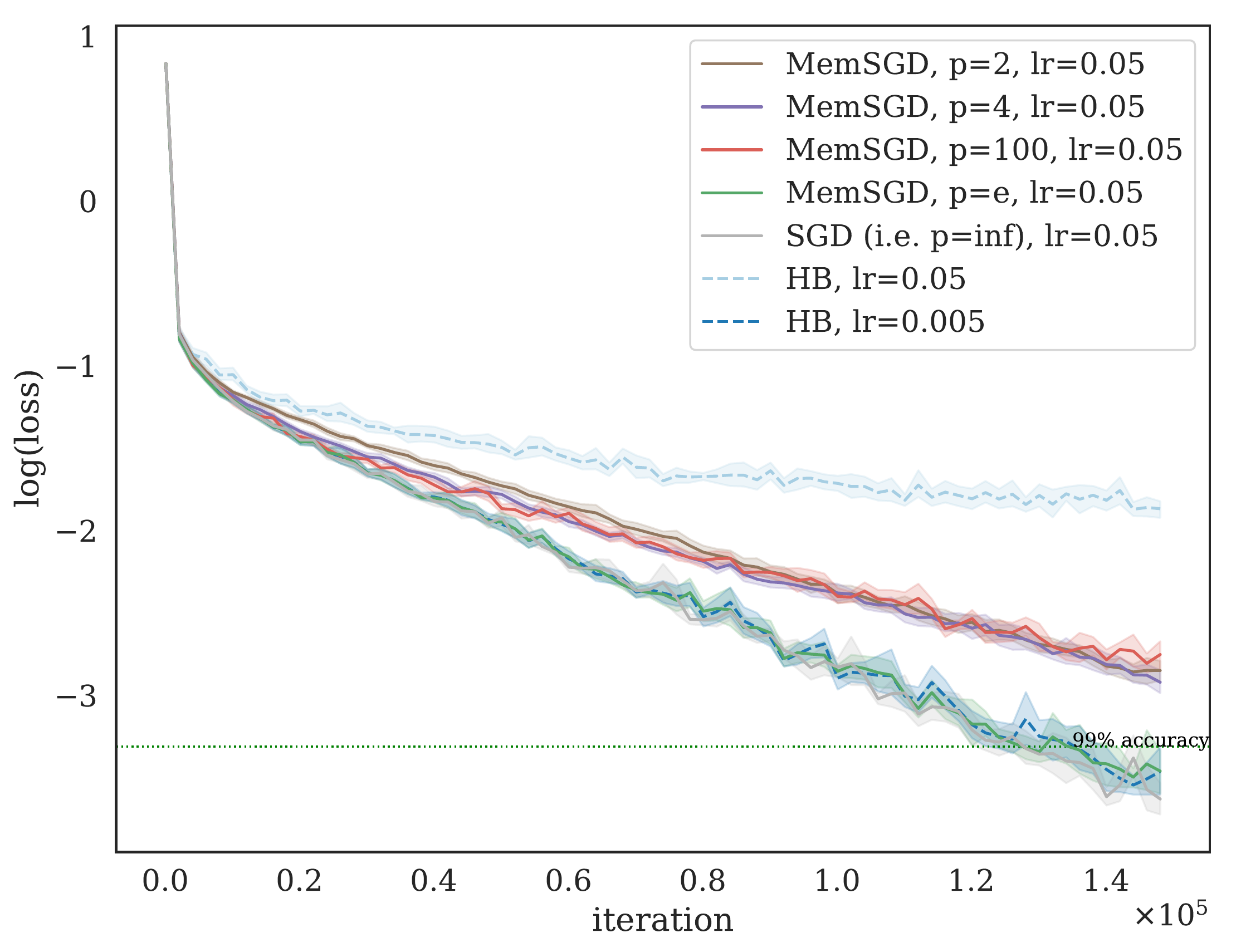}&
            \includegraphics[width=0.24\linewidth,valign=c]{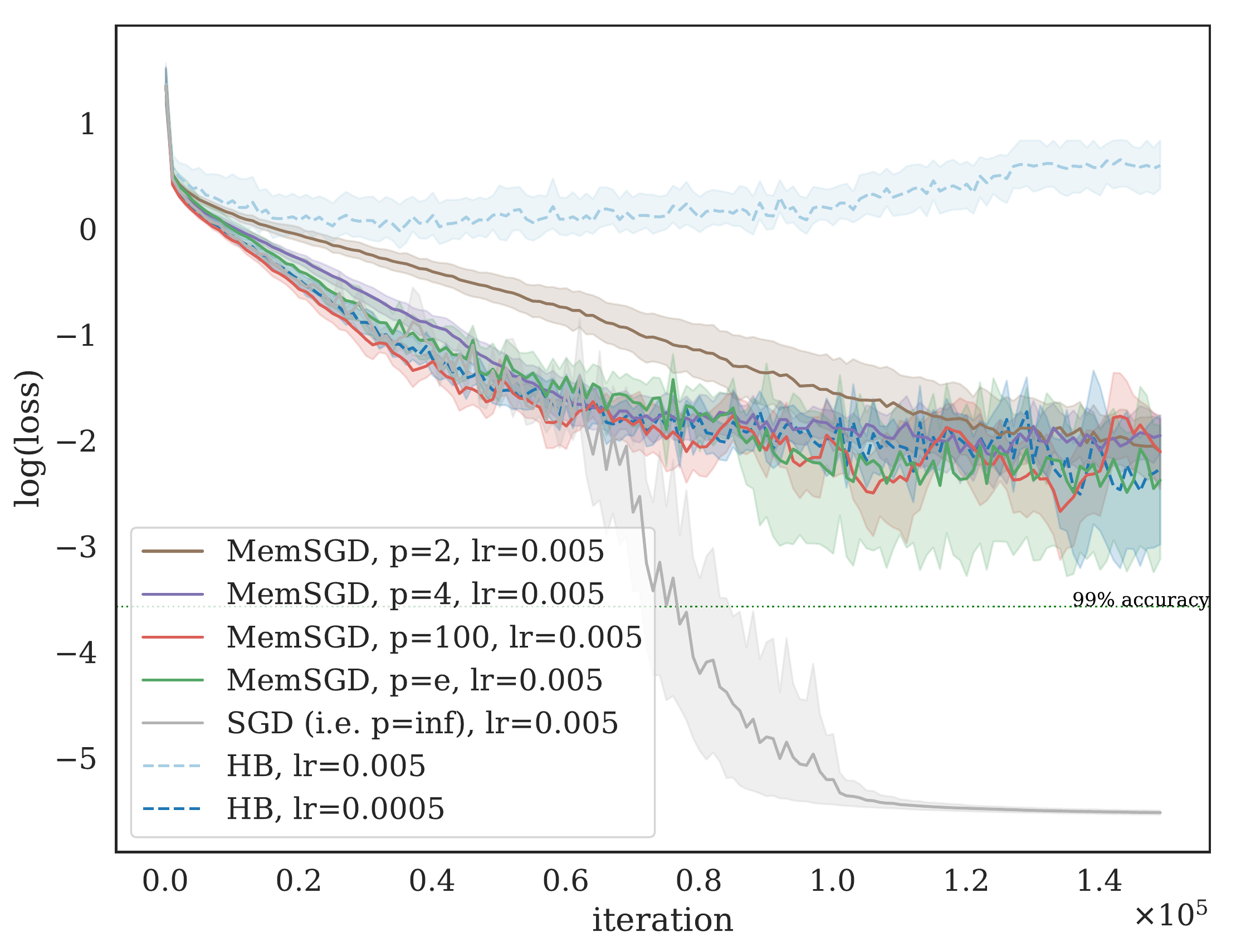}\\  
            ~
             \includegraphics[width=0.23\linewidth,valign=c]{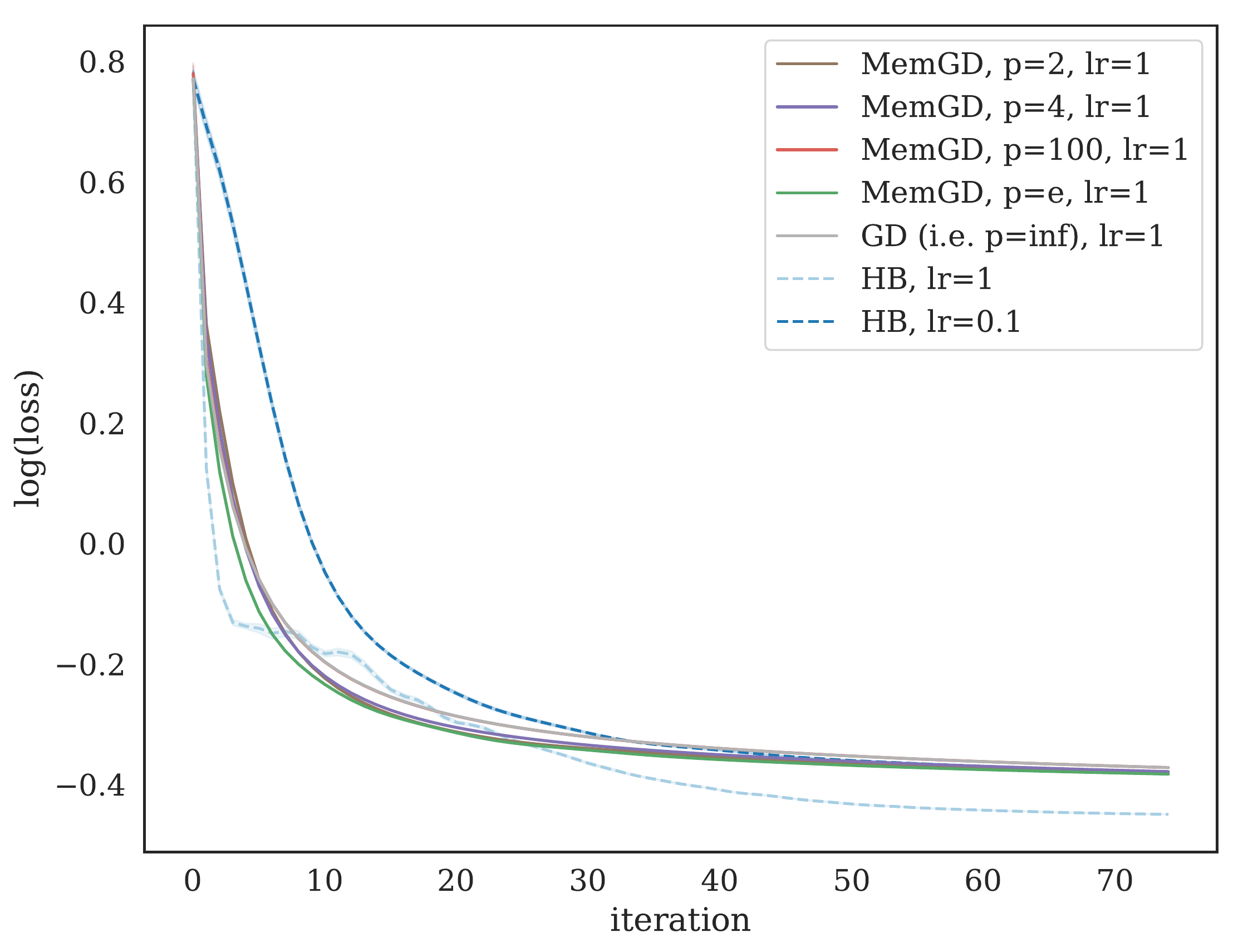}&
            \includegraphics[width=0.23\linewidth,valign=c]{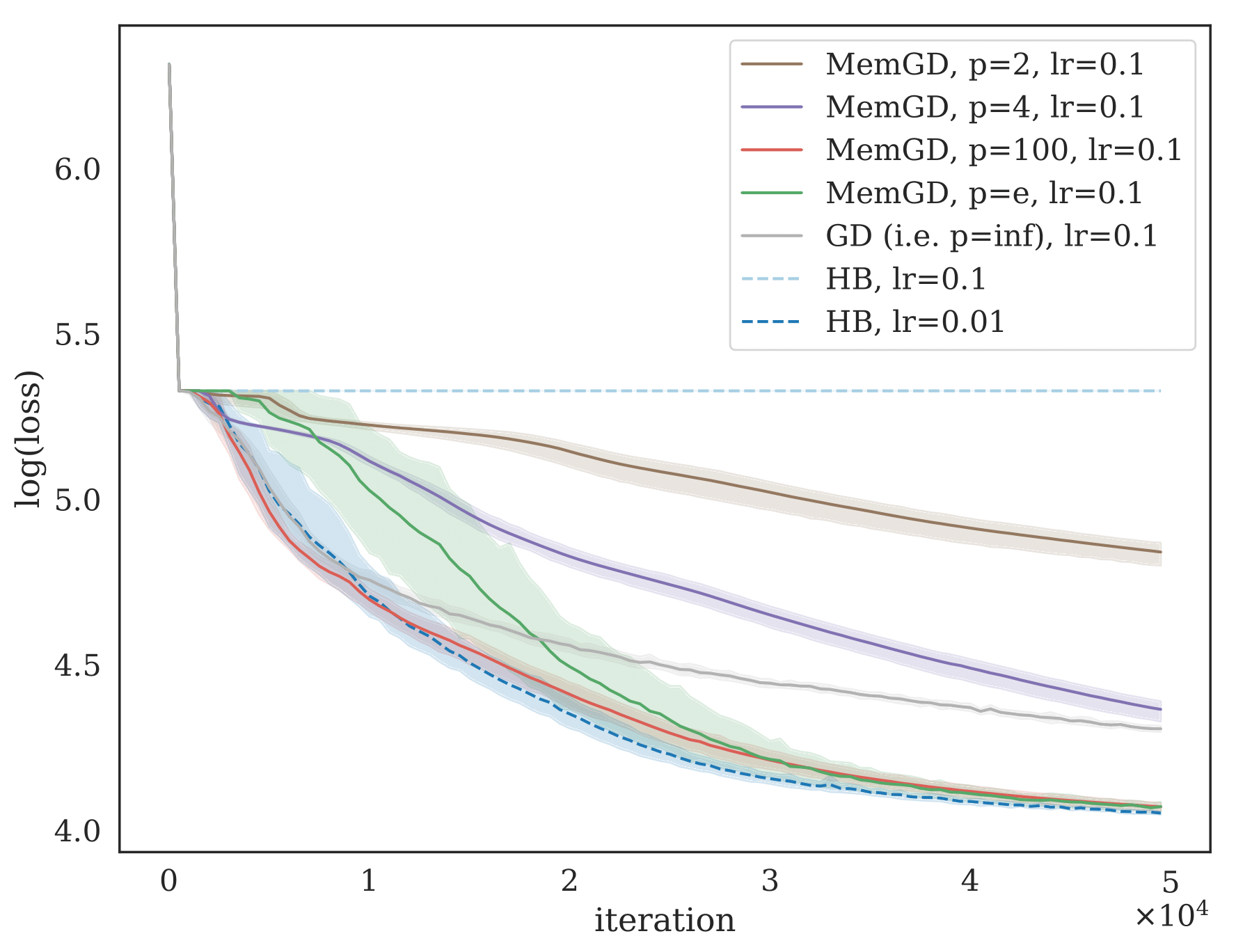}&
            \includegraphics[width=0.23\linewidth,valign=c]{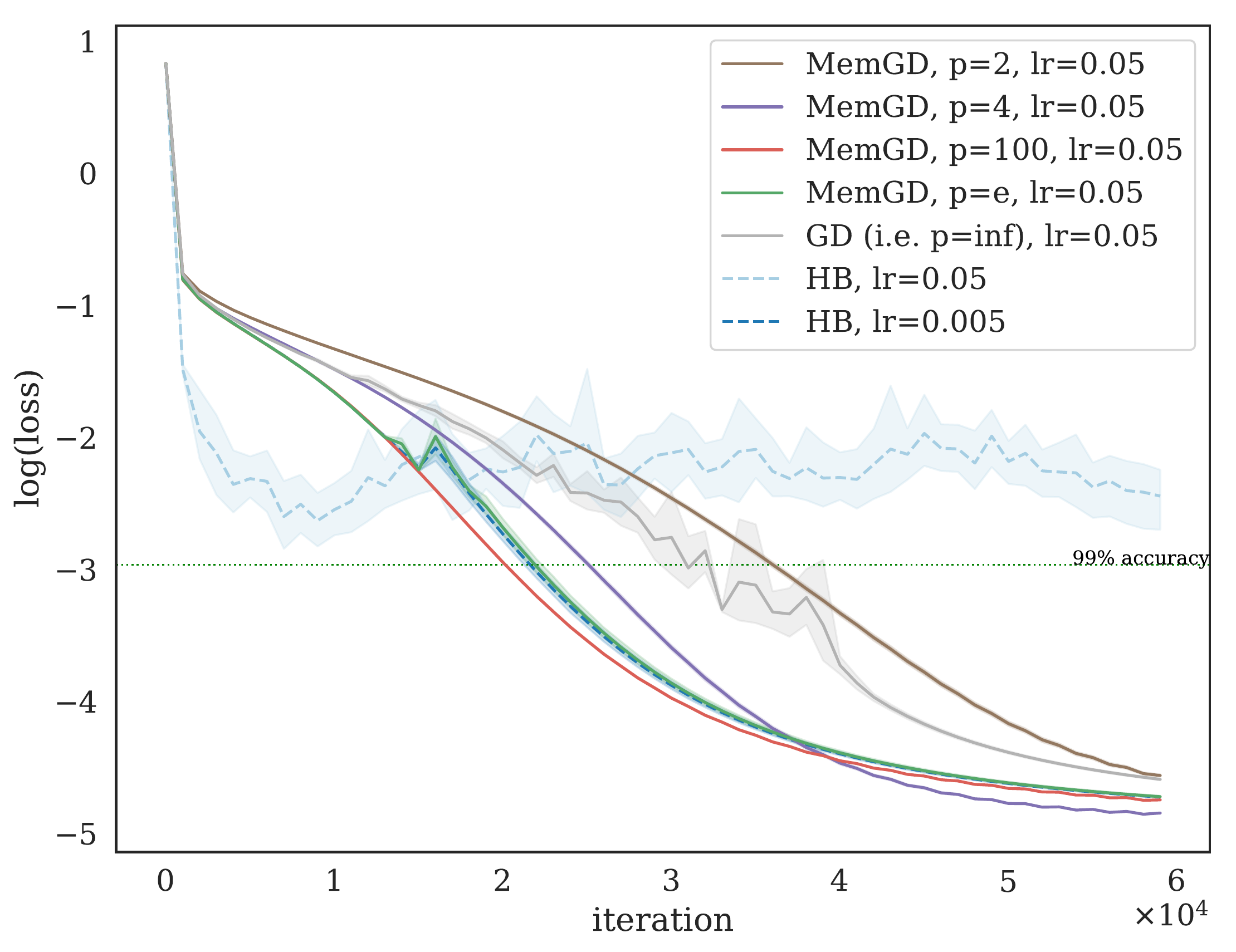}&
            \includegraphics[width=0.23\linewidth,valign=c]{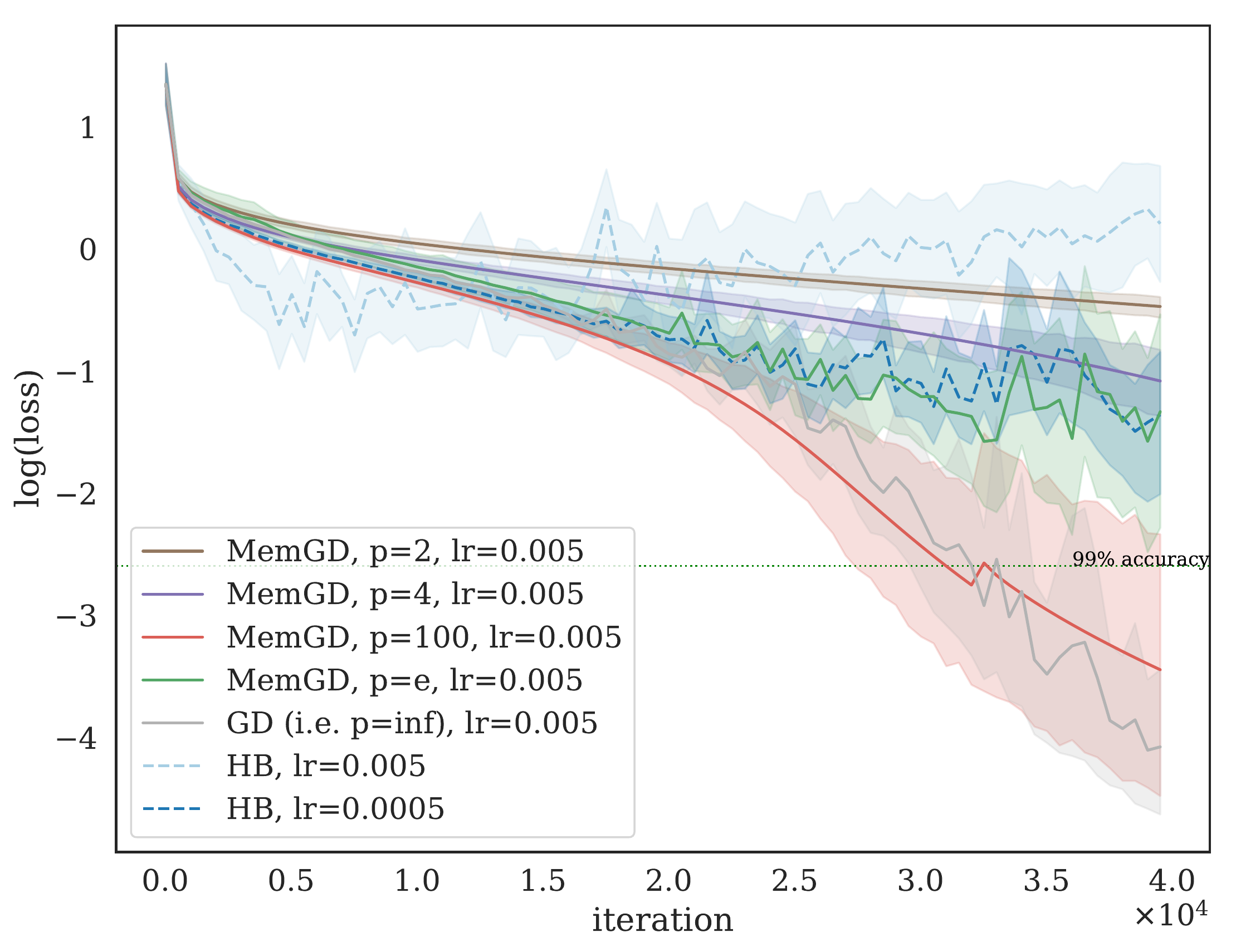}\\  
	  \end{tabular}
          \caption{Real world experiments: Log loss over iterations in mini- (top) and full batch (bottom) setting. Average and $95\%$ confidence interval of 10 runs with random initialization.}
     
\end{figure}

\begin{figure}[ht]
\centering 
          \begin{tabular}{c@{}c@{}c@{}c@{}}
          Covtype Logreg & MNIST Autoencoder & FashionMNIST MLP & CIFAR-10 CNN
          \\
            \includegraphics[width=0.23\linewidth,valign=c]{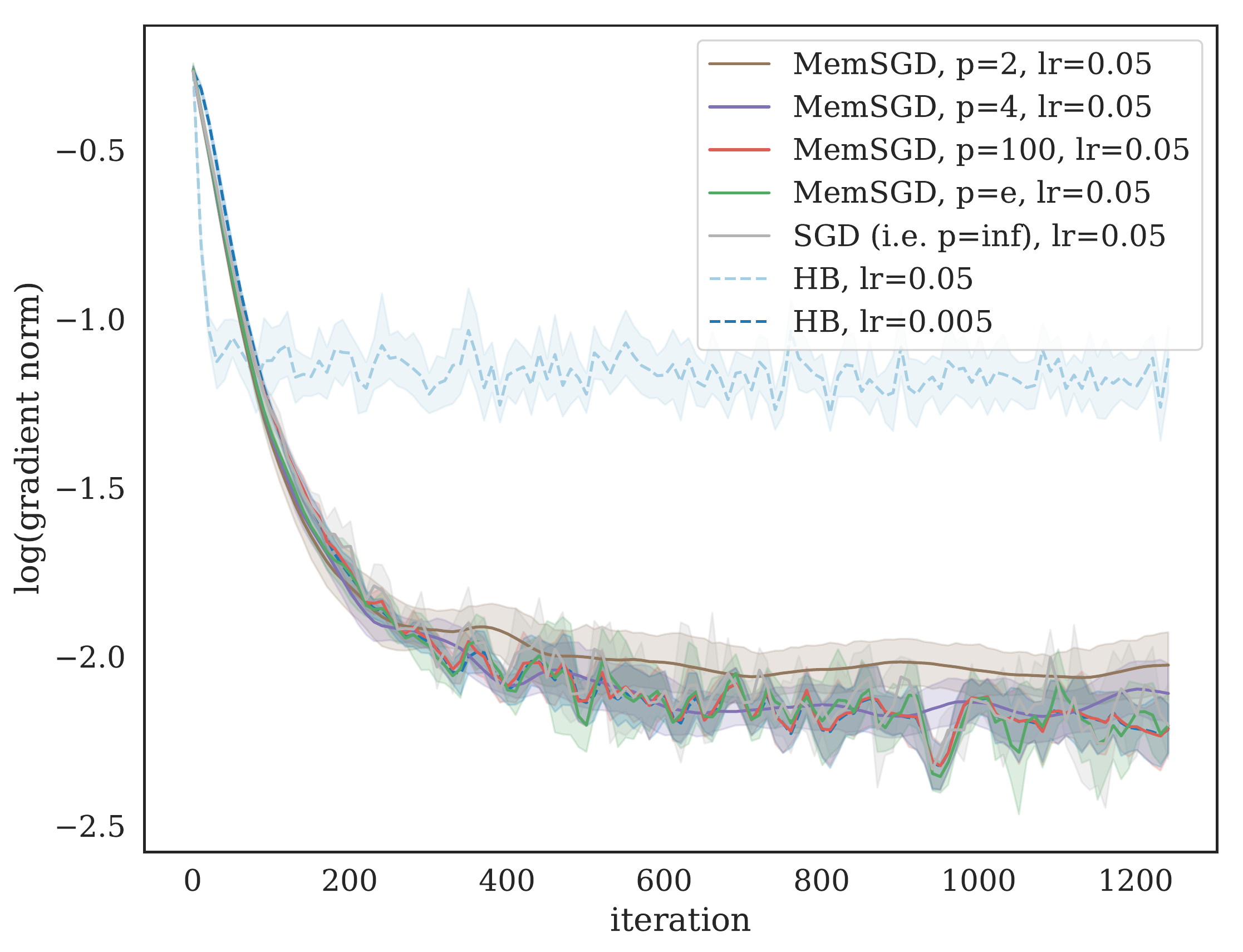}&
            \includegraphics[width=0.23\linewidth,valign=c]{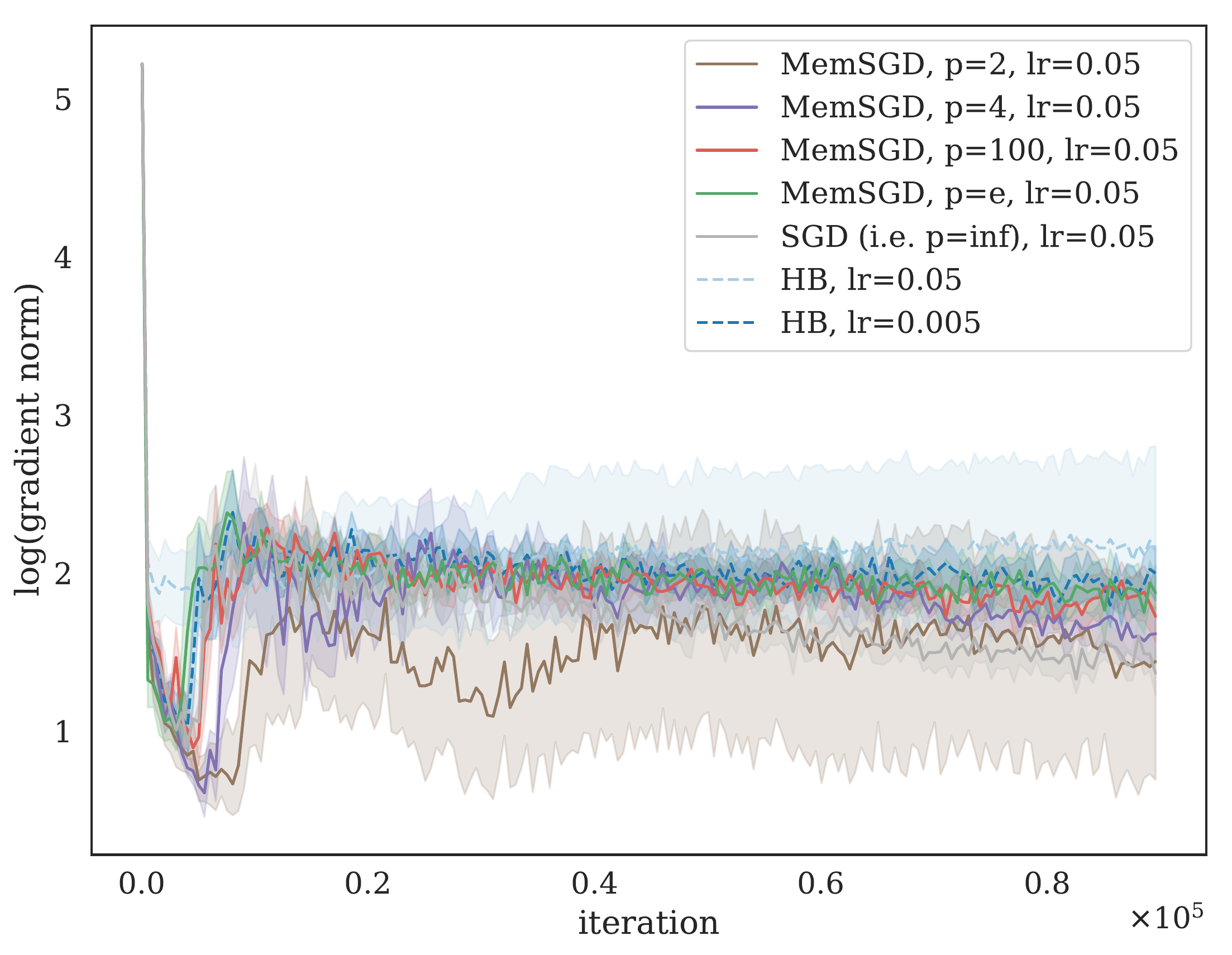}&
            \includegraphics[width=0.23\linewidth,valign=c]{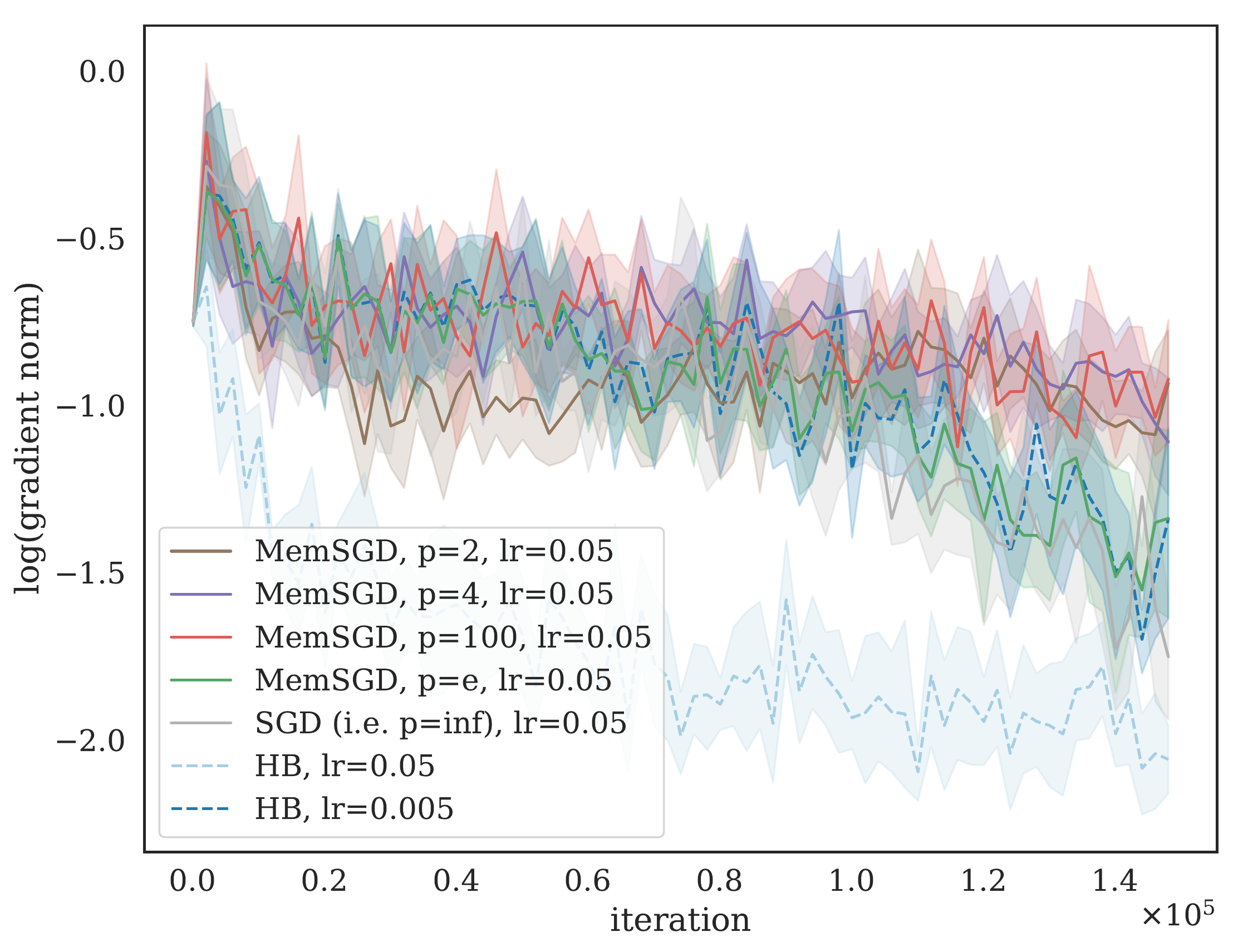}&
            \includegraphics[width=0.23\linewidth,valign=c]{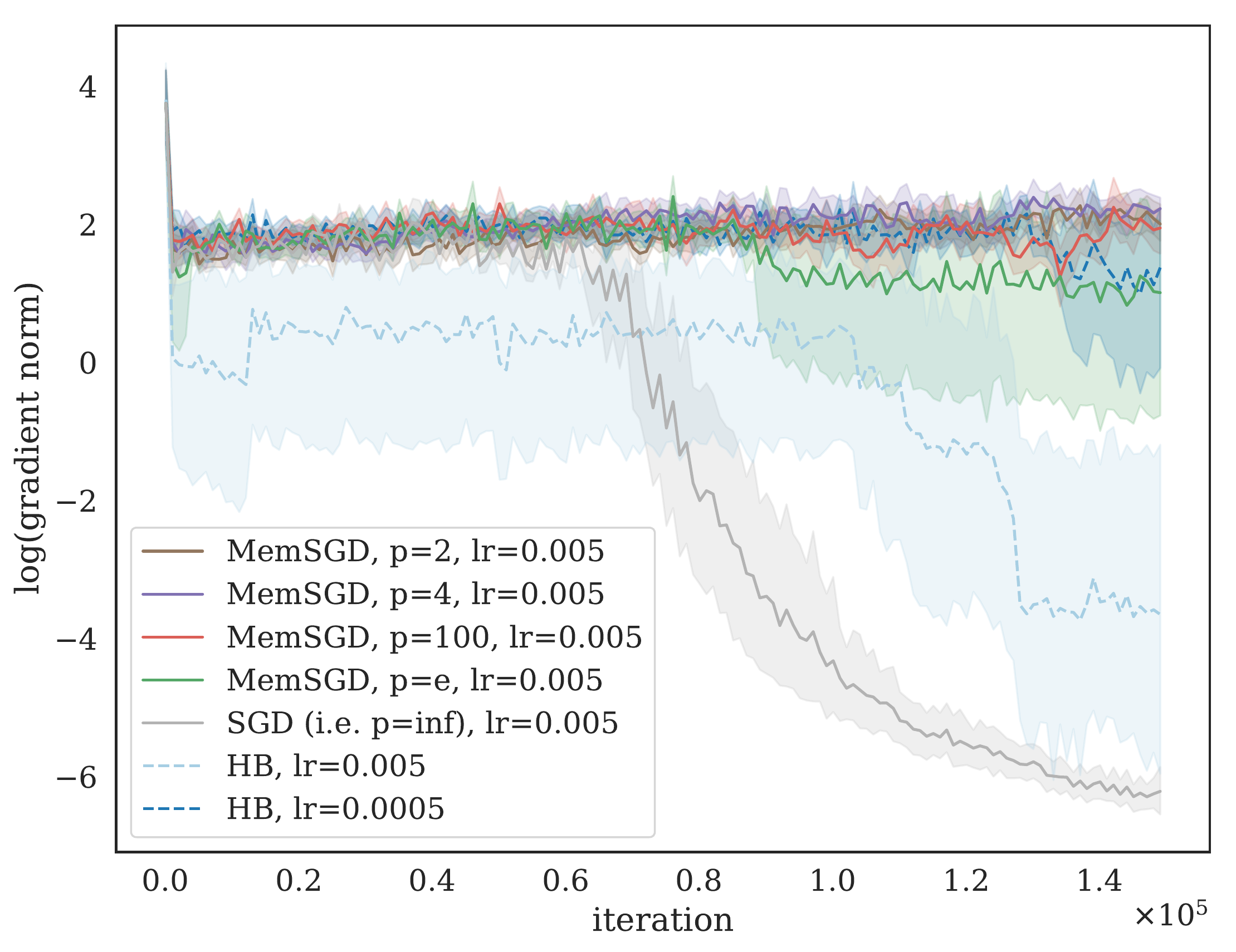}\\  
            ~
             \includegraphics[width=0.23\linewidth,valign=c]{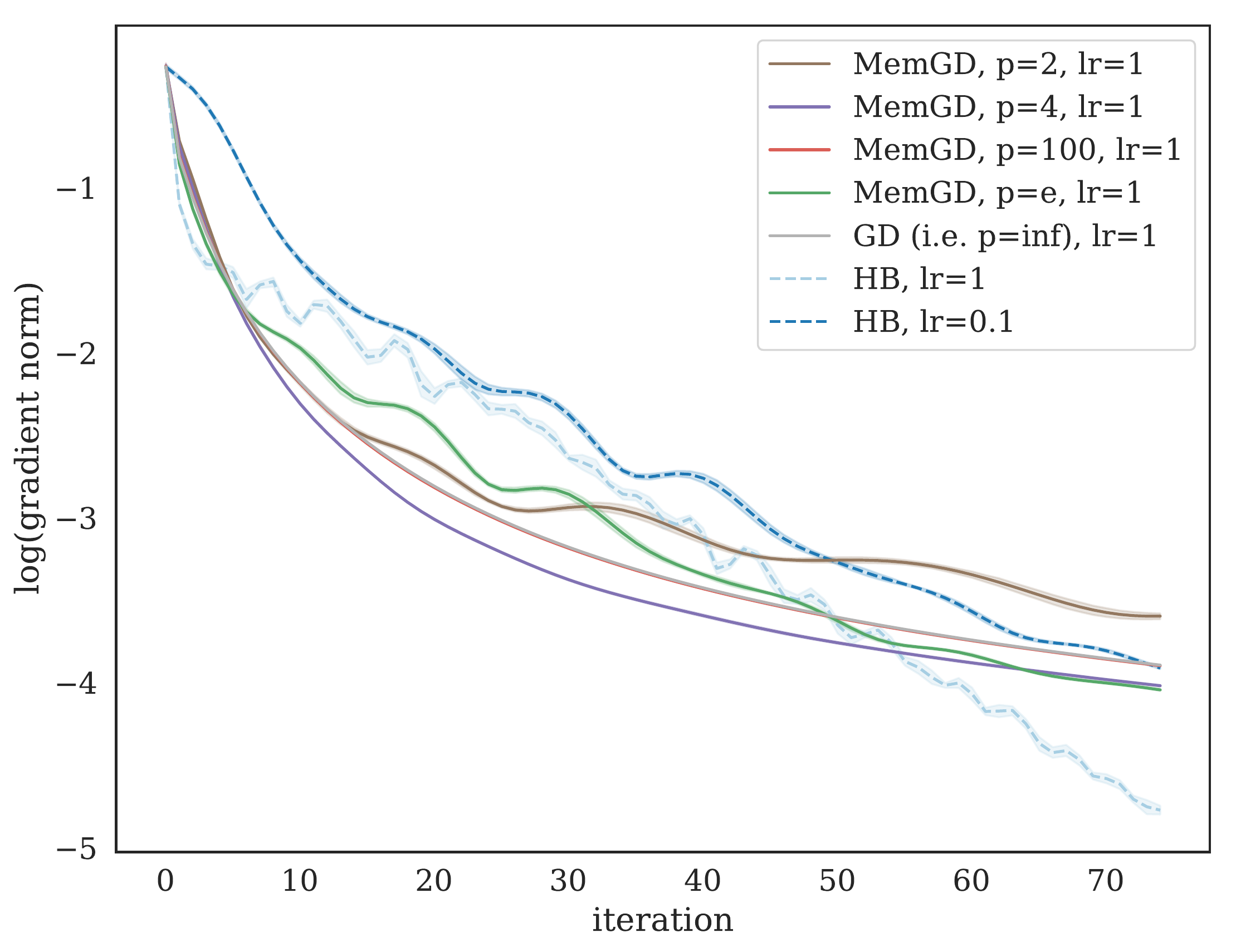}&
            \includegraphics[width=0.23\linewidth,valign=c]{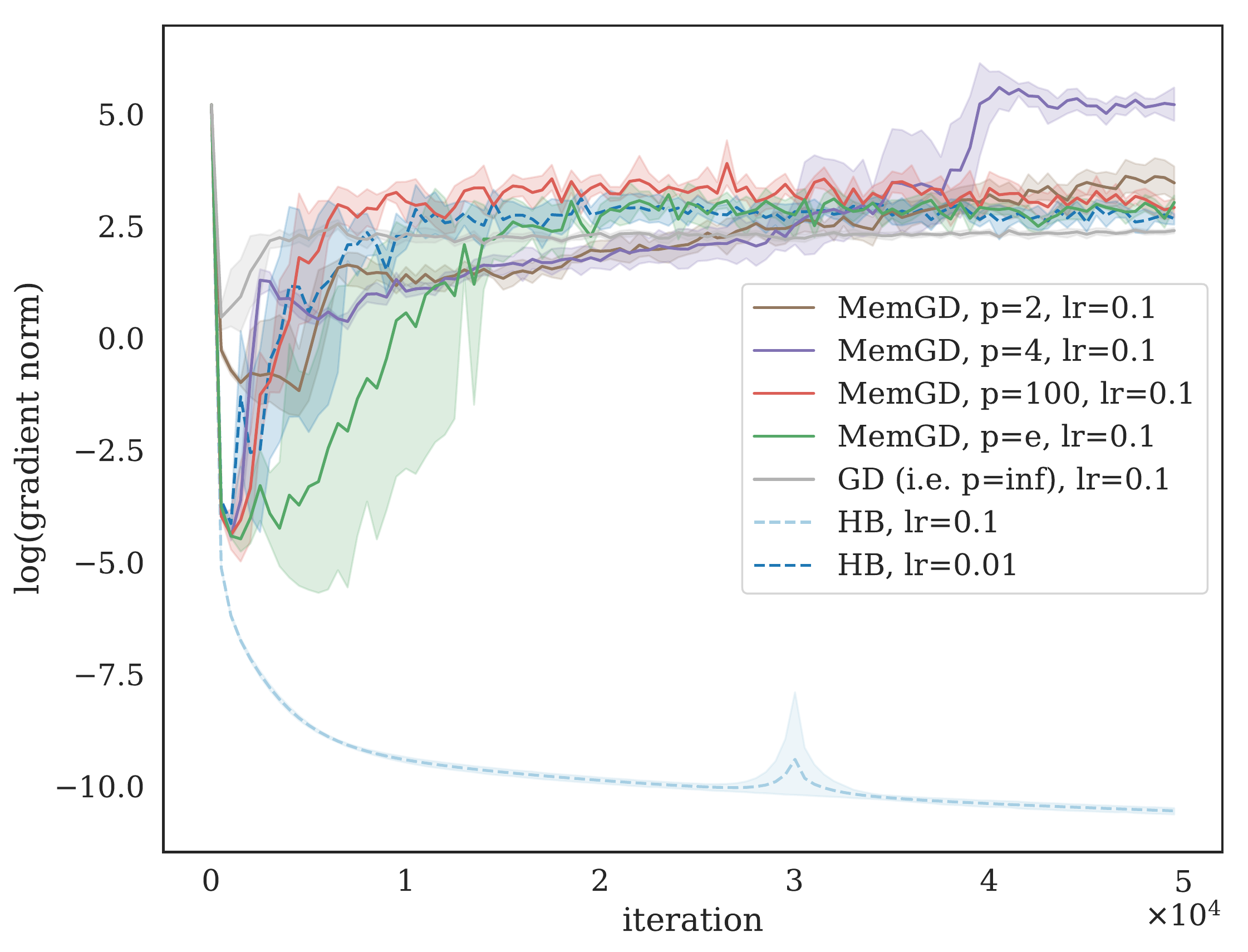}&
            \includegraphics[width=0.23\linewidth,valign=c]{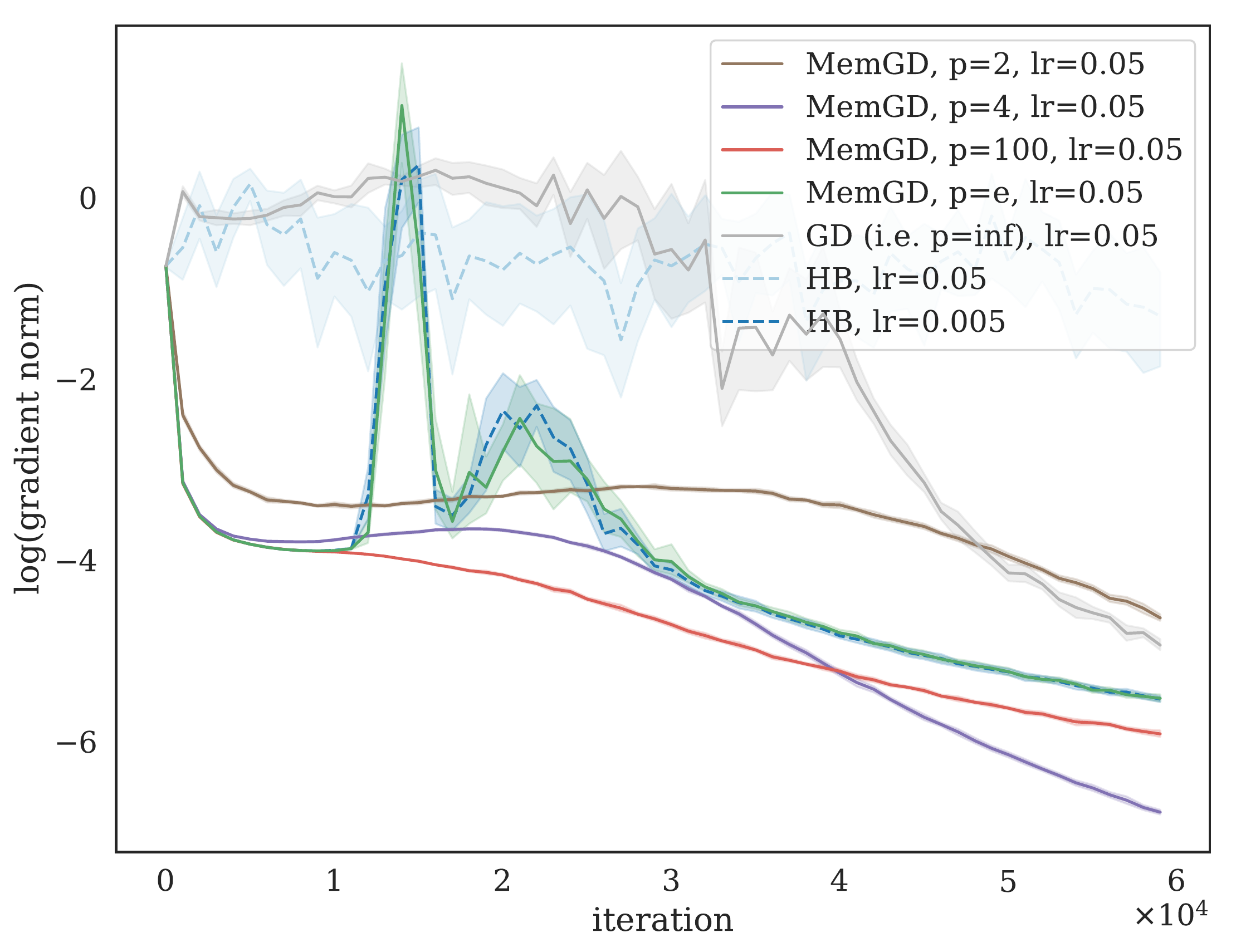}&
            \includegraphics[width=0.23\linewidth,valign=c]{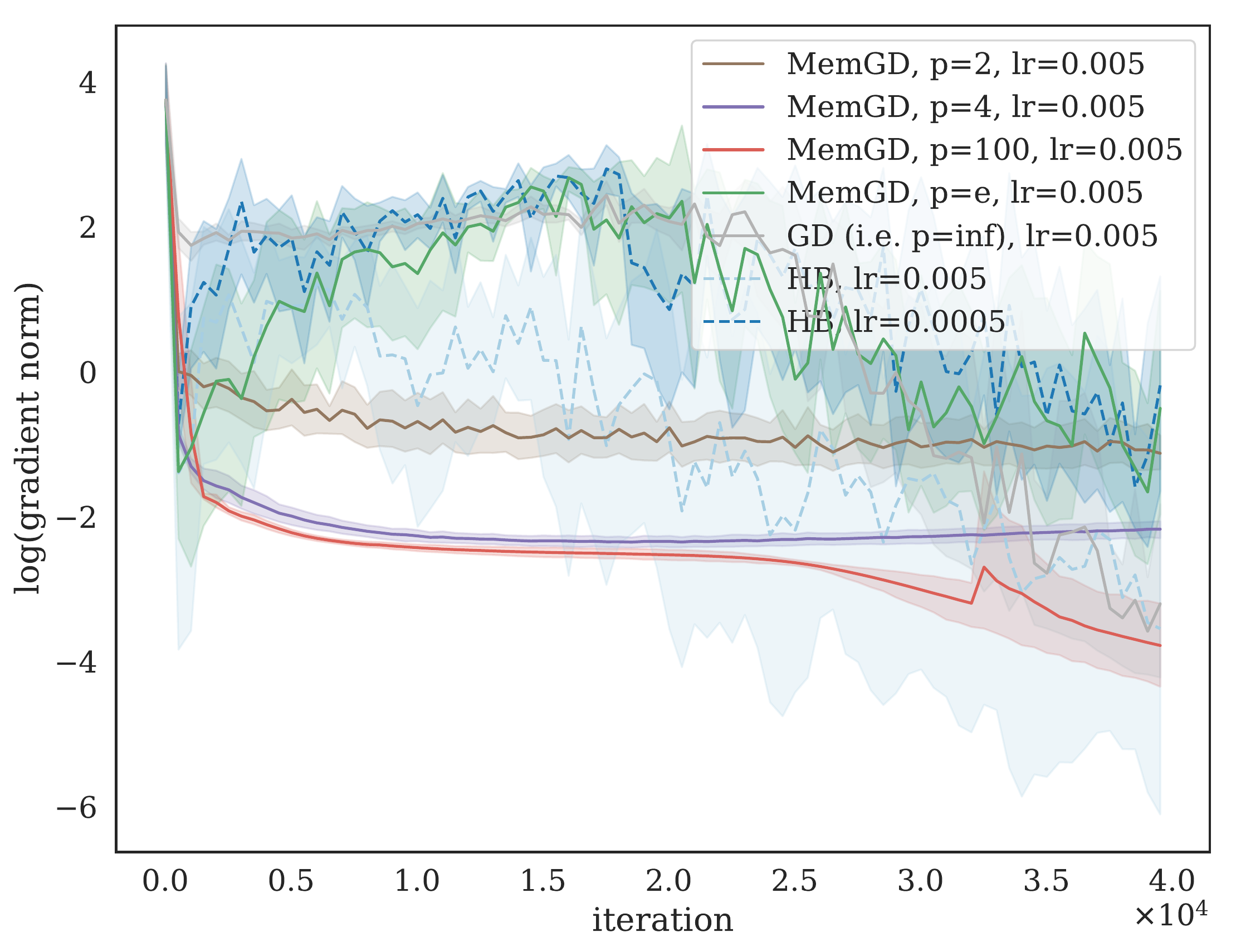}\\  
	  \end{tabular}
	  \vspace{-4mm}
          \caption{Real world experiments: Full gradient norm over iterations in mini- (top) and full batch (bottom) setting. Average and $95\%$ confidence interval of 10 runs with random initialization.}\label{fig:exp_loss}
     
\end{figure}

\begin{figure}[ht]
\centering 
          \begin{tabular}{c@{}c@{}c@{}}
          Covtype Logreg  & FashionMNIST MLP & CIFAR-10 CNN
          \\
            \includegraphics[width=0.23\linewidth,valign=c]{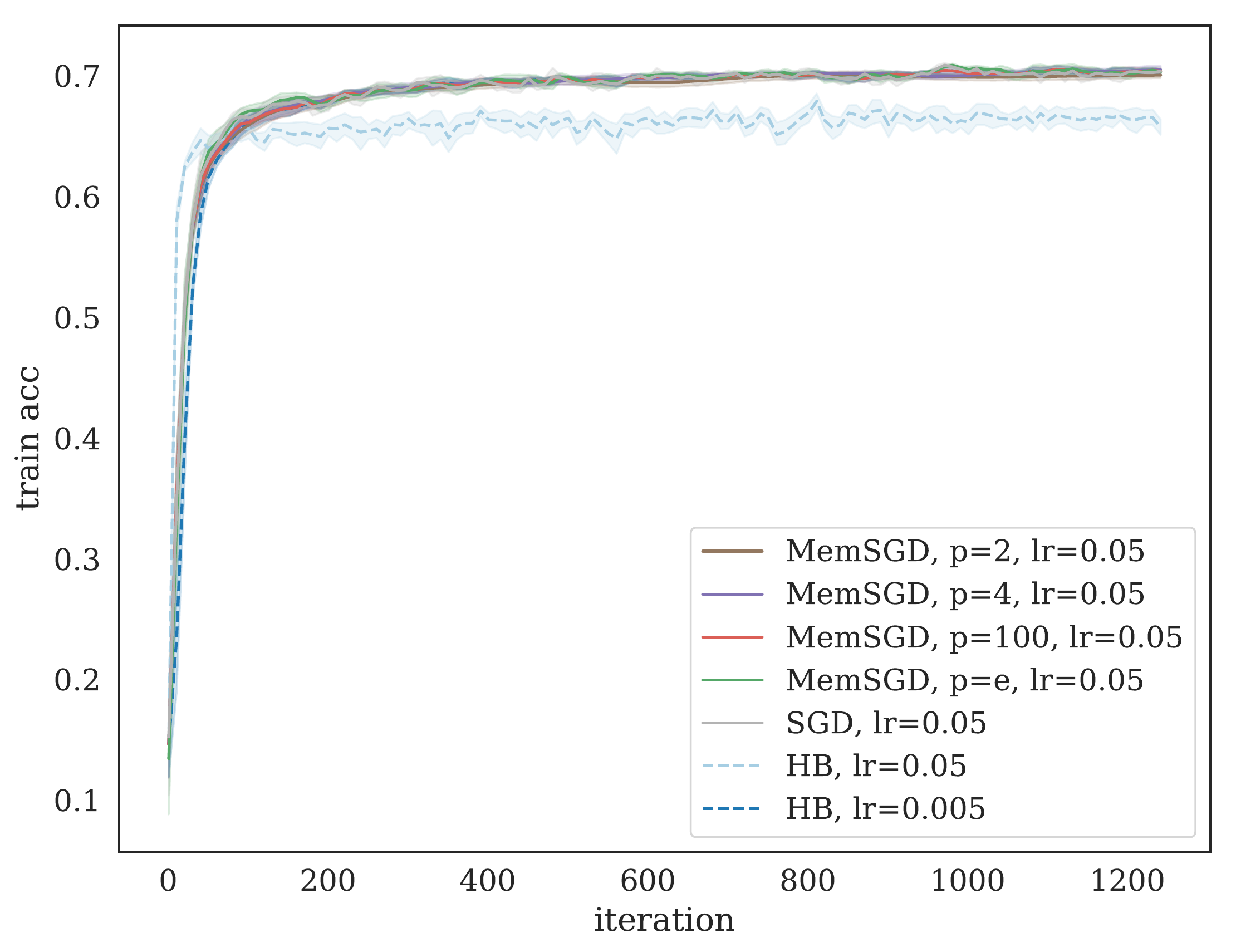}&
            \includegraphics[width=0.23\linewidth,valign=c]{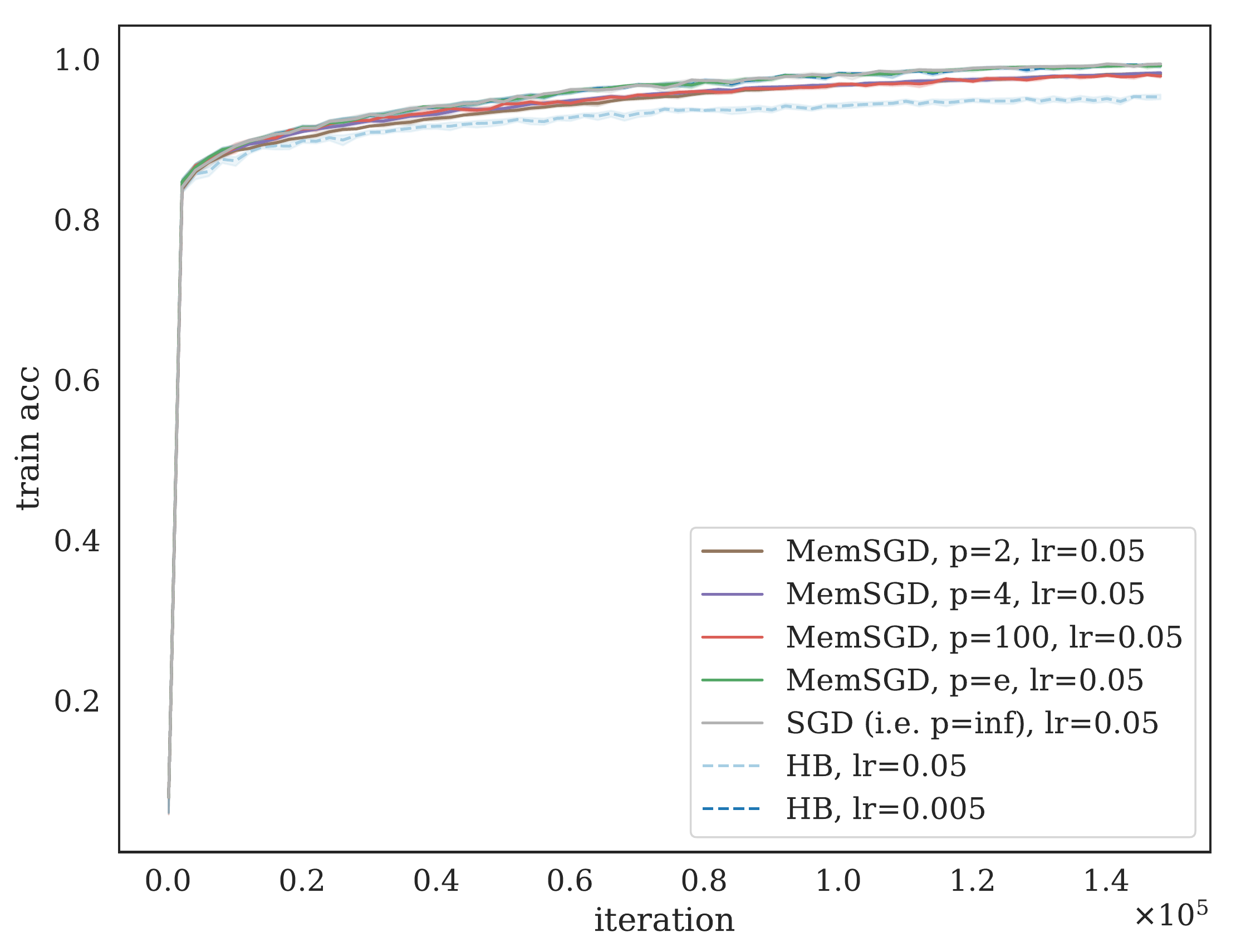}&
            \includegraphics[width=0.23\linewidth,valign=c]{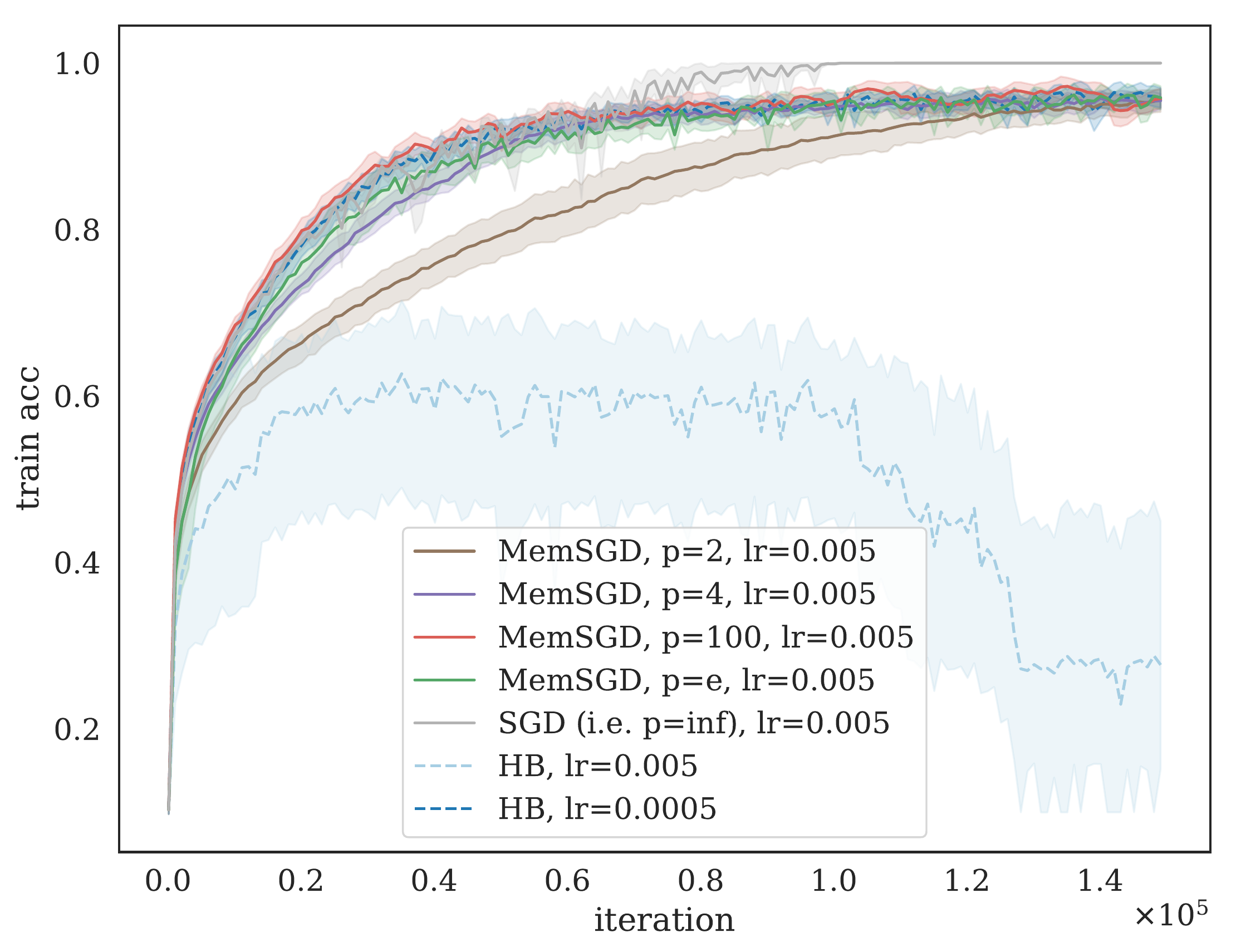}\\  
            ~
             \includegraphics[width=0.23\linewidth,valign=c]{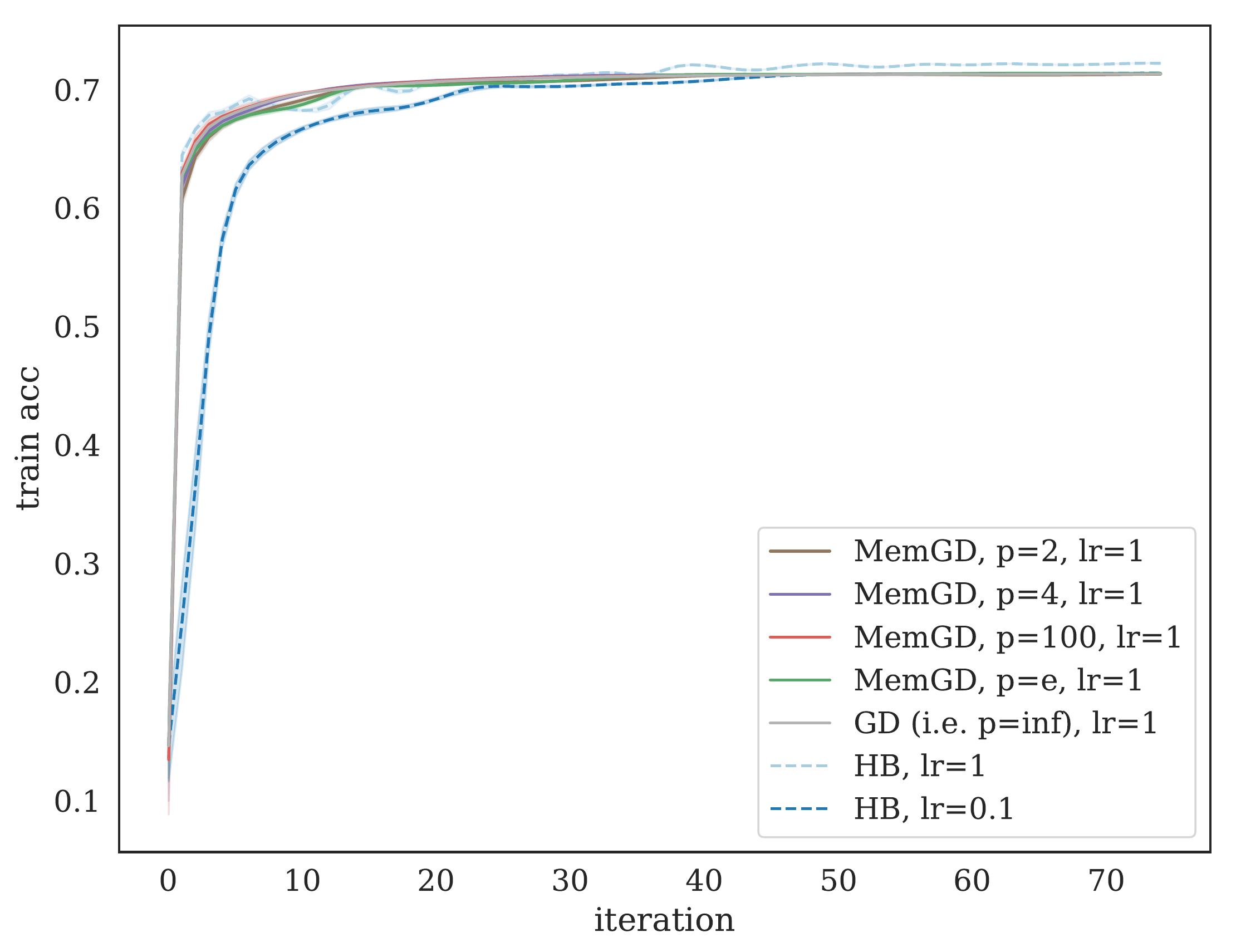}&
            \includegraphics[width=0.23\linewidth,valign=c]{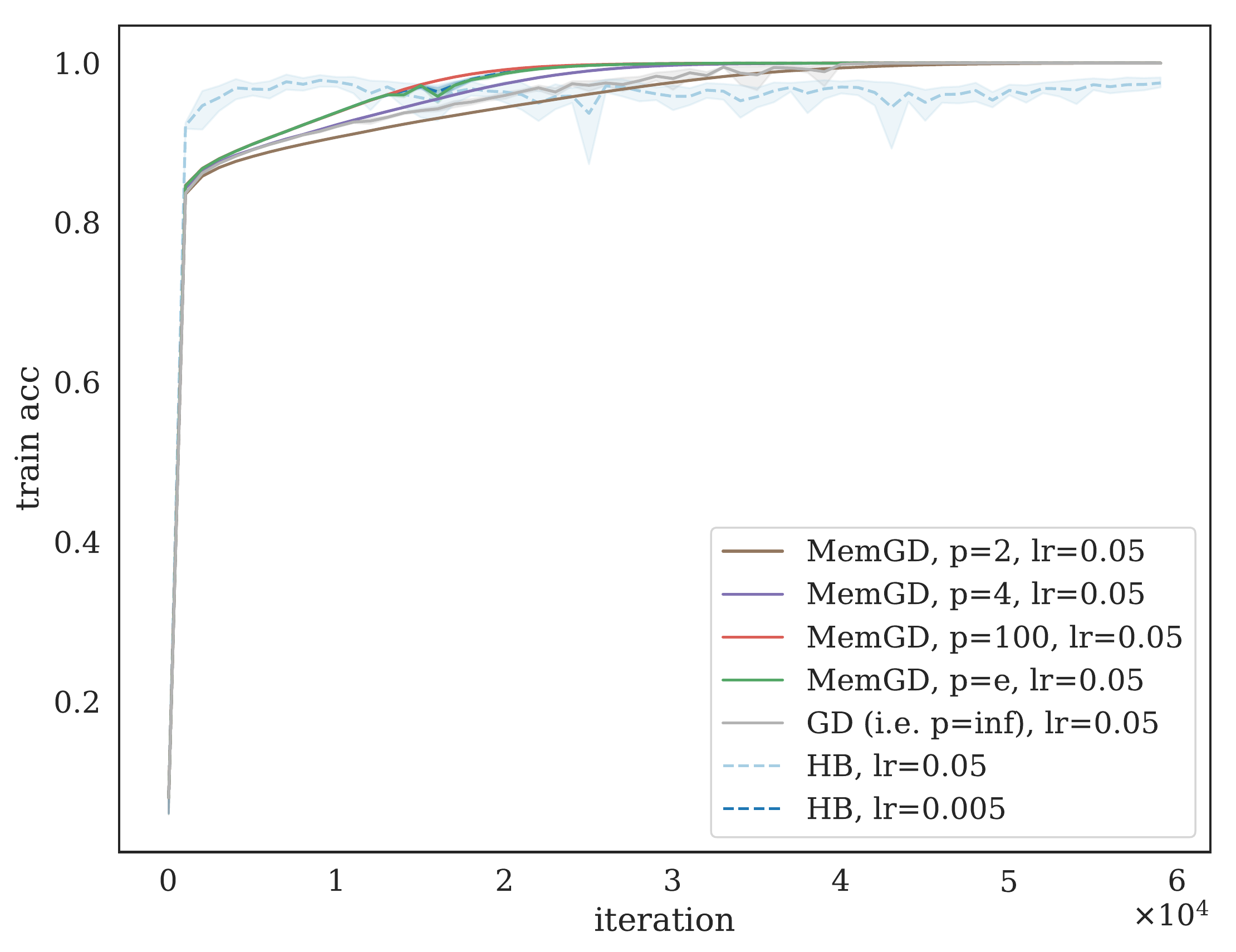}&
            \includegraphics[width=0.23\linewidth,valign=c]{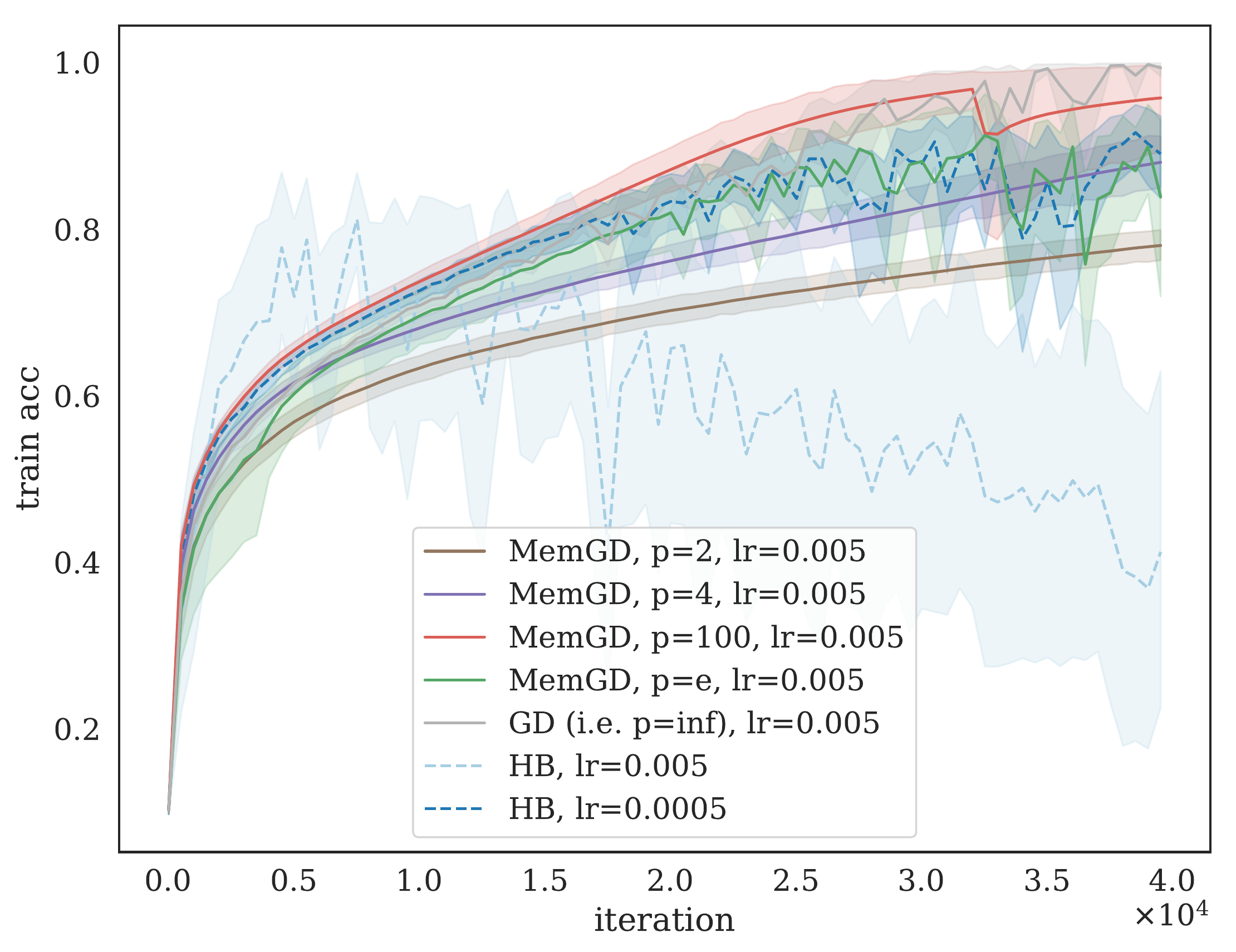}\\  
	  \end{tabular}
	  	  \vspace{-4mm}
          \caption{Real world experiments: Full training accuracy over iterations in mini- (top) and full batch (bottom) setting (undefined for autoencoder). Average and $95\%$ confidence interval of 10 runs with random initialization.} \label{fig:exp_acc}
\end{figure}

\begin{figure}[ht]
\centering 
          \begin{tabular}{c@{}c@{}c@{}c@{}}
          Covtype Logreg & MNIST Autoencoder & FashionMNIST MLP & CIFAR-10 CNN
          \\
            \includegraphics[width=0.23\linewidth,valign=c]{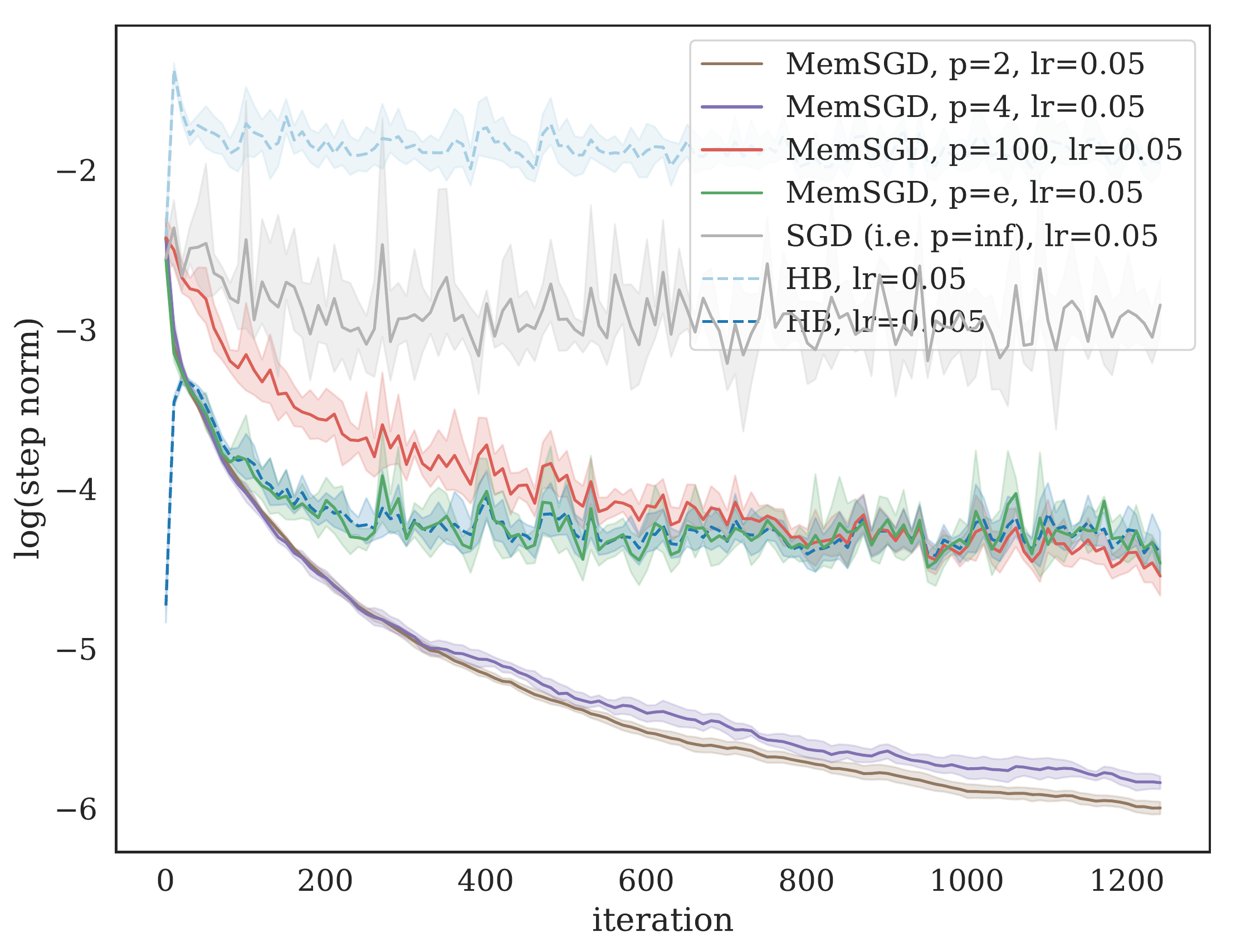}&
            \includegraphics[width=0.23\linewidth,valign=c]{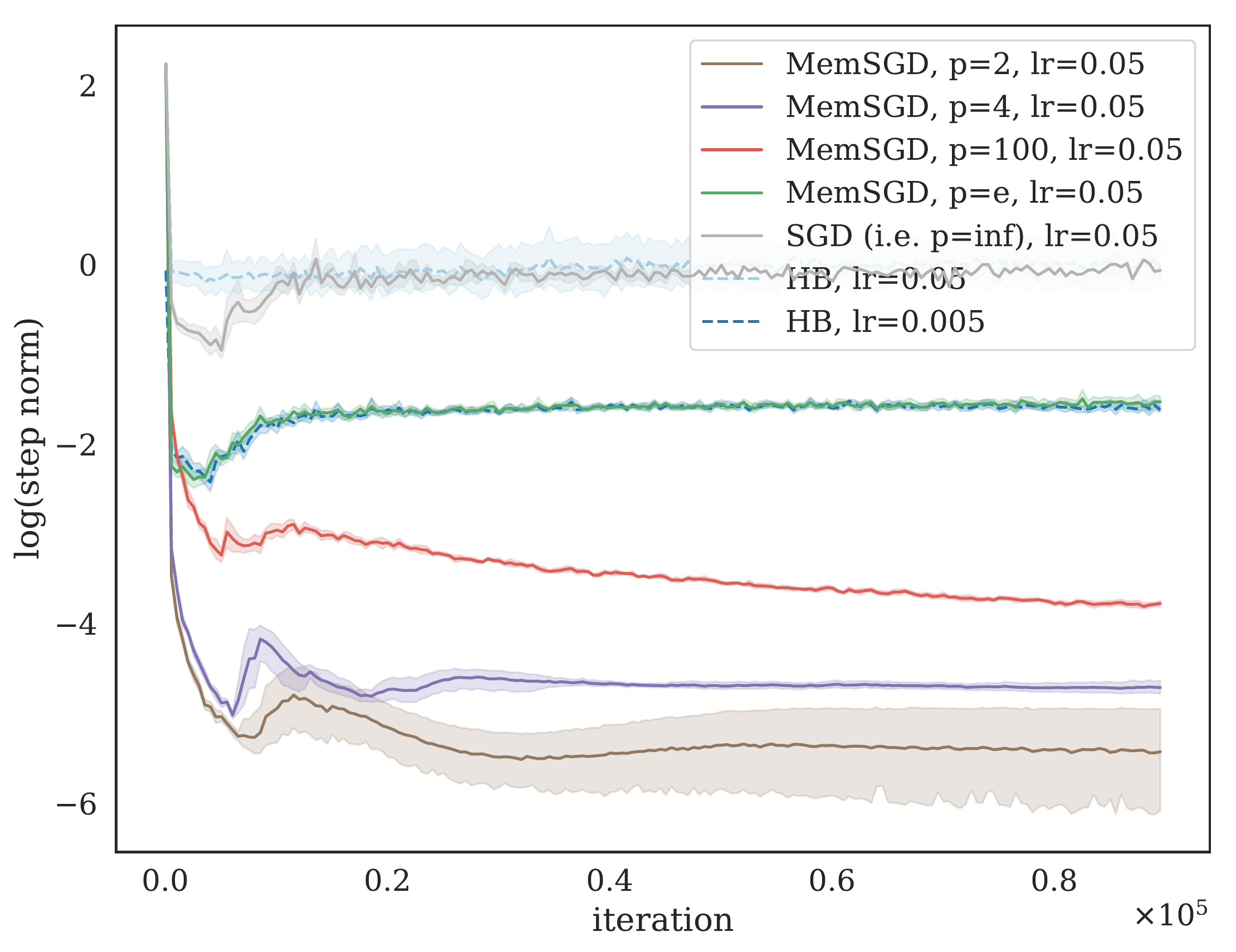}&
            \includegraphics[width=0.23\linewidth,valign=c]{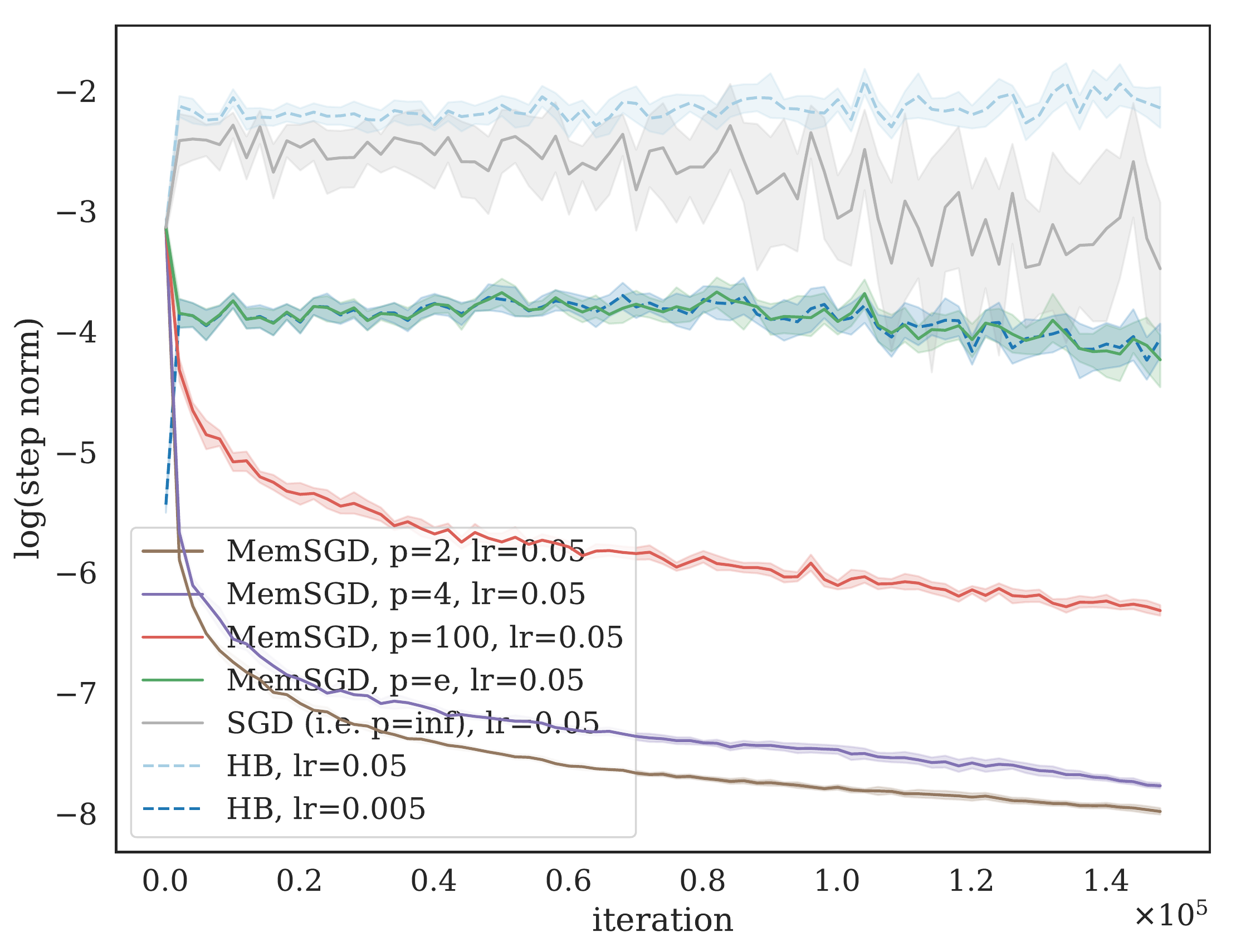}&
            \includegraphics[width=0.23\linewidth,valign=c]{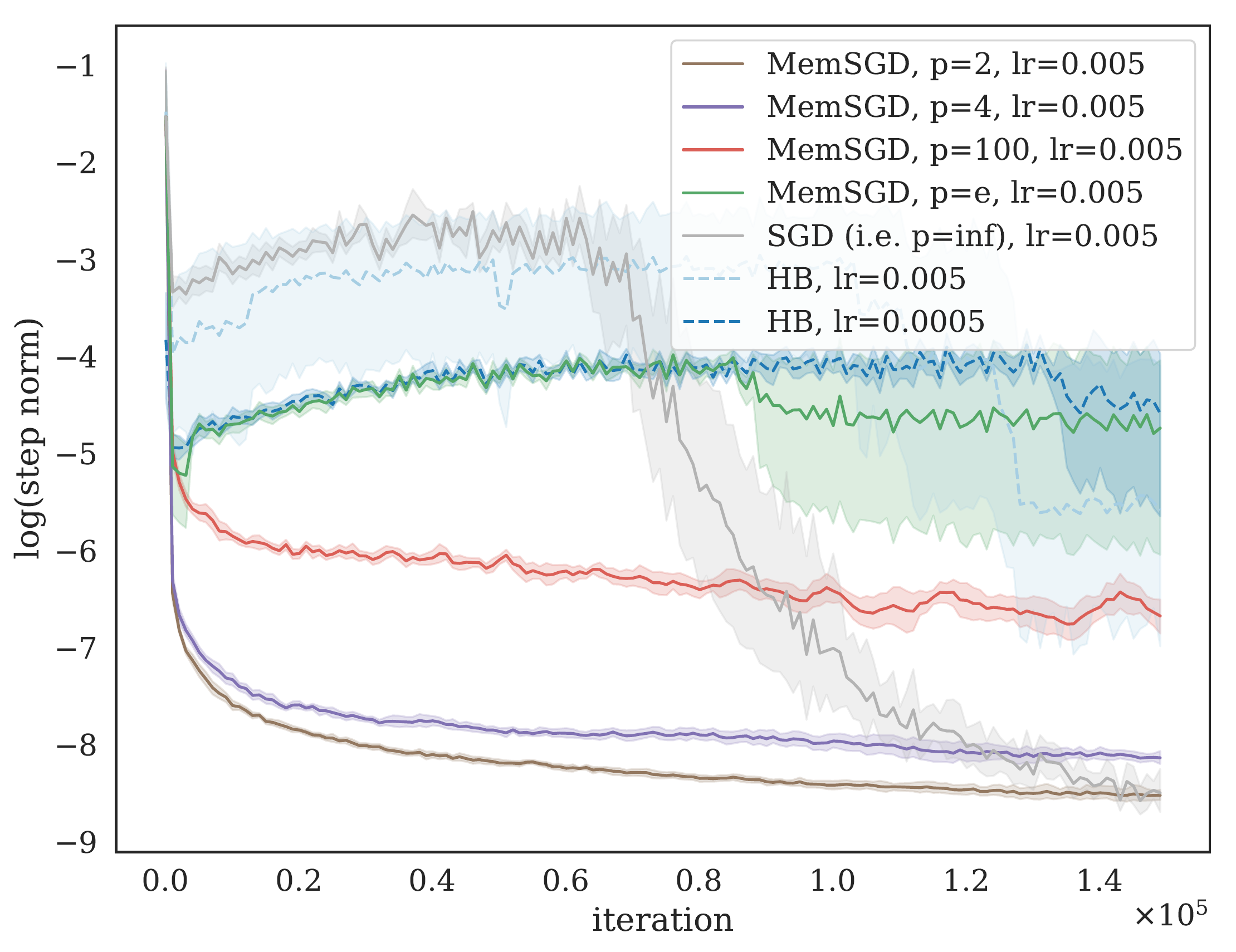}\\  
            ~
             \includegraphics[width=0.23\linewidth,valign=c]{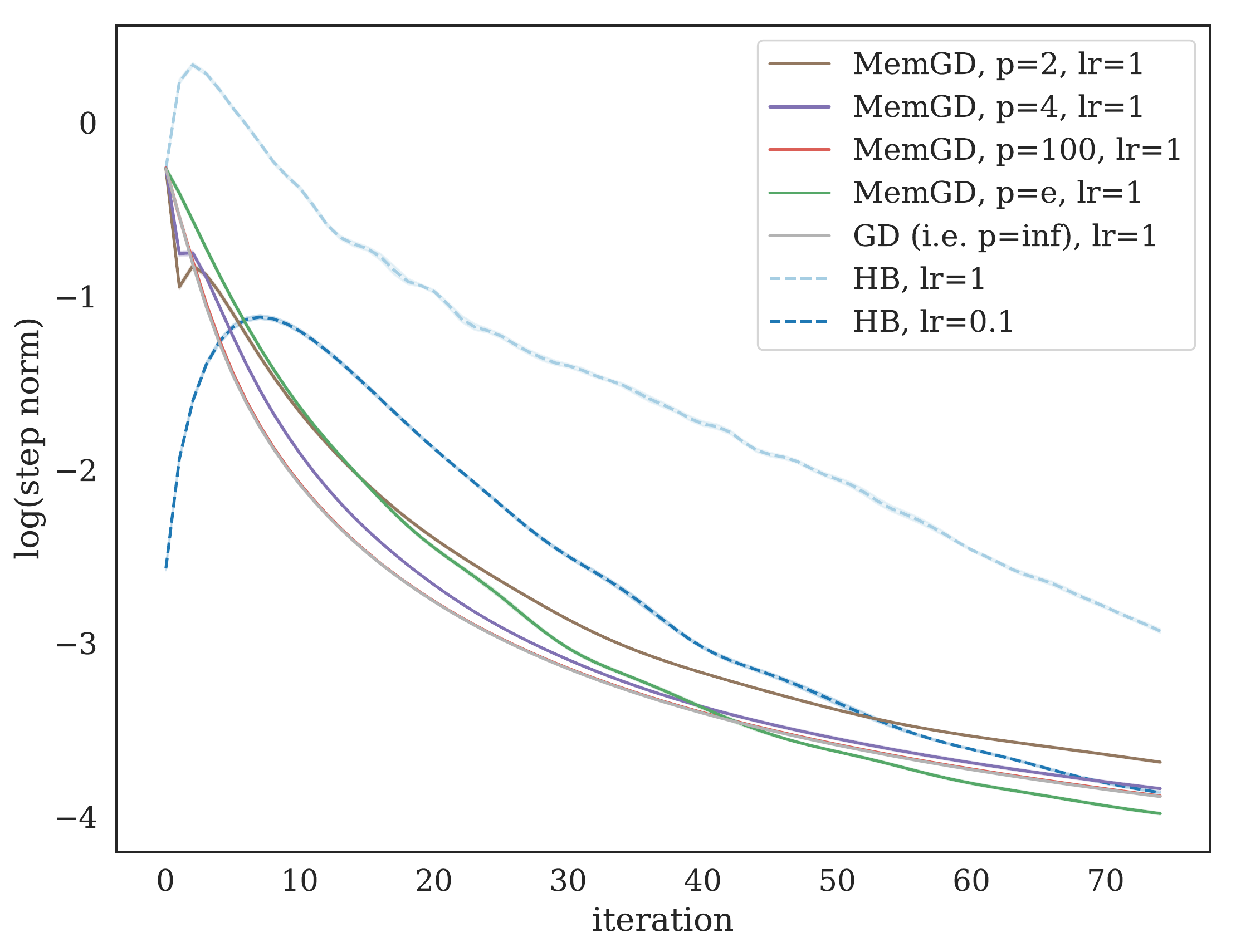}&
            \includegraphics[width=0.23\linewidth,valign=c]{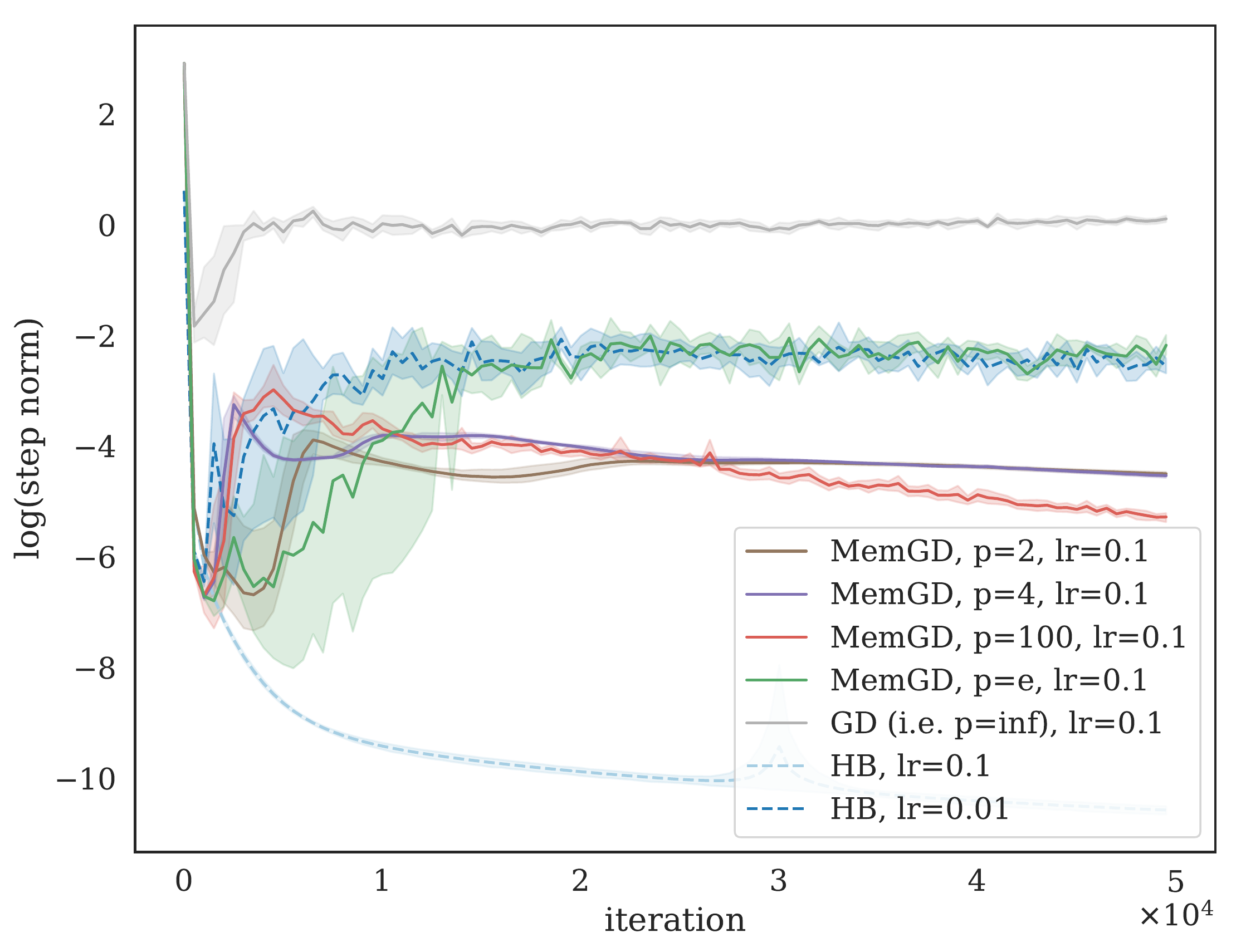}&
            \includegraphics[width=0.23\linewidth,valign=c]{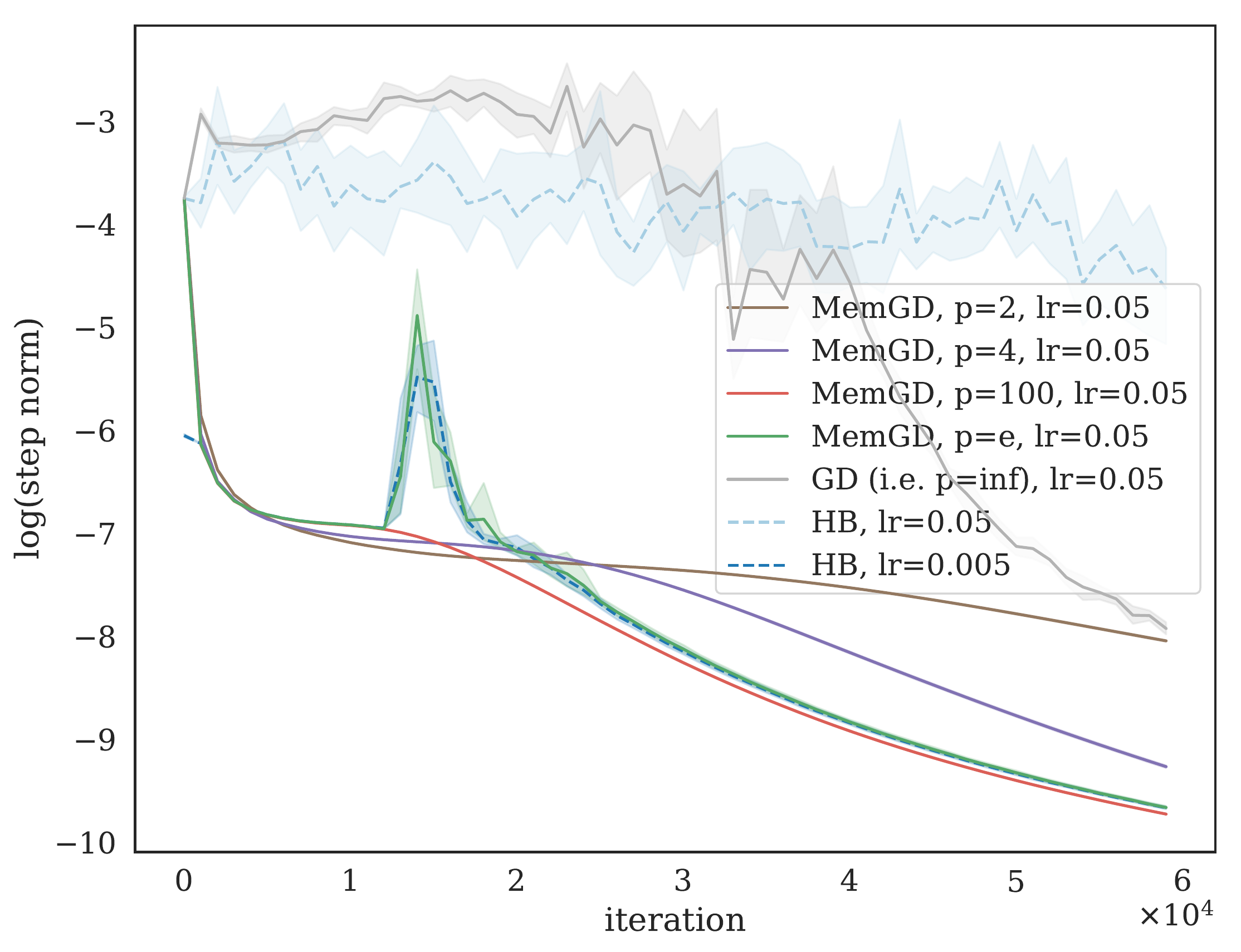}&
            \includegraphics[width=0.24\linewidth,valign=c]{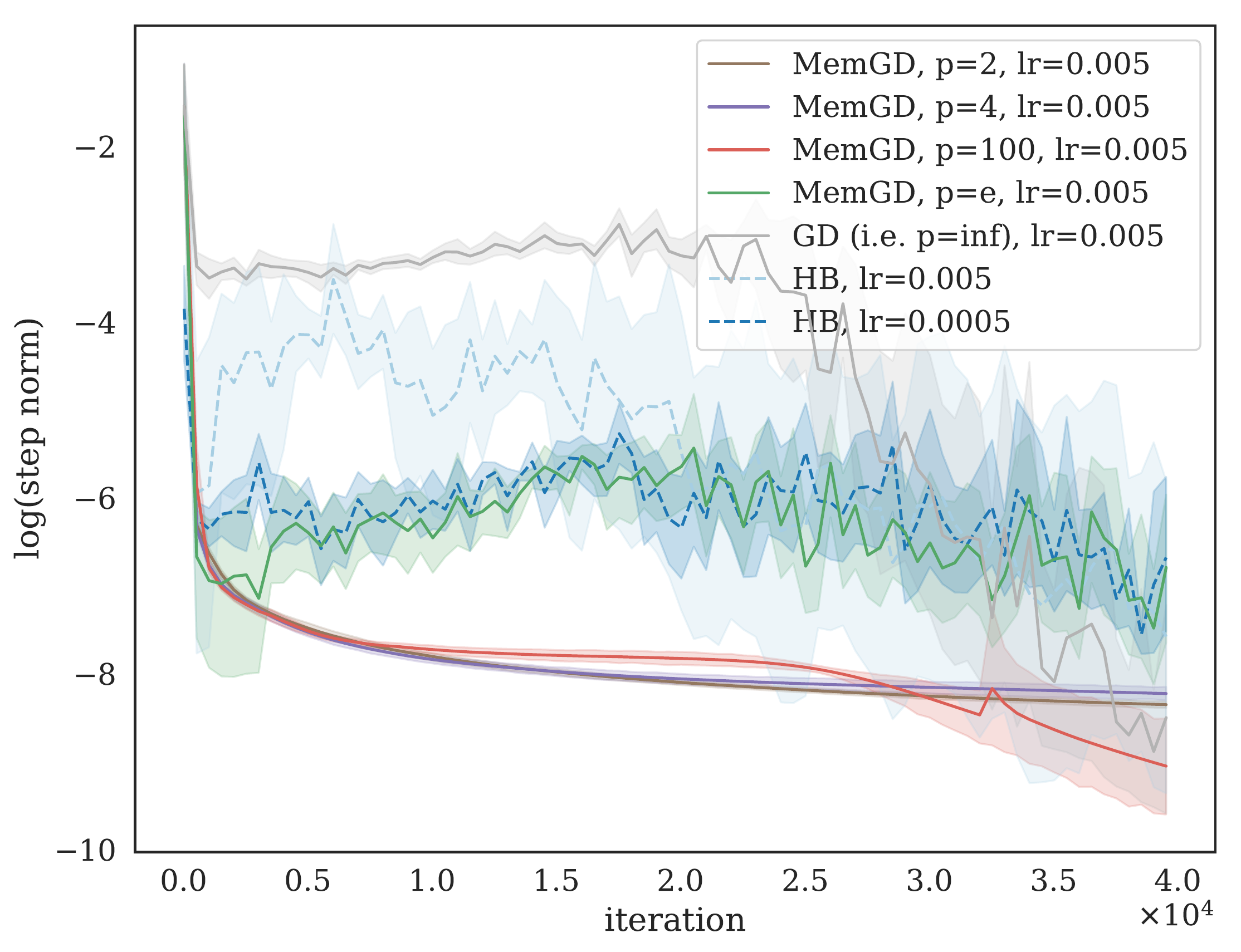}\\  
	  \end{tabular}
	  	  \vspace{-4mm}
          \caption{Real world experiments: Log step norm over iterations in mini- (top) and full batch (bottom) setting. Average and $95\%$ confidence interval of 10 runs with random initialization.}\label{fig:exp_stepnorm}
\end{figure}

\begin{figure}[ht]
\centering          
          \begin{tabular}{|c@{}|c@{}|}
          \hline
          Original&
             \includegraphics[width=0.7\linewidth,valign=c]{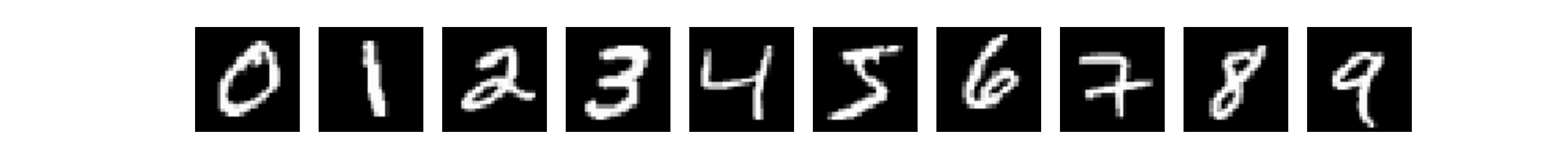} \\  \hline
              MemSGD (p=2) &
             \includegraphics[width=0.7\linewidth,valign=c]{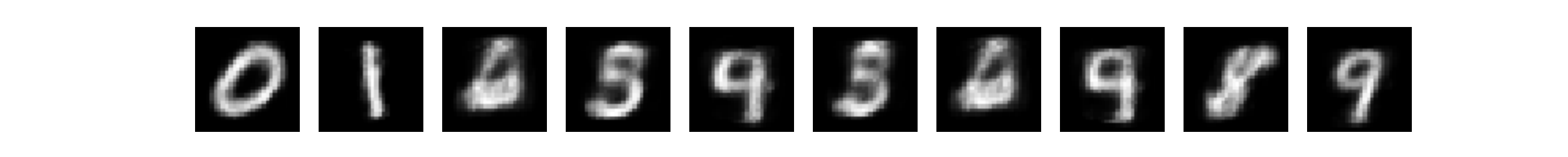} \\
              MemSGD (p=4) &
             \includegraphics[width=0.7\linewidth,valign=c]{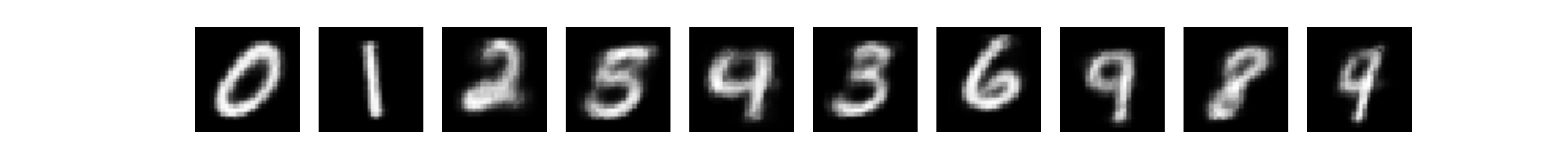} \\
             MemSGD (p=100) $\;$ &
             \includegraphics[width=0.7\linewidth,valign=c]{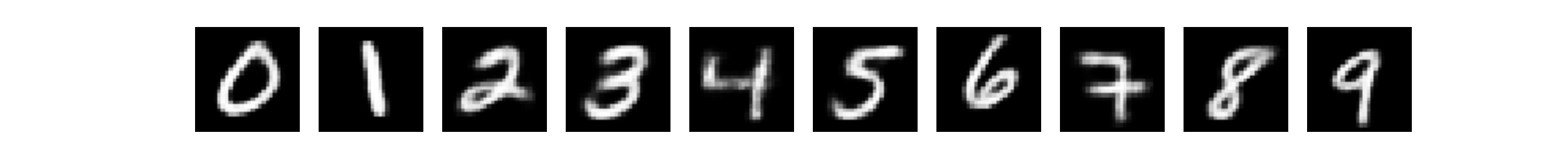} \\
              MemSGD (p=e) &
             \includegraphics[width=0.7\linewidth,valign=c]{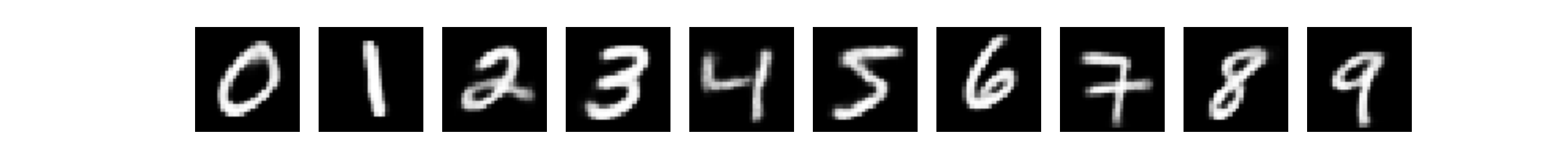} \\
             
          SGD (i.e. p=inf) &
             \includegraphics[width=0.7\linewidth,valign=c]{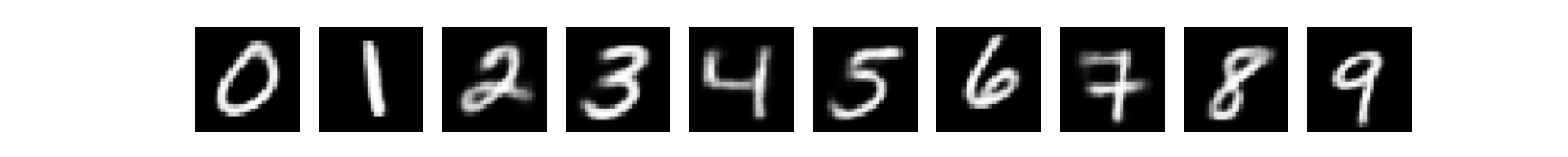} \\
         
          HB (lr=0.005) &
             \includegraphics[width=0.7\linewidth,valign=c]{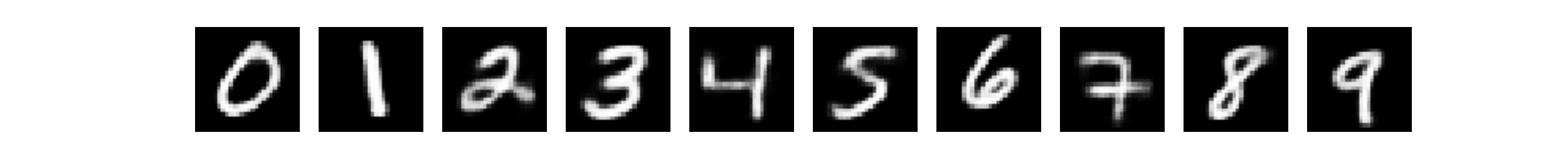} \\
       HB (lr=0.05) &
             \includegraphics[width=0.7\linewidth,valign=c]{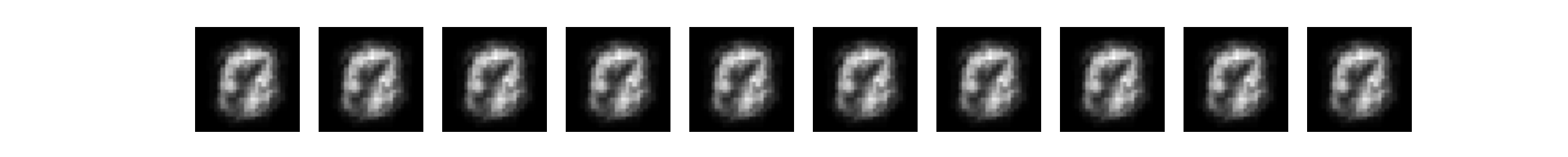} \\\hline
	  \end{tabular}
          \caption{ \footnotesize{Original and reconstructed MNIST digits for different stochastic optimization methods after convergence. Compare Fig.~\ref{fig:exp_loss} for corresponding loss.}}
     \label{fig:exp_rec}
     
\end{figure}

\end{document}